\RequirePackage{setspace}
\PassOptionsToPackage{table,usenames,dvipsnames,svgnames}{xcolor}
\documentclass[acmsmall, screen]{acmart}

\setcopyright{rightsretained}
\acmPrice{}
\acmDOI{10.1145/3276520}
\acmYear{2018}
\copyrightyear{2018}
\acmJournal{PACMPL}
\acmVolume{2}
\acmNumber{OOPSLA}
\acmArticle{150}
\acmMonth{11}

\usepackage{natbib}
\setcitestyle{round}

\usepackage{graphicx}
\usepackage{twoopt}
\usepackage{suffix}
\usepackage{xspace}
\usepackage{hyphenat}
\usepackage{textcomp}
\usepackage[super]{nth}
\usepackage{stackengine}
\setstackgap{L}{14pt}

\usepackage{tcolorbox}
\colorlet{bgcolor}{white!95!black}
\colorlet{bordercolor}{white!90!black}
\colorlet{darkbordercolor}{white!72!black}
\colorlet{darkerbordercolor}{white!36!black}

\usepackage{relsize}
\usepackage{bold-extra}
\usepackage{wasysym}
\usepackage{stmaryrd}

\usepackage{amsmath, amssymb}
\usepackage{mathrsfs}
\usepackage{mathdots}
\usepackage{mathtools}
\usepackage[nice]{nicefrac}
\usepackage{etoolbox}
\usepackage{proof}
\usepackage{wrapfig}
\usepackage{bookmark}

\usepackage{tabularx, booktabs}
\newcolumntype{C}[1]{>{\centering\arraybackslash}p{#1}}

\usepackage{tikz}
\usetikzlibrary{arrows.meta, decorations.pathmorphing, positioning, shapes}
\usepackage{forest}
\usepackage{pgfplots, pgffor}

\usepackage[inline]{enumitem}
\newcommand{\inliststyle}[1]{\textbf{\smaller#1}}
\newlist{inlist}{enumerate*}{1}
\setlist[inlist]{label={\inliststyle{(\arabic*)}}}
\setlist{topsep=0.5pt,partopsep=0pt,parsep=0pt,itemsep=0pt}

\usepackage[hang]{footmisc}
\usepackage{manyfoot}
\DeclareNewFootnote{Num}[arabic]
\DeclareNewFootnote{Sym}[fnsymbol]
\usepackage[justification=justified,
            aboveskip=3pt,belowskip=-2pt]{caption, subcaption}
\DeclareCaptionFont{sevenpt}{\fontsize{7pt}{8pt}\selectfont #1}
\usepackage{algorithm}
\usepackage[noend]{algpseudocode}

\makeatletter
\newcounter{algorithmicH}% New algorithmic-like hyperref counter
\let\oldalgorithmic\algorithmic
\renewcommand{\algorithmic}{%
  \stepcounter{algorithmicH}% Step counter
  \oldalgorithmic}% Do what was always done with algorithmic environment
\renewcommand{\theHALG@line}{ALG@line.\thealgorithmicH.\arabic{ALG@line}}
\makeatother

\algnewcommand{\LeftComment}[1]{\Statex \vspace*{1pt}\hspace*{-12pt}{\color{white!32!black} $\blacktriangleright$ \:\textsf{#1}}}
\algnewcommand{\IfThenElse}[3]{% \IfThenElse{<if>}{<then>}{<else>}
  \State \algorithmicif\ #1\ \algorithmicthen\ #2\ \algorithmicelse\ #3}
\algrenewcommand\alglinenumber[1]{\color{gray} \small $#1\,\cdot$}
\newenvironment{algfunction}[3]
    {\begingroup\textbf{func} \textsc{#1}($#2$)
     \par\leavevmode\hangindent=3.5em\hangafter=1\indent\leavevmode\hangindent=3.5em\hangafter=1\indent
     {\color{white!32!black} \textbf{output:} #3}
     \vspace*{0.5pt}
     \begin{algorithmic}[1]
     \algrenewcommand\algorithmicindent{1.25em}}
    {\end{algorithmic}\endgroup}
\newcommand{\resized}[3]{
  \resizebox{#1\linewidth}{!}{
    \begin{subfigure}[t]{#2\linewidth}
      #3
    \end{subfigure}
  }}

\newcommand{\algobox}[2]{
  \vspace*{-2.5pt}
  \begin{tcolorbox}[boxrule=0.5pt,arc=2pt,
                    left=0.75pt,right=-10pt,top=3pt,bottom=3pt,boxsep=0pt,
                    colback=white,colframe=darkbordercolor]
  \resized{0.965}{#1}{#2}\end{tcolorbox}
  \vspace*{-3pt}}
\DeclareCaptionSubType*{algorithm}

\let\oldReturn\Return
\renewcommand{\Return}{\State\oldReturn}

\newcommand{\redact}[1]{#1}

\usepackage[nameinlink]{cleveref}
\crefname{section}{\S\!}{\S\S\!}
\Crefname{section}{\S\!}{\S\S\!}

\usepackage{listings}

\lstdefinelanguage{ProseDSL}{
  texcl=true, % latex in comments
  morecomment=[l]{//},
  morestring=[b]",
  keywords=[1]{bool},
  keywords=[2]{let,in},
  keywords=[3]{IsNull,EndOf,SuffixAfter,Fail,Disjunction},
  keywords=[4]{IToken,SuffixRegion}
}
\newcommand{\rulesep}{\unskip\ \vrule\ }

\newcommand{\bordersep}{\colorlet{saved}{.}\color{darkbordercolor}\hspace*{1pt}\rulesep\hspace*{1pt}\color{saved}}
\newcommand{\visiblespace}{{\color{darkerbordercolor}\,\textvisiblespace\,}}
\newcommand{\nsp}{\kern-0.1em}

\newcommand{\xcirc}[1]{\vcenter{\hbox{$#1\diamond$}}}
\newcommand{\concat}{\:\,{\color{darkerbordercolor}%
    \mathbin{\mathchoice
            {\xcirc\scriptstyle}
            {\xcirc\scriptstyle}
            {\xcirc\scriptscriptstyle}
            {\xcirc\scriptscriptstyle}}}\,\:}

\newcommand{\patternmode}[2][\small]{{#1\ensuremath{#2}}}

\newcommand{\circledchar}[2][]{%
  \tikz[baseline=(char.base)]{%
    \node[shape = circle, draw, inner sep = 1pt]
    (char) {\phantom{\ifblank{#1}{#2}{#1}}};%
    \node at (char.center) {\makebox[0pt][c]{#2}};}}
\robustify{\circledchar}

\newcommand{\numFFRecallCond}{131\xspace}
\newcommand{\numFFPrecisionCond}{140\xspace}

\newcommand{\numFFTotalCond}{163\xspace}
\newcommand{\numFFRecallCondPercent}{80\%\xspace}
\newcommand{\numFFPrecisionCondPercent}{86\%\xspace}

\newcommand{\numDatasets}{75\xspace}

\newcommand{\numTasks}{153\xspace}

\newcommand{\github}[1]{\url{https://github.com/SaswatPadhi/FlashProfileDemo/tree/master/#1}\xspace}
\newcommand{\dsetgroup}[1]{\textsc{\small#1}}

\newcommand{\tool}[1]{\textsf{#1}}
\newcommand{\CSharp}{\tool{C\#}\xspace}
\newcommand{\LStar}{\tool{L-Star}\xspace}
\newcommand{\RPNI}{\tool{RPNI}\xspace}
\newcommand{\Excel}{\textsf{Microsoft Excel}\xspace}
\newcommand{\AML}{\textsf{Azure ML Workbench}\xspace}

\newcommand{\FlashFill}{\tool{Flash Fill}\xspace}
\newcommand{\FlashMeta}{\tool{FlashMeta}\xspace}
\newcommand{\FlashProfile}{\tool{FlashProfile}\xspace}
\newcommand{\PottersWheel}{\tool{Potter's Wheel}\xspace}
\newcommand{\PROSE}{\tool{PROSE}\xspace}

\newcommand{\SSDT}{\tool{SSDT}\xspace}
\newcommand{\Ataccama}{\tool{Ataccama One}\xspace}
\newcommand{\One}{\tool{A1}\xspace}

\newcommand{\OpenRefine}{\tool{OpenRefine}\xspace}

\newcommand{\kmeans}{\tool{k-means}\xspace}

\newcommand{\FP}{\tool{FP}}
\newcommand{\FPk}[1]{\tool{FP}\ensuremath{_{#1}}}

\newcommand{\SampleDissimilarities}{SampleDissimilarities}
\newcommand{\ApproxDissimilarityMatrix}{ApproxDMatrix}
\newcommand{\BestPattern}{LearnBestPattern}

\newcommand{\GetPrefixTokens}{GetMaxCompatibleAtoms}
\newcommand{\Clustering}{Partition}
\newcommand{\AHC}{AHC}
\newcommand{\Profiling}{Profile}

\newcommand{\BuildHierarchy}{BuildHierarchy}

\newcommand{\OrderPartitions}{OrderPartitions}
\newcommand{\MergeMostSimilarPatterns}{CompressProfile}
\newcommand{\LargeProfile}{BigProfile}

\newcommand{\algo}[1]{\textsc{#1}}

\newcommand{\SampleDissimilaritiesAlgo}{\algo{\SampleDissimilarities}\xspace}
\newcommand{\BestPatternAlgo}{\algo{\BestPattern}\xspace}

\newcommand{\GetPrefixTokensAlgo}{\algo{\GetPrefixTokens}\xspace}
\newcommand{\AHCAlgo}{\algo{\AHC}\xspace}
\newcommand{\ClusteringAlgo}{\algo{\Clustering}\xspace}
\newcommand{\ProfilingAlgo}{\algo{\Profiling}\xspace}
\newcommand{\BuildHierarchyAlgo}{\algo{\BuildHierarchy}\xspace}

\newcommand{\OrderPartitionsAlgo}{\algo{\OrderPartitions}\xspace}

\newcommand{\ApproxDissimilarityMatrixAlgo}{\algo{\ApproxDissimilarityMatrix}\xspace}
\newcommand{\MergeMostSimilarPatternsAlgo}{\algo{\MergeMostSimilarPatterns}\xspace}
\newcommand{\LargeProfileAlgo}{\algo{\LargeProfile}\xspace}

\newcommand{\WithLCParam}[1]{#1$_\mathsmaller{\langle\learner,\cost\rangle}$}

\newcommand{\polar}{\ensuremath{\rho}\xspace}

\DeclareMathOperator*{\argmin}{\arg\min}
\DeclareMathOperator*{\argmax}{\arg\max}

\DeclareMathSymbol{\mlq}{\mathord}{operators}{``}
\DeclareMathSymbol{\mrq}{\mathord}{operators}{`'}

\DeclarePairedDelimiter\ceil{\lceil}{\rceil}

\newcommand{\ScaledEqn}[2]{\begin{equation}
  \resizebox{!}{#1\baselineskip}{\ensuremath{#2}\smallbreak}
\end{equation}}
\WithSuffix\newcommand\ScaledEqn*[2]{\begin{equation*}
  \resizebox{!}{#1\baselineskip}{\ensuremath{#2}\smallbreak}
\end{equation*}}

\newcommand{\Tint}{\ensuremath{\textsf{Int}}\xspace}
\newcommand{\Tatom}{\ensuremath{\textsf{Atom}}\xspace}
\newcommand{\Tdouble}{\ensuremath{\textsf{Real}}\xspace}
\newcommand{\Tbool}{\ensuremath{\textsf{Bool}}\xspace}
\newcommand{\Tstring}{\ensuremath{\textsf{String}}\xspace}
\newcommand{\Tpattern}{\ensuremath{\textsf{Pattern}}\xspace}

\newcommand{\Tprofile}{\ensuremath{\textsf{Profile}}\xspace}

\newcommand{\dsl}{\ensuremath{\mathscr{L}}\xspace}
\newcommand{\spec}{\ensuremath{\varphi}}

\newcommand{\vsa}[1][N]{\widetilde{#1}}

\newcommand{\stringliteral}[1]{\ensuremath{\textnormal{\textcolor{darkerbordercolor}{``\texttt{#1}''}}}}
\newcommand{\assuming}{\;\ifnum\currentgrouptype=16 \middle\fi\vert\;}
\newcommand{\True}{\ensuremath{\texttt{True}}\xspace}
\newcommand{\tospec}{\rightsquigarrow}
\newcommand{\bydef}{\mathrel{\overset{\mathsf{\scriptscriptstyle def}}{=}}}
\newcommand{\dataset}{\ensuremath{\mathcal{S}}\xspace}

\newcommand{\NonNegReals}{\ensuremath{\mathbb{R}_{\resizebox{!}{0.32\baselineskip}{$\,\nsp\geq\!0$}}}}
\newcommand{\Strings}{\ensuremath{\mathbb{S}}}
\newcommand{\Naturals}{\ensuremath{\mathbb{N}}}
\newcommand{\Integers}{\ensuremath{\mathbb{Z}}}

\newcommand{\cost}{\ensuremath{\mathcal{C}}\xspace}
\newcommand{\staticCost}{\ensuremath{Q}\xspace}
\newcommand{\costP}{\ensuremath{\cost_\mathsmaller{\textsf{FP}}}\xspace}
\newcommand{\costmax}{\ensuremath{\infty}\xspace}

\newcommand{\learner}{\ensuremath{\mathcal{L}}\xspace}
\newcommand{\learnerP}{\ensuremath{\learner_\mathsmaller{\textsf{FP}}}\xspace}

\newcommand{\objective}{\ensuremath{\mathcal{O}}\xspace}

\newcommandtwoopt{\clustering}[2][\vsa][\state]{#1 /_{#2}}

\robustify{\stringliteral}

\newcommand{\LP}{\ensuremath{\dsl_\mathsmaller{\textsf{FP}}}\xspace}
\newcommand{\bookmarkLP}{\texorpdfstring{{\large \LP}\:}{L\textpinferior}}

\newcommand{\patfail}{\ensuremath{\pmb{\bot}}\xspace}
\newcommand{\token}{\ensuremath{\alpha}\xspace}
\newcommand{\Tokens}{\ensuremath{\mho}\xspace}
\newcommand{\CCTokens}{\ensuremath{C}\xspace}
\newcommand{\enrichedTokens}{\ensuremath{\widehat{\Tokens}}\xspace}
\newcommand{\pattern}{\ensuremath{P}\xspace}
\newcommand{\patexpr}[1]{\ensuremath{\pattern\,[#1]}\xspace}
\newcommand{\profile}{\ensuremath{\widetilde{\pattern}}\xspace}
\newcommand{\datascore}{\ensuremath{W}\xspace}
\newcommand{\patternapply}[1]{\ensuremath{P(#1)}\xspace}
\newcommand{\substr}[3]{\ensuremath{#1[#2\,\pmb{\colon}#3]}}

\newcommand{\ApproxDissMatrix}{\ensuremath{A}\xspace}

\newcommand{\tokStyle}[1]{\ensuremath{\textsf{#1}}}
\newcommand{\tokClass}{\tokStyle{Class}}
\newcommand{\tokFunct}{\tokStyle{Funct}}
\newcommand{\tokRegEx}{\tokStyle{RegEx}}
\newcommand{\tokConst}{\tokStyle{Const}}
\newcommand{\rep}[2]{\ensuremath{{#1}^{\times#2}}}

\newcommand{\CClass}[1]{\ensuremath{\textsf{#1}}\xspace}
\newcommand{\Regex}[1]{\ensuremath{\langle\textsf{#1}\rangle}\xspace}

\newcommand{\Any}{\CClass{Any}}
\newcommand{\Alpha}{\CClass{Alpha}}
\newcommand{\Lower}{\CClass{Lower}}
\newcommand{\Upper}{\CClass{Upper}}
\newcommand{\Digit}{\CClass{Digit}}
\newcommand{\Symbol}{\CClass{Symb}}
\newcommand{\DotDash}{\CClass{DotDash}}
\newcommand{\Space}{\CClass{\!\visiblespace\!}}
\newcommand{\BinDigit}{\CClass{Bin}}
\newcommand{\HexDigit}{\CClass{Hex}}
\newcommand{\BaseSixtyFour}{\CClass{Base64}}
\newcommand{\AlphaDash}{\CClass{AlphaDash}}
\newcommand{\AlphaSpace}{\CClass{AlphaSpace}}
\newcommand{\AlphaDigitSpace}{\CClass{AlphaDigitSpace}}
\newcommand{\Punct}{\CClass{Punct}}
\newcommand{\AlphaDigit}{\CClass{AlphaDigit}}

\newcommand{\DOI}{\Regex{DOI}}
\newcommand{\ISBN}{\Regex{ISBN10}}
\newcommand{\TitleCaseWord}{\Regex{TitleCaseWord}}

\newcommand{\sCClass}[1]{\texttt{#1}}
\newcommand{\sDigit}{\sCClass{D}}
\newcommand{\sUpper}{\sCClass{U}}

\newcommand{\sAlphaSpace}{\ensuremath{\Pi}}

\newcommand{\sAlphaDigitSpace}{\ensuremath{\Sigma}}
\newcommand{\sSpace}{\,\Space\,}
\newcommand{\sconcat}{\;}
\newcommand{\sstringliteral}[1]{\stringliteral{#1}}
\newcommand{\srep}[2]{\ensuremath{#1^{\,#2}}}

\newcommand{\matchsymb}{\ensuremath{\triangleright}}
\newcommand{\matches}[2]{\ensuremath{#1\;\matchsymb\;#2}\xspace}
\newcommand{\compatsymb}{\ensuremath{\propto}}
\newcommand{\compat}[2]{\ensuremath{#1\compatsymb#2}\xspace}
\newcommand{\compatmax}[2]{\ensuremath{\max^{#1}_\compatsymb[\scalebox{0.9}{#2}]}\xspace}

\newcommand{\EmptySymb}{\ensuremath{\texttt{Empty}}}
\newcommand{\SuffixAfterSymb}{\ensuremath{\texttt{SuffixAfter}}}
\newcommand{\DSLOp}[2]{\ensuremath{\texttt{#1(}#2\texttt{)}}\xspace}

\newcommand{\Empty}[1]{\DSLOp{\EmptySymb}{#1}}
\newcommand{\SuffixAfter}[2]{\DSLOp{\SuffixAfterSymb}{#1,#2}}

\newcommand{\SyDissSymb}{\ensuremath{\eta}}
\newcommand{\SyDiss}[2]{\ensuremath{\SyDissSymb\,(#1, #2)}\xspace}
\newcommand{\cluSyDissSymb}{\ensuremath{\widehat{\eta}}}
\newcommand{\cluSyDiss}[3]{\ensuremath{\cluSyDissSymb\,(#1, #2 \mid #3)}\xspace}

\newcommand{\pointSampleFactor}{\ensuremath{\mu}\xspace}
\newcommand{\edgeSampleFactor}{\ensuremath{\theta}\xspace}
\newcommand{\maxClusters}{\ensuremath{M}\xspace}
\newcommand{\maxClustersT}{\ensuremath{\widehat{M}}\xspace}
\newcommand{\minClusters}{\ensuremath{m}\xspace}
\newcommand{\hierarchy}{\ensuremath{H}\xspace}

\newcommand{\stringPair}[2]{{\small \{\,\stringliteral{#1}, \stringliteral{#2}\,\}\xspace}}

\newcommand{\strlen}{\texttt{len}\xspace}
\newcommand{\strcount}{\texttt{cnt}\xspace}
\newcommand{\strbegin}{\texttt{begin}\xspace}
\newcommand{\SimBaseline}{\ensuremath{RF}\xspace}

\newcommand{\Strue}{\ensuremath{\texttt{true}}\xspace}

\newcommand{\vars}{\texttt}
\newcommand{\prop}{\textsf}

\newcommand{\BigStep}[2]{#1\ \ \Downarrow\ \ #2}
\newcommand{\bs}{\textbackslash}

\newenvironment{defn}[1]
  {\begin{definition}\textsf{#1}\ ---\ }
  {\end{definition}}

\newenvironment{exmp}[1][]
  {\begin{example}}
  {\end{example}}

\arrayrulecolor{bordercolor}
\begin{document}

\setlength{\abovedisplayshortskip}{0.5pt}
\setlength{\belowdisplayshortskip}{0.5pt}
\setlength{\abovedisplayskip}{1.5pt}
\setlength{\belowdisplayskip}{1.5pt}
\setlength{\intextsep}{2pt}

\Crefname{thrm}{Theorem}{Theorems}
\Crefname{exmp}{Example}{Examples}
\Crefname{definition}{Definition}{Definitions}
\Crefname{figure}{Fig.}{Figures}
\Crefname{table}{Table}{Tables}
\Crefname{equation}{Equation}{Equations}

\crefname{thrm}{Theorem}{Theorems}
\crefname{exmp}{Example}{Examples}
\crefname{definition}{Definition}{Definitions}
\crefname{figure}{Fig.}{Figures}
\crefname{table}{Table}{Tables}
\crefname{equation}{Equation}{Equations}

\title{\FlashProfile: A Framework for Synthesizing Data Profiles}

\author{Saswat Padhi}
\authornote{Work done during an internship with \PROSE team at Microsoft.}
\orcid{0000-0001-6865-4359}
\affiliation{
  \department{Dept. of Computer Science}
  \institution{University of California -- Los Angeles}
  \state{CA}
  \postcode{90095}
  \country{USA}
}
\email{padhi@cs.ucla.edu}

\author{Prateek Jain}
\affiliation{
  \institution{Microsoft Research}
  \city{Bangalore}
%  \postcode{560001}
  \country{India}
}
\email{prajain@microsoft.com}

\author{Daniel Perelman}
\affiliation{
  \institution{Microsoft Corporation}
  \city{Redmond}
  \state{WA}
  \postcode{98052}
  \country{USA}
}
\email{danpere@microsoft.com}

\author{Oleksandr Polozov}
\orcid{0000-0003-3669-4262}
\affiliation{
  \institution{Microsoft Research}
  \city{Redmond}
  \state{WA}
  \postcode{98052}
  \country{USA}
}
\email{polozov@microsoft.com}

\author{Sumit Gulwani}
\affiliation{
  \institution{Microsoft Corporation}
  \city{Redmond}
  \state{WA}
  \postcode{98052}
  \country{USA}
}
\email{sumitg@microsoft.com}

\author{Todd Millstein}
\affiliation{
  \department{Dept. of Computer Science}
  \institution{University of California -- Los Angeles}
  \state{CA}
  \postcode{90095}
  \country{USA}
}
\email{todd@cs.ucla.edu}

\renewcommand{\shortauthors}{S. Padhi, P. Jain, D. Perelman, O. Polozov, S. Gulwani, and T. Millstein}

\begin{abstract}
We address the problem of learning a syntactic \emph{profile} for a collection of strings,
i.e. a set of regex-like patterns that succinctly describe the syntactic variations in the strings.
Real-world datasets, typically curated from multiple sources, often contain data in various syntactic formats.
Thus, any data processing task is preceded by the critical step of data format identification.
However, manual inspection of data to identify the different formats is infeasible in standard big-data scenarios.

Prior techniques are restricted to a small set of pre-defined patterns
(e.g. digits, letters, words, etc.),
and provide no control over granularity of profiles.
We define syntactic profiling as a problem of clustering strings based on \emph{syntactic similarity},
followed by identifying patterns that succinctly describe each cluster.
We present a technique for synthesizing such profiles over a given language of patterns,
that also allows for interactive refinement by requesting a desired number of clusters.

Using a state-of-the-art inductive synthesis framework, \PROSE, we have implemented our technique as \FlashProfile.
Across $153$ tasks over $75$ large real datasets, we observe a median profiling time of only $\sim$0.7\,s.
Furthermore, we show that access to syntactic profiles may allow for more accurate synthesis of programs,
i.e. using fewer examples, in programming-by-example (PBE) workflows such as \FlashFill.

\end{abstract}

%% 2012 ACM Computing Classification System (CSS) concepts
%% Generate at 'http://dl.acm.org/ccs/ccs.cfm'.
\begin{CCSXML}
  <ccs2012>
    <concept>
      <concept_id>10002951.10003317.10003347.10003356</concept_id>
      <concept_desc>Information systems~Clustering and classification</concept_desc>
      <concept_significance>500</concept_significance>
    </concept>
    <concept>
      <concept_id>10002951.10003317.10003347.10003357</concept_id>
      <concept_desc>Information systems~Summarization</concept_desc>
      <concept_significance>300</concept_significance>
    </concept>
    <concept>
      <concept_id>10011007.10011006.10011050.10011056</concept_id>
      <concept_desc>Software and its engineering~Programming by example</concept_desc>
      <concept_significance>500</concept_significance>
    </concept>
    <concept>
      <concept_id>10011007.10011006.10011050.10011017</concept_id>
      <concept_desc>Software and its engineering~Domain specific languages</concept_desc>
      <concept_significance>300</concept_significance>
    </concept>
    <concept>
      <concept_id>10010147.10010257.10010258.10010260.10010229</concept_id>
      <concept_desc>Computing methodologies~Anomaly detection</concept_desc>
      <concept_significance>100</concept_significance>
    </concept>
  </ccs2012>
\end{CCSXML}

\ccsdesc[500]{Information systems~Clustering and classification}
\ccsdesc[300]{Information systems~Summarization}
\ccsdesc[500]{Software and its engineering~Programming by example}
\ccsdesc[300]{Software and its engineering~Domain specific languages}
\ccsdesc[100]{Computing methodologies~Anomaly detection}
%% End of generated code

%% Keywords
%% comma separated list
\keywords{data profiling, pattern profiles, outlier detection, hierarchical clustering, pattern learning, program synthesis}

\maketitle

\section{Introduction} \label{sec:intro}

\noindent
In modern data science, most real-life datasets lack high-quality metadata ---
they are often incomplete, erroneous, and unstructured~\citep{dong2013big}.
This severely impedes data analysis, even for domain experts.
For instance, a merely preliminary task of \emph{data wrangling} (importing, cleaning, and reshaping data) consumes
50\,--\,80\% of the total analysis time~\citep{lohr2014big}.
Prior studies show that high-quality metadata not only help users clean, understand, transform, and reason over data, but
also enable advanced applications, such as compression, indexing, query optimization,
and schema matching~\citep{abedjan2015profiling}.
Traditionally, data scientists engage in \emph{data gazing}~\citep{maydanchik2007data} ---
they manually inspect small samples of data, or experiment with aggregation queries to get a bird's-eye view of the data.
Naturally, this approach does not scale to modern large-scale datasets~\citep{abedjan2015profiling}.

\emph{Data profiling} is the process of generating small but useful metadata
(typically as a succinct summary) for the data~\citep{abedjan2015profiling}.
In this work, we focus on syntactic profiling, i.e. learning structural patterns that summarize the data.
A syntactic profile is a disjunction of regex-like patterns that describe all of the syntactic variations in the data.
Each pattern succinctly describes a specific variation,
and is defined by a sequence of \emph{atomic patterns} or \emph{atoms}, such as digits or letters.

While existing tools, such as Microsoft SQL Server Data Tools (\SSDT)~\citep{web.ssdt},
and \Ataccama~\citep{web.ataccama} allow pattern-based profiling,
they generate a single profile that cannot be customized.
In particular,
\begin{inlist}
    \item they use a small predetermined set of atoms,
          and do not allow users to supply custom atoms specific to their domains, and
    \item they provide little support for controlling granularity, i.e. the number of patterns in the profile.
\end{inlist}

We present a novel application of program synthesis techniques to addresses these two key issues.
Our implementation, \FlashProfile, supports custom user-defined atoms
that may encapsulate arbitrary pattern-matching logic,
and also allows users to interactively control the granularity of generated profiles,
by providing desired bounds on the number of patterns.
%Clustering provides control over granularity,
%and synthesis techniques allow for efficiently learning patterns over complex atoms.

\paragraph{A Motivating Example}
\Cref{tab:intro-data} shows a fragment of a dataset containing a set of references in various formats,
and its profiles generated by \Ataccama (in \cref{fig:intro-ataccama-profile}),
\tool{Microsoft} \SSDT (in \cref{fig:intro-ssdt-profile}), and our tool \FlashProfile (in \cref{fig:intro-base-token-profile}).
Syntactic profiles expose rare variations that are hard to notice by manual inspection of the data,
or from simple statistical properties such as distribution of string lengths.
For example, \Ataccama reveals a suspicious pattern \stringliteral{W\_W}, which matches less than 0.5\% of the dataset.
\SSDT, however, groups this together with other less frequent patterns into a \stringliteral{.*} pattern.
Since \SSDT does not provide a way of controlling the granularity of the profile,
a user would be unable to further refine the \stringliteral{.*} pattern.
\FlashProfile shows that this pattern actually corresponds to missing entries, which read \stringliteral{not\_available}.

% HACK: Not sure why the number was being incremented by 2.
\addtocounter{footnoteNum}{-1}

\begin{figure}[t]
    \vspace*{-3pt}\centering\small
    \subcaptionbox{Sample data\label{tab:intro-data}}{
    \begin{subfigure}[t]{0.2\linewidth}
        \centering\smaller[1.5]
        \hspace*{-12pt}\texttt{
        \def\arraystretch{0.95}
        \begin{tabular}{|p{2.5cm}|} \hline
            \rowcolor{bordercolor}
            \textsf{Reference ID} \\\hline
            ISBN:\visiblespace1-158-23466-X\\\hline
            not\_available\\\hline
            doi:\visiblespace10.1016/S1387- 7003(03)00113-8\\\hline
            \rule{0pt}{2ex}{\smaller[4]\smash{$\vdots$}}\\\hline
            PMC9473786\\\hline
            ISBN:\visiblespace0-006-08903-1\\\hline
            doi:\visiblespace\visiblespace\\10.13039/100005795\\\hline
            PMC9035311\\\hline
            \rule{0pt}{2ex}{\smaller[4]\smash{$\vdots$}}\\\hline
            PMC5079771\\\hline
            ISBN:\visiblespace2-287-34069-6\\\hline
        \end{tabular}}
    \end{subfigure}
    }\hspace*{8pt}\begin{subfigure}[b]{0.37\linewidth}
    \subcaptionbox{Profile from Ataccama One\label{fig:intro-ataccama-profile}}{
        \begin{subfigure}[t]{\linewidth}\smaller
            \renewcommand\labelitemi{{\boldmath$\cdot$}}
            \begin{itemize}[leftmargin=2mm]
                \item \texttt{W\_W} \hfill (5)
                \item \texttt{W: N.N/LN-N(N)N-D} \hfill (11)
                \item \texttt{W: D-N-N-L} \hfill (34)
                \item \texttt{W: N.N/N} \hfill (110)
                \item \texttt{W: D-N-N-D} \hfill (267)
                \item \texttt{WN} \hfill (1024)
            \end{itemize}
            \begin{center}
                Classes:\: [\texttt{L}]etter,\:[\texttt{W}]ord,\:[\texttt{D}]igit,\:[\texttt{N}]umber
            \end{center}
        \end{subfigure}
    }\\[5pt]
    \subcaptionbox{Profile from \tool{Microsoft} \SSDT\label{fig:intro-ssdt-profile}}{
        \begin{subfigure}[t]{\linewidth}\smaller
            \renewcommand\labelitemi{{\boldmath$\cdot$}}
            \begin{itemize}[leftmargin=2mm]
                \item \texttt{doi:$\Space$+10\bs{}.\bs{}d\bs{}d\bs{}d\bs{}d\bs{}d/\bs{}d+} \hfill (110)
                \item \texttt{.*} \hfill (113)
                \item \texttt{ISBN:\Space\!0-\bs{}d\bs{}d\bs{}d-\bs{}d\bs{}d\bs{}d\bs{}d\bs{}d-\bs{}d} \hfill (204)
                \item \texttt{PMC\bs{}d$^+$} \hfill (1024)
            \end{itemize}
        \end{subfigure}
    }
    \end{subfigure}\hspace*{15pt}\subcaptionbox{Default profile from \FlashProfile\label{fig:intro-base-token-profile}}{
        \begin{subfigure}[t]{0.35\linewidth}\smaller
            \renewcommand\labelitemi{{\boldmath$\cdot$}}
            \begin{itemize}[leftmargin=2mm,itemsep=3pt]
                \item $\sstringliteral{not\_available}$ \hfill (5)
                \item $\sstringliteral{doi:} \sconcat \sSpace^+ \sconcat \sstringliteral{10.1016/} \sconcat \sUpper \sconcat \srep{\sDigit}{4} \sconcat \sstringliteral{-} \sconcat \srep{\sDigit}{4} \sconcat$\\
                      $\sstringliteral{(} \sconcat \srep{\sDigit}{2} \sconcat \sstringliteral{)} \sconcat \srep{\sDigit}{5} \sconcat \sstringliteral{-} \sconcat \sDigit$ \hfill (11)
                \item $\sstringliteral{ISBN:} \sconcat \sSpace \sconcat \sDigit \sconcat \sstringliteral{-} \sconcat \srep{\sDigit}{3} \sconcat \sstringliteral{-} \sconcat \srep{\sDigit}{5} \sconcat \sstringliteral{-X}$ \hfill (34)
                \item $\sstringliteral{doi:} \sconcat \sSpace^+ \sconcat \sstringliteral{10.13039/} \sconcat \sDigit^+$ \hfill (110)
                \item $\sstringliteral{ISBN:} \sconcat \sSpace \sconcat \sDigit \sconcat \sstringliteral{-} \sconcat \srep{\sDigit}{3} \sconcat \sstringliteral{-} \sconcat \srep{\sDigit}{5} \sconcat \sstringliteral{-} \sconcat \sDigit$ \hfill (267)
                \item $\sstringliteral{PMC} \sconcat \srep{\sDigit}{7}$ \hfill (1024)
            \end{itemize}\vspace*{2pt}

            \begin{center}
                Classes:\: [\texttt{U}]ppercase,\:[\texttt{D}]igit
            \end{center}\vspace*{2pt}

            Superscripts indicate repetition of atoms.\\
            Constant strings are surrounded by quotes.
        \end{subfigure}\vspace*{2pt}
    }
    \captionsetup{skip=1.75pt}
    \caption{Profiles for a set of references\protect\footnotemarkNum ---
             ~number of matches for each pattern is shown on the right\vspace*{-6.75pt}}
    \label{fig:intro-example}
\end{figure}

\footnotetextNum{The full dataset is available at \redact{\github{motivating_example.json}}.}

For this dataset, although \Ataccama suggests a profile of the same granularity as from \FlashProfile,
the patterns in the profile are too coarse to be immediately useful.
For instance, it may not be immediately obvious that the pattern \texttt{W: D-N-N-L} maps to ISBNs in the dataset.
\FlashProfile further qualifies the \texttt{W} (word) to the constant \stringliteral{ISBN},
and restricts the [\texttt{N}]umber patterns to $\srep{\sDigit}{3}$ (short for \rep{\Digit}{3}) and $\srep{\sDigit}{5}$,
and the final [\texttt{L}]etter to the constant \stringliteral{X}.

\begin{wrapfigure}{r}{0.5\textwidth}
    \vspace*{6pt}\centering\small
    \begin{subfigure}[b]{0.96\linewidth}\smaller
        \renewcommand\labelitemi{{\boldmath$\cdot$}}
        \begin{itemize}[leftmargin=2mm,itemsep=1pt]
            \item $\sstringliteral{not\_available}$ \hfill (5)
            \item $\sstringliteral{doi:} \sconcat \sSpace^+ \sconcat \DOI$ \hfill (121)
            \item $\sstringliteral{ISBN:} \sconcat \sSpace \sconcat \ISBN$ \hfill (301)
            \item $\sstringliteral{PMC} \sconcat \srep{\sDigit}{7}$ \hfill (1024)
        \end{itemize}
        \caption{Auto-suggested profile from \FlashProfile}
        \label{fig:intro-extended-profile}
    \end{subfigure}\\[3pt]
    \begin{subfigure}[b]{0.96\linewidth}\smaller
        \renewcommand\labelitemi{{\boldmath$\cdot$}}
        \begin{itemize}[leftmargin=2mm,itemsep=0.5pt]
            \item $\sstringliteral{not\_available}$ \hfill (5)
            \item $\sstringliteral{doi:} \sconcat \sSpace^+ \sconcat \sstringliteral{10.1016/} \sconcat \sUpper \sconcat \srep{\sDigit}{4} \sconcat \sstringliteral{-} \sconcat \srep{\sDigit}{4} \sconcat \sstringliteral{(} \sconcat \srep{\sDigit}{2} \sconcat \sstringliteral{)} \sconcat \srep{\sDigit}{5} \sconcat \sstringliteral{-} \sconcat \sDigit$ \hfill (11)
            \item $\sstringliteral{doi:} \sconcat \sSpace^+ \sconcat \sstringliteral{10.13039/} \sconcat \sDigit^+$ \hfill (110)
            \item $\sstringliteral{ISBN:} \sconcat \sSpace \sconcat \ISBN$ \hfill (301)
            \item $\sstringliteral{PMC} \sconcat \srep{\sDigit}{7}$ \hfill (1024)
        \end{itemize}
        \caption{A refined profile on requesting 5 patterns}
        \label{fig:intro-refined-profile}
    \end{subfigure}
    % \captionsetup{skip=3pt}
    \caption{Custom atoms,\!\protect\FootnotemarkNum{\ref{foot:ISBN_DOI}}
             and refinement of profiles\vspace*{4pt}}
    \label{fig:intro-more-profiles}
\end{wrapfigure}

\FlashProfile also allows users familiar with their domains to define custom patterns,
that cluster data in ways that are specific to the domain.
For example, the two patterns for \stringliteral{doi} in \cref{fig:intro-base-token-profile} are vastly different ---
one contains letters and parentheses, whereas the other contains only digits.
However, grouping them together makes the profile more readable,
and helps spot outliers differing from the expected patterns.
\Cref{fig:intro-extended-profile} shows a profile suggested by \FlashProfile
when provided with two custom atoms: \DOI and \ISBN,\nsp\footnoteNum{\label{foot:ISBN_DOI}
  \hspace*{-2pt}\DOI is defined as the regex \texttt{10.\bs{}d\{4,9\}/[-.\_;()/:A-Z0-9a-z]+}.\\[1pt]
  \ISBN is defined as the regex \texttt{\bs{}d-\bs{}d\{3\}-\bs{}d\{5\}-[0-9Xx]}.
} with appropriate costs.
Users may refine the profile to observe more specific variations within the DOIs and ISBNs.
On requesting one more pattern, \FlashProfile unfolds \DOI,
since the DOIs are \emph{more dissimilar} to each other than ISBNs,
and produces the profile shown in \Cref{fig:intro-refined-profile}.

\paragraph{Key Challenges}
A key barrier to allowing custom atoms is the large search space for the desirable profiles.
Prior tools restrict their atoms to letters and digits,
followed by simple upgrades such as sequences of digits to numbers, and letters to words.
However, this simplistic approach is not effective
in the presence of several overlapping atoms and complex pattern-matching semantics.
Moreover, a na\"ive exhaustive search over all profiles is prohibitively expensive.
Every substring might be generalized in multiple ways into different atoms,
and the search space grows exponentially when composing patterns as sequences of atoms,
and a profile as a disjunction of patterns.

One approach to classifying strings into matching patterns might be to construct
decision trees or random forests~\citep{breiman2001random} with features based on atoms.
However features are typically defined as predicates over entire strings,
whereas atoms match specific substrings and may match multiple times within a string.
Moreover, the location of an atomic match within a string depends
on the lengths of the preceding atomic matches within that string.
Therefore, this approach seems intractable since generating features based on atoms
leads to an exponential blow up.

Instead, we propose to address the challenge of learning a profile by first \emph{clustering}~\citep{xu2005survey} ---
partitioning the dataset into \emph{syntactically similar} clusters of strings
and then learning a succinct pattern describing each cluster.
This approach poses two key challenges:
\begin{inlist}
    \item efficiently learning patterns for a given cluster of strings over an arbitrary set of atomic patterns provided by the user, and
    \item defining a suitable notion of \emph{pattern-based similarity} for clustering,
          that is aware of the user-specified atoms
\end{inlist}.
For instance, as we show in the motivating example (\cref{fig:intro-example} and \cref{fig:intro-more-profiles}),
the clustering must be sensitive to the presence of \DOI and \ISBN atoms.
Traditional character-based similarity measures over strings~\citep{gomaa2013survey}
are ineffective for imposing a clustering that is susceptible to high-quality explanations using a given set of atoms.

\paragraph{Our Technique}
We address both the aforementioned challenges
by leveraging recent advances in \emph{inductive program synthesis}~\citep{gulwani2017program} ---
an approach for learning programs from incomplete specifications, such as input-output examples for the desired program.

First, to address challenge \inliststyle{(1)}, we present a novel domain-specific language (DSL) for patterns,
and define a specification over a given set of strings.
Our DSL provides constructs that allow users to easily augment it with new atoms.
We then give an efficient synthesis procedure for learning patterns that are consistent with the specification,
and a cost function to select compact patterns that are not overly general,
out of all patterns that are consistent with a given cluster of strings.

Second, we observe that the cost function for patterns induces
a natural \emph{syntactic dissimilarity measure} over strings,
which is the key to addressing challenge \inliststyle{(2)}.
We consider two strings to be similar if both can be described by a low-cost pattern.
Strings requiring overly general / complex patterns are considered dissimilar.
Typical clustering algorithms require computation of all pairwise dissimilarities~\citep{xu2005survey}.
However, in contrast to standard clustering scenarios,
computing dissimilarity for a pair of strings not only gives us a numeric measure, but also a pattern for them.
That this allows for practical performance optimizations.
In particular, we present a strategy to approximate dissimilarity computations
using a small set of carefully sampled patterns.

%The first challenge stems from the fact that classic character-based measures do not capture syntactic
%dissimilarity accurately (details in \cref{subsec:eval-syntactic-similarity}).
%Our main observation is that the desired measure is not a property of the strings per se,
%but of the patterns describing them, within those supplied by the user.

%However, a large number of patterns may be consistent with a pair of strings,
%and the challenge is to identify the most suitable one amongst them.
%Existing tools either
%\begin{inlist}
%    \item avoid overlapping patterns (e.g., digits and letters are the building blocks, in both \SSDT and \Ataccama), or
%    \item only support simple generalizations (e.g. contiguous digits $\rightarrow$ number, contiguous letters $\rightarrow$ word)
%\end{inlist}.

\tikzstyle{block} = [rectangle, draw, fill=olive!12, text width=5.5em,
                     text centered, rounded corners, minimum height=4em]
\tikzstyle{arrow} = [draw, bend left, line width=0.16mm, -{Latex[length=2mm,width=2mm]}]

\begin{wrapfigure}{r}{0.45\textwidth}
    \vspace*{-1pt}\centering
    \resizebox{0.91\linewidth}{!}{\hspace*{-8pt}
      \begin{tikzpicture}[node distance = 9em]
        \node [text width = 7em, text centered] (approx) {\large Approximation Parameters};
        \node [above of = approx, node distance = 3.6em, text width=7em, text centered] (atoms)  {\large Custom Atoms with Costs};
        \node [above of = atoms, node distance = 3.6em, text width=7em, text centered] (bounds)  {\large Number of Patterns};

        \node [block, right of = atoms, minimum height = 6.25em] (learning) {Pattern Learner\\+\\Cost Function};
        \node [block, right of = learning] (clustering) {Hierarchical Clustering};
        \node [above right = -1.75em and 1.75em of clustering, minimum height = 1.6em] (profile) {\Large Profile};
        \node [below of = profile, node distance = 2.25em, minimum height = 1.6em] (dataset) {\Large Dataset};

        \path [arrow] (bounds) -| (clustering);
        \path [arrow] (atoms) -- (learning);
        \path [arrow] (approx) -| ($(clustering.south)-(0.4,0)$);
        \path [arrow, line width=0.65mm, -{Latex[length=2.5mm,width=3.5mm]}] (dataset) -- ++(0em,-2.25em) -| ($(clustering.south)+(0.4,0)$);

        \path [arrow, line width=0.65mm, dashed, transform canvas={yshift=0.75em}, -{Latex[length=2.5mm,width=3.5mm]}] (learning) -- (clustering);
        \path [arrow, line width=0.65mm, dashed, transform canvas={yshift=-0.75em}, -{Latex[length=2.5mm,width=3.5mm]}] (clustering) -- (learning);

        \path [arrow, line width=0.65mm, -{Latex[length=2.5mm,width=3.5mm]}] (learning) -- ++(0em,4.25em) -| (profile);
      \end{tikzpicture}}
    \captionsetup{skip=2pt}
    \caption{\FlashProfile's interaction model: \textit{dashed} edges denote internal communication,
             and \textit{thin} edges denote optional parameters to the system.}
    \label{fig:intro-interaction}
\end{wrapfigure}
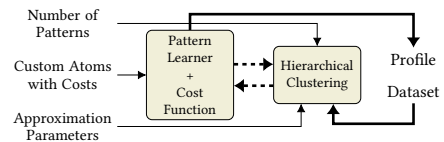

To summarize, we present a framework for syntactic profiling based on clustering,
that is parameterized by a pattern learner and a cost function.
\Cref{fig:intro-interaction} outlines our interaction model.
In the default mode, users simply provide their dataset.
Additionally, they may control the performance \emph{vs.} accuracy trade-off,
define custom atoms, and provide bounds on the number of patterns.
To enable efficient refinement of profiles based on the given bounds,
we construct a \emph{hierarchical clustering}~\citep[Section IIB]{xu2005survey}
that may be \emph{cut} at a suitable height to extract the desired number of clusters.

\paragraph{Evaluation}
We have implemented our technique as \FlashProfile using \PROSE~\citep{web.prose},
also called \FlashMeta~\citep{polozov2015flashmeta},
a state-of-the-art inductive synthesis framework.
We evaluate our technique on \numDatasets publicly-available datasets collected from online sources.\nsp\footnoteNum{
  All public datasets are available at: \redact{\github{tests}}.}
Over \numTasks tasks, \FlashProfile achieves a median profiling time of 0.7s, 77\% of which complete in under 2s.
We show a thorough analysis of our optimizations,
and a comparison with state-of-the-art tools.

\paragraph{Applications in PBE Systems}
The benefits of syntactic profiles extend beyond data understanding.
An emerging technology, programming by examples (PBE)~\citep{lieberman2001your, gulwani2017program},
provides end users with powerful semi-automated alternatives to manual data wrangling.
For instance, they may use a tool like \FlashFill~\citep{gulwani2011automating},
a popular PBE system for data transformations within \Excel and \AML~\citep{web.azure-ml, web.azure-ml-talk}.
However, a key challenge to the success of PBE is finding a representative set of examples
that best discriminates the desired program from a large space of possible programs~\citep{mayer2015user}.
Typically users provide the desired outputs over the first few entries, and
\FlashFill then synthesizes the \emph{simplest generalization} over them.
However, this often results in incorrect programs,
if the first few entries are not representative of the various formats present in the entire dataset~\citep{mayer2015user}.

Instead, a syntactic profile can be used to select a representative set of examples from syntactically dissimilar clusters.
We tested \numFFTotalCond scenarios where \FlashFill requires more than one input-output example to learn the desired transformation.
In \numFFRecallCondPercent of them, the examples belong to different syntactic clusters identified by \FlashProfile.
Moreover, we show that a \emph{profile-guided interaction model} for \FlashFill,
which we detail in \cref{sec:applications-in-pbe-systems},
is able to complete \numFFPrecisionCondPercent of these tasks requiring the \emph{minimum} number of examples.
Instead of the user having to select a representative set of examples,
our dissimilarity measure allows for proactively requesting the user to provide
the desired output on an entry that is \emph{most discrepant} with respect to those previously provided.

%these representative examples may be identified from the various syntactic clusters in the dataset. \FlashProfile indicates that the required examples belong to different syntactic clusters, for \fix{[Updated FF recall numbers from Daniel]}\% of the \FlashFill benchmarks which need than one example to learn the desired transformation. \FlashProfile therefore enables a better, targeted interaction model for \FlashFill\ -- requesting users for examples from different clusters. In fact, to our surprise, incorporating the dissimilarity measure from \FlashProfile, dramatically reduces the number of examples required to learn the desired transformation. For the cases where \FlashFill requires multiple examples to learn the desired transformation, we show that, instead of providing the output for the next input where \FlashFill erred, if we provide the output for the input which is syntactically most dissimilar with the previous examples, we reduce the number of examples required to learn the desired transformation for \fix{[Updated FF precision numbers from Daniel]}\% of the benchmarks.

\vspace{2pt}\noindent
In summary, we make the following major contributions:
\begin{itemize}[leftmargin=8.5mm,itemsep=0.5pt,topsep=0.5pt]
  \item[(\,\cref{sec:overview}\,)]
    We formally define syntactic profiling as a problem of clustering of strings,
    followed by learning a succinct pattern for each cluster.

  \item[(\,\cref{sec:hierarchical-clustering}\,)]
    We show a hierarchical clustering technique that uses pattern learning to
    measure dissimilarity of strings,
    and give performance optimizations that further exploit the learned patterns.

  \item[(\,\cref{sec:pattern-synthesis}\,)]
    We present a novel DSL for patterns, and give an efficient synthesis procedure
    with a cost function for selecting desirable patterns.

  \item[(\,\cref{sec:evaluation}\,)]
    We evaluate \FlashProfile's performance and accuracy on large real-life datasets,
    and provide a detailed comparison with state-of-the-art tools.

  \item[(\,\cref{sec:applications-in-pbe-systems}\,)]
    We present a profile-guided interaction model for \FlashFill,
    and show that data profiles may aid PBE systems by identifying a representative set of inputs.
\end{itemize}

\section{Overview} \label{sec:overview}

Henceforth, the term \emph{dataset} denotes a set of strings.
We formally define a syntactic profile as:

\begin{defn}{Syntactic Profile}
    Given a dataset \dataset and a desired number $k$ of patterns, syntactic profiling involves learning
    \begin{inlist}
        \item a partitioning $\dataset_1 \sqcup \ldots \sqcup \dataset_k = \dataset$, and
        \item a set of patterns $\{\pattern_1, \ldots, \pattern_k\}$,
              where each $\pattern_i$ is an expression that describes the strings in $\dataset_i$.
    \end{inlist}
    We call the disjunction of these patterns
    $\profile = \pattern_1 \vee \ldots \vee \pattern_k$
    a \emph{syntactic profile} of \dataset, which describes all the strings in \dataset.
\end{defn}

The goal of syntactic profiling is to learn a set of patterns that summarize a given dataset,
but is neither too specific nor too general (to be practically useful).
For example, the dataset itself is a trivial overly specific profile,
whereas the regex \stringliteral{.*} is an overly general one.
We propose a technique that leverages the following two key subcomponents to generate and rank profiles:\nsp\footnoteNum{
    We denote the universe of all strings as \Strings, the set of non-negative reals as \NonNegReals, and the power set of a set $X$ as $2^X$.
}
\begin{itemize}[leftmargin=5mm,topsep=1pt,itemsep=0.75pt]
    \item a pattern learner $\learner\colon 2\,^\Strings \to 2\,^\dsl$,
          which generates a set of patterns over an arbitrary pattern language \dsl,
          that are consistent with a given dataset.
    \item a cost function $\cost\colon \dsl \,\times\, 2\,^\Strings \to \NonNegReals$,
          which quantifies the suitability of an arbitrary pattern (in the same language \dsl)
          with respect to the given dataset.
\end{itemize}
Using \learner and \cost, we can quantify the suitability of clustering a set of strings together.
More specifically, we can define a minimization objective $\objective\colon 2\,^\Strings \to \NonNegReals$
that indicates an aggregate cost of a cluster.
We can now define an \emph{optimal syntactic profile} that minimizes $\objective$ over a given dataset \dataset:

\begin{defn}{Optimal Syntactic Profile}
    Given a dataset \dataset, a desired number $k$ of patterns,
    and access to a pattern learner \learner, a cost function \cost for patterns,
    and a minimization objective \objective for partitions, we define:
    \begin{inlist}
        \item the optimal partitioning $\widetilde{\dataset}_{opt}$ as one that
              minimizes the objective function \objective over all partitions, and
        \item the optimal syntactic profile $\profile_{opt}$ as the disjunction
              of the least-cost patterns describing each partition in $\widetilde{\dataset}_{opt}$
    \end{inlist}. Formally,
    $$
        \resizebox{!}{1.5\baselineskip}{$\displaystyle
            \widetilde{\dataset}_{opt}
                \bydef \argmin_{\substack{\{\dataset_1, \,\ldots\,, \dataset_k\} \\ \text{s.t. } \dataset \,=\, \bigsqcup\limits_{i \,=\, 1}^k \dataset_i}}\;
                    \sum\limits_{i \,=\, 1}^k\; \objective(\dataset_i)$}
        \qquad \text{and} \qquad
            \profile_{opt}
                \bydef \bigvee_{\dataset_i \,\in\, \widetilde{\dataset}_{opt}}
                    \argmin_{\pattern \,\in\, \learner(\dataset_i)}\;
                        \cost(\pattern, \dataset_i)
    $$
\end{defn}

Ideally, we would define the aggregate cost of a partition as the minimum cost incurred by a pattern that describes it entirely.
This is captured by $\objective(\dataset_i) \bydef \min_{\pattern \,\in\, \learner(\dataset_i)} \cost(\pattern, \dataset_i)$.
However, with this objective, computing the optimal partitioning $\widetilde{\dataset}_{opt}$ is intractable in general.
For an arbitrary learner \learner and cost function \cost,
this would require exploring all $k$-partitionings of the dataset \dataset.\nsp\footnoteNum{
    The number of ways to partition a set \dataset into $k$ non-empty subsets is given by
    Stirling numbers of the second kind~\citep{graham1994concrete}, \resizebox{!}{\baselineskip}{$\braceVectorstack{|\dataset| k}$}.
    When $k \ll |\dataset|$, the asymptotic value of \resizebox{!}{\baselineskip}{$\braceVectorstack{|\dataset| k}$}
    is given by \resizebox{!}{\baselineskip}{$\displaystyle\frac{k^{|\dataset|}}{k!}$}.
}

Instead, we use an objective that is tractable and works well in practice ---
the aggregate cost of a cluster is given by
the maximum cost of describing any two strings belonging to the cluster, using the best possible pattern.
Formally, this objective is $\widehat{\objective}(\dataset_i) \bydef \max_{x,y \,\in\, \dataset_i} \min_{\pattern \,\in\, \learner(\{x,y\})} \cost(\pattern, \{x,y\})$.
This objective is inspired by the \emph{complete-linkage} criterion~\citep{sorensen1948method},
which is widely used in clustering applications across various domains~\citep{jain1999data}.
To minimize $\widehat{\objective}$, it suffices to only compute the costs of describing (at most) $|\dataset|^2$  pairs of strings in \dataset.

\begin{wrapfigure}{r}{0.48\textwidth}
    \vspace*{4pt}\algobox{1.225}{
    \begin{algfunction}
      {\WithLCParam{\Profiling}}
      {\dataset\colon \Tstring[\,], \minClusters\colon \Tint, \maxClusters\colon \Tint, \edgeSampleFactor\colon \Tdouble}
      {\profile, a partitioning of \dataset with the associated patterns\\\hspace*{-3pt}
                 for each partition, such that $\minClusters \leqslant |\: \profile \:| \leqslant \maxClusters$}
        \State $\hierarchy \gets \Call{\WithLCParam{\BuildHierarchy}}{\dataset, \maxClusters, \edgeSampleFactor} \quad;\quad \profile \gets \{\}$
        \ForAll{$X \in \Call{\ClusteringAlgo}{\hierarchy, \minClusters, \maxClusters}$}
          \State $\langle \prop{Pattern:}\;\pattern, \prop{Cost:}\;c \rangle \gets \Call{\WithLCParam{\BestPattern}}{X}$
          \State $\profile \gets \profile \cup \{ \langle \prop{Data:}\;X, \prop{Pattern:}\;\pattern \rangle \}$
        \EndFor
        \Return{\profile}
    \end{algfunction}}
    \captionsetup{skip=1pt}
    \caption{Our main profiling algorithm}
    \label{fig:profiling-algo}
\end{wrapfigure}

We outline our main algorithm \ProfilingAlgo in \cref{fig:profiling-algo}.
It is parameterized by an arbitrary learner \learner and cost function \cost.
\ProfilingAlgo accepts a dataset \dataset,
the bounds $\big[\minClusters, \maxClusters\big]$ for the desired number of patterns,
and a sampling factor \edgeSampleFactor that decides the efficiency \emph{vs.} accuracy trade-off.
It returns the generated partitions paired with the least-cost patterns describing them:
$\{ \langle \dataset_1, \pattern_1 \rangle, \ldots, \langle \dataset_k, \pattern_k \rangle \}$,
where $\minClusters \leqslant k \leqslant \maxClusters$.

At a high level, we partition a dataset using the cost of patterns to induce a syntactic dissimilarity measure over its strings.
For large enough \edgeSampleFactor, we compute all $O(|\dataset|^2)$ pairwise dissimilarities,
and generate the partitioning $\widetilde{\dataset}_{opt}$ that minimizes $\widehat{\objective}$.
However, many large real-life datasets have a very small number of syntactic clusters,
and we notice that we can closely approximate $\widetilde{\dataset}_{opt}$ by sampling only a few pairwise dissimilarities.
We invoke \BuildHierarchyAlgo, in line 1, to construct a hierarchy \hierarchy over \dataset
with accuracy controlled by \edgeSampleFactor.
The hierarchy \hierarchy is then \emph{cut} at a certain height to obtain $k$ clusters by calling \ClusteringAlgo\ in line 2 ---
if $m \neq M$, $k$ is heuristically decided based on the quality of clusters obtained at various heights.
Finally, using \BestPatternAlgo, we learn a pattern \pattern for each cluster $X$, and add it to the profile \profile.

In the following subsections, we explain the two main components:
\begin{inlist}
    \item \BuildHierarchyAlgo for building a hierarchical clustering, and
    \item \BestPatternAlgo for pattern learning
\end{inlist}.

\subsection{Pattern-Specific Clustering}\label{subsec:overview-clustering}

\BuildHierarchyAlgo uses an agglomerative hierarchical clustering (AHC)~\citep[Section IIB]{xu2005survey}
to construct a hierarchy (also called a dendrogram)
that depicts a nested grouping of the given collection of strings, based on their syntactic similarity.
\cref{fig:example-hierarchy} shows such a hierarchy over an incomplete and inconsistent dataset containing years,
using the default set of atoms listed in \cref{tab:lp-default-atoms}.
Once constructed, a hierarchy may be split at a suitable height to extract clusters of desired granularity,
which enables a natural form of refinement --- supplying a desired number of clusters.
In contrast, flat clustering methods like \kmeans~\citep{macqueen1967some} generate a fixed partitioning within the same time complexity.
% \FlashProfile also supports fully automatic profiling using a simple heuristic to suggest a partitioning.
In \cref{fig:example-hierarchy-splits}, we show a heuristically suggested split with 4 clusters,
and a refined split on a request for 5 clusters.
A key challenge to clustering is defining an appropriate pattern-specific measure of dissimilarity over strings,
as we show below.

\begin{exmp}
    Consider the pairs: $p = \stringPair{1817}{1813?}$ and $q = \stringPair{1817}{1907}$.
    Selecting the pair that is syntactically \emph{more similar} is ambiguous, even for humans.
    The answer depends on the user's application ---
    it may make sense to either cluster homogeneous strings (containing only digits) together,
    or to cluster strings with a longer common prefix together.

    A natural way to resolve this ambiguity is to allow users to express
    their application-specific preferences by providing custom atoms,
    and then to make the clustering algorithm sensitive to the available atoms.
    Therefore we desire a dissimilarity measure that incorporates the user-specified atoms,
    and yet remains efficiently computable,
    since typical clustering algorithms compute dissimilarities between all pairs of strings~\citep{xu2005survey}.

    % As shown in \cref{fig:example-hierarchy-splits}, the {\small\texttt{\{1813...1898\}}} cluster
    % is linked to {\small\texttt{\{1900...1903\}}} first, since they can all be described with a natural pattern,
    % \patternmode{\stringliteral{1} \concat \rep{\Digit}{3}}.
    % Although years followed by \patternmode{\stringliteral{?}}, such as {\small\texttt{1850?}},
    % are more similar to {\small\texttt{\{1813...1898\}}} character-wise,
    % an overly general pattern \patternmode{\stringliteral{18} \concat \Any^+}
    % is needed to describe them together\footnoteNum{
    %     We do not allow atoms to match empty strings,
    %     thereby forbidding patterns such as $\stringliteral{18} \concat \rep{\Digit}{2} \concat \Punct^?$.
    %     We discuss our language for patterns, and this design choice in \cref{subsec:lp-syntax-semantics}.
    % }, using our default set of atoms provided in \cref{tab:lp-default-atoms}.

    % However, our clustering must be sensitive to custom atoms.
    % A user may desire to cluster all the 1800s years, with and without \stringliteral{?},
    % together for her application. She may provide an appropriate atom,
    % for example using the regex \patternmode{\stringliteral{18\textbackslash{}d\textbackslash{}d[?]?}}.
    % In such a case, our hierarchy construction procedure should make use this new atom
    % to link all the 1800s years \emph{before} linking the 1900s.
\end{exmp}

\paragraph{Syntactic Dissimilarity}
Our key insight is to leverage program synthesis techniques to efficiently learn patterns describing a given set of strings,
and induce a dissimilarity measure using the learned patterns ---
overly general or complex patterns indicate a high degree of syntactic dissimilarity.

In \cref{subsec:dissimilarity}, we formally define the dissimilarity measure \SyDissSymb\,
as the minimum cost incurred by any pattern for describing a given pair of strings,
using a specified pattern learner \learner and cost function \cost.
%Notice that, \learner and \cost are the only components required to implement our profiling technique.
We evaluate our measure \SyDissSymb\ in \cref{subsec:eval-syntactic-similarity},
and demonstrate that for estimating syntactic similarity it is superior
to classical character-based measures~\citep{gomaa2013survey},
and simple machine-learned models such as random forests based on intuitive features.

% HACK: Not sure why the number was being incremented by 2.
\addtocounter{footnoteNum}{-1}

\newcommand{\tokendesc}[1]{{\color{black!48!white} $\mathsmaller{#1}$}}

\begin{figure}[t]
    \vspace*{1pt}
    \begin{minipage}{0.655\textwidth}
    \centering\small
    \subcaptionbox{Dataset\,\protect\footnotemarkNum\label{tab:years-example}}{
      \begin{subfigure}[t]{0.16\linewidth}
      \centering
      \texttt{\smaller[1.5]
      \def\arraystretch{1.025}
      \hspace*{-10pt}\begin{tabular}{|c|} \hline
      \rowcolor{bordercolor}
      \textbf{Year} \\\hline
      1900\\\hline
      1877\\\hline
      $\epsilon$\\\hline
      1860\\\hline
      ?\\\hline
      1866\\\hline
      $\epsilon$\\\hline
      1893\\\hline
      \rule{0pt}{3ex}{\small\smash{$\vdots$}}\\\hline
      1888\,?\\\hline
      1872\\\hline
      \end{tabular}}
      \end{subfigure}
    }\hspace*{4pt}\subcaptionbox{\vspace*{4pt}A hierarchy based on default atoms from \cref{tab:lp-default-atoms}\label{fig:example-hierarchy-splits}}{
    \resizebox{0.75\linewidth}{!}{%
      \forestset{
        roof/.style={
          parent anchor=south,
          child anchor=north,
          edge path={
            \noexpand\path[\forestoption{edge}]
            (!u.parent anchor) -- ([xshift=-28pt].child anchor) -- ([xshift=28pt] .child anchor) -- cycle
            \forestoption{edge label};
          },
        },
      }
    \begin{forest}
      for tree={
        l sep-=.6em,l=0,
        edge path={\noexpand\path [\forestoption{edge}] (!u.parent anchor) -- ++(0,-4pt) -| (.child anchor)\forestoption{edge label};},
        parent anchor=south,
      }
      [\patfail
        [{\large $\Any^+$}, edge={decorate, decoration={border, segment length=4pt, amplitude=3pt}}
          [{\large $\stringliteral{1} \concat \Any^+$}, l+=10mm
            [{\large $\stringliteral{1} \concat \rep{\Digit}{3}$}, name=years, l+=18mm
              [{\large $\stringliteral{18} \concat \rep{\Digit}{2}$}, l+=13mm
                [{\large $\texttt{1813}\;\cdots\;\texttt{1898}$},
                edge label={node[midway,above, yshift=-10pt,xshift=-14pt]{$\iddots\vdots\ddots$}},
                roof, tier=string] ]
              [{\large $\stringliteral{190} \concat \Digit$}, l+=17mm
                [{\large $\texttt{1900}\;\cdots\;\texttt{1903}$},
                edge label={node[midway,above, yshift=-8pt,xshift=-14pt]{$\iddots\vdots\ddots$}},
                roof, tier=string] ] ]
            [{\large $\stringliteral{18} \concat \rep{\Digit}{2} \concat \stringliteral{?}$}, l+=33mm
              [{\large $\texttt{1850?}\,\cdots\,\texttt{1875?}$},
              edge label={node[midway,above, yshift=-8pt,xshift=-14pt]{$\iddots\vdots\ddots$}},
              roof, tier=string] ] ]
          [{\large $\stringliteral{?}$}, l+=47mm
            [{\large \texttt{?}}, tier=string]] ]
        [{\large $\EmptySymb$}, l+=47mm,
        edge={decorate, decoration={border, segment length=3pt, amplitude=2.25pt, angle=-45}}
          [{\large $\epsilon$}, tier=string]]]
      \begin{scope}[line width=0.8pt,ForestGreen]
          \draw[decorate,decoration={snake, segment length=10pt, amplitude=2pt}]
               ([shift={(-1.6,0.45)}]years.north west) -- ([shift={(7.25,0.45)}]years.north east);
          \node[above left = 5mm and -2mm of years] {\large\bf\sf Suggested};
      \end{scope}
      \begin{scope}[line width=0.8pt,NavyBlue]
          \draw[decorate,decoration={snake, segment length=6pt, amplitude=2pt}]
               ([shift={(-1.6,-1.05)}]years.north west) -- ([shift={(7.25,-1.05)}]years.north east);
          \node[below left = -1.5mm and 3mm of years] {\large\bf\sf Refined};
      \end{scope}
    \end{forest}}}
    \captionsetup{skip=0pt}
    \caption{A hierarchy with suggested and refined clusters:
             Leaf nodes represent strings, and internal nodes are labelled with patterns describing the strings below them.
             Atoms are concatenated using ``$\concat$''.
             A dashed edge denotes the absence of a pattern
             that describes the strings together.}
    \label{fig:example-hierarchy}
    \end{minipage}\hspace*{8pt}%
    \begin{minipage}{0.33\textwidth}
      \centering\smaller[1.5]
      \def\arraystretch{0.71}
      \setlength{\tabcolsep}{4pt}
      \begin{tabular}{c|c}
        \Lower                                & \BinDigit                                           \\
        \tokendesc{[a-z]}                     & \tokendesc{[01]}                                    \\[2.5pt]
        \Upper                                & \Digit                                              \\
        \tokendesc{[A-Z]}                     & \tokendesc{[0-9]}                                   \\[2.5pt]
        \TitleCaseWord                        & \HexDigit                                           \\
        \tokendesc{\Upper \concat \Lower^+}   & \tokendesc{[a-fA-F0-9]}                             \\[2.5pt]
        \Alpha                                & \AlphaDigit                                         \\
        \tokendesc{[a-zA-Z]}                  & \tokendesc{[a-zA-Z0-9]}                             \\[2.5pt]
        \Space                                & \AlphaDigitSpace                                    \\
        \tokendesc{\textbackslash{}s}         & \tokendesc{[a-zA-Z0-9\textbackslash{}s]}            \\[2.5pt]
        \DotDash                              & \Punct                                              \\
        \tokendesc{[.-]}                      & \tokendesc{[.,\,:\,?\,/-]}                          \\[2.5pt]
        \AlphaDash                            & \Symbol                                             \\
        \tokendesc{[a-zA-Z-]}                 & \tokendesc{[-.,\,:\,?\,/@\,\#\,\$\,\%\,\&\,\cdots]} \\[2.5pt]
        \AlphaSpace                           & \BaseSixtyFour                                      \\
        \tokendesc{[a-zA-Z\textbackslash{}s]} & \tokendesc{[a-zA-Z0-9+\setminus=]}
      \end{tabular}
      \captionsetup{skip=2pt}
      \caption{Default atoms in \FlashProfile, with their regex:
               We also allow ``\Any'' atom that matches any character.}
      \label{tab:lp-default-atoms}
    \end{minipage}
    \vspace*{-5pt}
\end{figure}

\footnotetextNum{Linda K. Jacobs, The Syrian Colony in New York City 1880-1900. \url{http://bit.ly/LJacobs}}

\paragraph{Adaptive Sampling and Approximation}
While \SyDissSymb\ captures a high-quality syntactic dissimilarity,
with it, each pairwise dissimilarity computation requires learning and scoring of patterns,
which may be expensive for large real datasets.
To allow end users to quickly generate approximately correct profiles for large datasets,
we present a two-stage sampling technique.
\begin{inlist}
    \item At the top-level, \FlashProfile employs a \textsf{Sample$-$\textsc{\Profiling}$-$Filter} cycle:
          we sample a small subset of the data, profile it, and filter out data that is described by the profile learned so far.
    \item While profiling each sample, our \BuildHierarchyAlgo algorithm approximates some pairwise dissimilarities
          using previously seen patterns.
\end{inlist}
We allow end users to control the degree of approximations using two optional parameters.

Our key insight that, unlike typical clustering scenarios,
computing dissimilarity between a pair of strings gives us more than just a measure --- we also learn a pattern.
We test this pattern on other pairs to approximate their dissimilarity,
which is typically faster than learning new patterns.
Our technique is inspired by counter-example guided inductive synthesis (CEGIS)~\citep{solar-lezama2006combinatorial}.
CEGIS was extended to sub-sampling settings by \citet{raychev2016learning},
but they synthesize a single program and require the outputs for all inputs.
In contrast, we learn a disjunction of several programs,
the outputs for which over a dataset, i.e. the partitions, are unknown a priori.

\begin{exmp}\label{exmp:sampling-dissimilarity}
    The pattern \patternmode{\stringliteral{PMC} \concat \rep{\Digit}{7}} learned for the string pair \stringPair{PMC2536233}{PMC4901429},
    also describes the string pair \stringPair{PMC4901429}{PMC2395569},
    and may be used to accurately estimate their dissimilarity without invoking learning again.
\end{exmp}

However this sampling needs to be performed carefully for accurate approximations.
Although the pattern \patternmode{\stringliteral{1} \concat \rep{\Digit}{3}} learned for \stringPair{1901}{1875},
also describes \stringPair{1872}{1875}, there exists another pattern \patternmode{\stringliteral{187} \concat \Digit},
which indicates a much lower syntactic dissimilarity.
We propose an adaptive algorithm for sampling patterns
based on previously observed patterns and strings in the dataset.
Our sampling and approximation algorithms are detailed in \cref{subsec:adaptive-sampling}
and \cref{subsec:approx-clustering} respectively.

\subsection{Pattern Learning via Program Synthesis}

\noindent
An important aspect of our clustering-based approach to profiling,
described in \cref{subsec:overview-clustering}, is its generality.
It is agnostic to the specific language \dsl in which patterns are expressed,
as long as appropriate pattern learner \learner and cost function \cost are provided for \dsl.
%It may leverage any pattern learning method over an arbitrary pattern language \dsl, if it exposes:
%\begin{itemize}[leftmargin=5mm,topsep=0pt]
%    \item a pattern learner $\learner\colon 2\,^\Strings \to 2\,^\dsl$,
%          which generates a set of patterns consistent with a given set of strings\footnoteNum{
%            \label{foot:powersetString}We denote the universe of all strings as \Strings, and power set of a set $X$ as $2^X$.}.
%          \learner is used to learn patterns for pairs of strings, and for the final clusters.
%    \item a cost function $\cost\colon \dsl \,\times\, 2\,^\Strings \to \Reals$,
%          which quantifies the suitability of a pattern for a given dataset.
%          \cost is used to identify the most suitable pattern, and to compute a measure of dissimilarity.
%\end{itemize}

\begin{wrapfigure}{r}{0.42\textwidth}
    \vspace*{2pt}\algobox{1.175}{
    \begin{algfunction}
      {\WithLCParam{\BestPattern}}
      {\dataset\colon \Tstring[\,]}
      {The least-cost pattern and its cost, for \dataset}
        \State $V \gets \Call{\learner}{\dataset}$
        \If{$V = \{\}$} \textbf{return} $\langle \prop{Pattern:\,}\patfail, \prop{Cost:\,}\costmax \rangle$
        \EndIf
        \State $\pattern \gets \argmin_{\pattern\,\in\,V}\ \Call{\cost}{\pattern, \dataset}$
        \Return{$\langle \prop{Pattern:\,}\pattern, \prop{Cost:\,}\Call{\cost}{\pattern, \dataset} \rangle$}
    \end{algfunction}}
    \caption{Learning the best pattern for a dataset}
    \label{fig:pattern-learning-algo}
\end{wrapfigure}

\BestPatternAlgo, listed in \cref{fig:pattern-learning-algo},
first invokes \learner to learn a set $V$ of patterns each of which describes all strings in \dataset.
If pattern learning fails,\nsp\footnoteNum{
    Pattern learning may fail, for example, if the language \dsl is too restrictive
    and no pattern can describe the given strings.
} in line 2, we return the special pattern \patfail and a very high cost \costmax.
Otherwise, we return the pattern that has the minimum cost using \cost w.r.t. \dataset.
\BestPatternAlgo is used during clustering to compute pairwise dissimilarity
and finally compute the least-cost patterns for clusters.

A natural approach to learning patterns is \emph{inductive program synthesis}~\citep{gulwani2017program},
which generalizes a given specification to desired programs over a domain-specific language (DSL).
We propose a rich DSL for patterns, and present an efficient inductive synthesizer for it.

\paragraph{Language for Patterns}
Our DSL \LP is designed to support efficient synthesis using existing technologies
while still being able to express rich patterns for practical applications.
A pattern is an arbitrary sequence of atomic patterns (atoms),
each containing low-level logic for matching a sequence of characters.
A pattern $\pattern \in \LP$ \emph{describes} a string $s$, i.e. $\patternapply{s} = \True$,
iff the atoms in \pattern match contiguous non-empty substrings of $s$, ultimately matching $s$ in its entirety.
\FlashProfile uses a default set of atoms listed in \cref{tab:lp-default-atoms},
which may be augmented with new regular expressions, constant strings, or ad hoc functions.
We formally define our language \LP in \cref{subsec:lp-syntax-semantics}.

\paragraph{Pattern Synthesis}
The inductive synthesis problem for pattern learning is: given a set of strings \dataset,
learn a pattern $\pattern \in \LP$ such that $\forall\,s \in \dataset$: $\patternapply{s} = \True$.
Our learner \learnerP decomposes the synthesis problem for \pattern over the strings in \dataset
into synthesis problems for individual atoms in \pattern over appropriate substrings.
However, a na\"{\i}ve approach of tokenizing each string to (exponentially many) sequences of atoms,
and computing their intersection is simply impractical.
Instead, \learnerP computes the intersection incrementally at each atomic match,
using a novel decomposition technique.

\learnerP is implemented using \PROSE~\citep{web.prose, polozov2015flashmeta},
a state-of-the-art inductive synthesis framework.
%, which is being deployed industrially~\citep{polozov2016program}.
%It performs a top-down walk over the specified DSL grammar,
%at each step reducing a given synthesis problem, based on the reduction logic specified by the DSL designer.
\PROSE requires DSL designers to define the logic for decomposing a synthesis problem
over an expression to those over its subexpressions,
which it uses to automatically generate an efficient synthesizer for their DSL.
We detail our synthesis procedure in \cref{subsec:lp-witnesses}.

\paragraph{Cost of Patterns}
Once a set of patterns has been synthesized, a variety of strategies may be used to identify the most desirable one.
Our cost function \costP is inspired by \emph{regularization}~\citep{tikhonov1963solution} techniques
that are heavily used in statistical learning to construct generalizations
that do not overfit to the data.
\costP decides a trade-off between two opposing factors:
\begin{inlist}
    \item \textit{specificity:} prefer a pattern that is not general, and
    \item \textit{simplicity:} prefer a compact pattern that is easy to interpret
\end{inlist}.

\begin{exmp}\label{exmp:cost-overview}
    The strings \stringPair{Male}{Female} are matched by the patterns \patternmode{\Upper \concat \Lower^+},
    and \patternmode{\Upper\concat\HexDigit\concat\Lower^+}.
    Although the latter is more specific, it is overly complex.
    On the other hand, the pattern \patternmode{\Alpha^+} is simpler and easier to interpret, but is overly general.
\end{exmp}

To this end, each atom in \LP has a fixed \emph{static cost}
similar to fixed \emph{regularization hyperparameters} used in machine learning~\citep{bishop2016pattern},
and a dataset-driven \emph{dynamic weight}.
The cost of a pattern is the weighted sum of the cost of its constituent atoms.
%For default atoms, static costs have been optimized for typical scenarios.
In \cref{subsec:lp-ranking}, we detail our cost function \costP,
and provide some guidelines on assigning costs for new atoms defined by users.

\section{Hierarchical Clustering} \label{sec:hierarchical-clustering}

\noindent
We now detail our clustering-based approach for generating syntactic profiles
and show practical optimizations for fast approximately-correct profiling.
In \cref{subsec:dissimilarity}\,--\;\cref{subsec:hierarchy-construction},
we explain these in the context of a small chunk of data drawn from a large dataset.
In \cref{subsec:profiling-large-datasets},
we then discuss how profile large datasets
by generating profiles for as many chunks as necessary and combining them.

% HACK: Not sure why the number was being incremented by 2.
\addtocounter{footnoteNum}{-1}

\begin{wrapfigure}{r}{0.5\textwidth}
    \vspace*{-1pt}\algobox{1.2}{
    \begin{algfunction}
      {\WithLCParam{\BuildHierarchy}}
      {\dataset\colon \Tstring[\,], \maxClusters\colon \Tint, \edgeSampleFactor\colon \Tdouble}
      {A hierarchical clustering over \dataset}
        \State $D \gets \Call{\WithLCParam{\SampleDissimilarities}}{\dataset,  \ceil*{\edgeSampleFactor\,\maxClusters}}$
        \State $\ApproxDissMatrix \gets \Call{\ApproxDissimilarityMatrix}{\dataset, D}$
        \Return{$\AHC(\dataset, \ApproxDissMatrix)$}
    \end{algfunction}}
    \captionsetup{skip=1.5pt}
    \caption{Building an approximately-correct hierarchy\protect\footnotemarkNum}
    \label{fig:hierarchy-algo}
\end{wrapfigure}

\footnotetextNum{$\ceil*{x}$ denotes the \emph{ceiling} of $x$, i.e. $\ceil*{x} = \min\;\{ m \in \Integers \mid m \geqslant x \}$.}

Recall that our first step in \ProfilingAlgo is to build a hierarchical clustering over the data.
The \BuildHierarchyAlgo procedure, listed in \cref{fig:hierarchy-algo},
constructs a hierarchy \hierarchy over a given dataset \dataset,
with parameters \maxClusters and \edgeSampleFactor.
\maxClusters is the maximum number of clusters in a desired profile.
The \emph{pattern sampling factor} \edgeSampleFactor
decides the performance \emph{vs.} accuracy trade-off while constructing \hierarchy.

Henceforth, we use \emph{pair} to denote a pair of strings.
In line 1 of \BuildHierarchyAlgo,
we first sample pairwise dissimilarities, i.e. the best patterns and their costs,
for a small set (based on the \edgeSampleFactor factor) of string pairs.
Specifically, out of all $O(|\dataset|^2)$ pairs within \dataset,
we adaptively sample dissimilarities for only $O(\edgeSampleFactor\maxClusters|\dataset|)$
pairs by calling \SampleDissimilaritiesAlgo, and cache the learned patterns in $D$.
We formally define the dissimilarity measure in \cref{subsec:dissimilarity},
and describe \SampleDissimilaritiesAlgo in \cref{subsec:adaptive-sampling}.
The cache $D$ is then used by \ApproxDissimilarityMatrixAlgo, in line 2, to complete
the dissimilarity matrix \ApproxDissMatrix over \dataset,
using approximations wherever necessary.
We describe these approximations in \cref{subsec:approx-clustering}.
Finally, a standard agglomerative hierarchical clustering (AHC)~\citep[Section IIB]{xu2005survey}
is used to construct a hierarchy over \dataset using the matrix \ApproxDissMatrix.

\subsection{Syntactic Dissimilarity} \label{subsec:dissimilarity}

\noindent
We formally define our syntactic dissimilarity measure as follows:

\begin{defn}{Syntactic Dissimilarity}\label{defn:dissimilarity}
For a given pattern learner \learner and a cost function \cost
over an arbitrary language of patterns \dsl,
we define the syntactic dissimilarity between strings $x, y \in \Strings$
as the minimum cost incurred by a pattern in \dsl to describe them together, i.e.
\begin{equation*}
    \SyDiss{x}{y} \bydef
      \begin{cases}
        0 & \text{if $x = y$} \\
        \costmax & \text{if $x \not= y \,\bigwedge\, V = \{\}$} \\
        \displaystyle\min_{\pattern \,\in\, V}\;\cost\big(\pattern, \{x, y\}\big) & \text{otherwise}
    \end{cases}
\end{equation*}
where $V = \learner\big(\{x,y\}\big) \subseteq \dsl$ is the set of patterns that describe strings $x$ and $y$,
and $\costmax$ denotes a high cost for a failure to describe $x$ and $y$ together using patterns learned by \learner.
\end{defn}

The following example shows some candidate patterns and their costs
encountered during dissimilarity computation for various pairs.
The actual numbers depend on the pattern learner and cost function used,
in this case \FlashProfile's \learnerP and \costP, which we describe in \cref{sec:pattern-synthesis}.
However, this example highlights the desirable properties for a natural measure of syntactic dissimilarity.

\begin{exmp}
For three pairs, we show the shortcomings of classical character-based similarity measures.
We compare the Levenshtein distance (LD)~\citep{levenshtein1966binary} for these pairs against
the pattern-based dissimilarity \SyDissSymb\ computed with our default atoms from \cref{tab:lp-default-atoms}.
On the right, we also show the least-cost pattern, and two other randomly sampled patterns that describe the pair.

First, we compare two dates both using the same syntactic format \patternmode{\stringliteral{YYYY-MM-DD}}:
\begin{table}[!h]
    \hspace*{10pt}\textit{\small(a)}\hspace*{2pt}
    \begin{subtable}[!h]{0.2\linewidth}\small
        \setlength{\tabcolsep}{2pt}
        \begin{tabular}{|c|}\hline
            \texttt{1990-11-23}\\\hline
            \texttt{2001-02-04}\\\hline
        \end{tabular}
    \end{subtable}
    \hfill
    $\textsf{LD} = 8$
    \hspace*{4pt}\emph{vs.}\hspace*{4pt}
    $\SyDissSymb = 4.96\:\Bigg\{$\hspace*{-2pt}
    \begin{subtable}[!h]{0.45\linewidth}\smaller
        \def\arraystretch{1.125}
        \setlength{\tabcolsep}{8pt}
        \begin{tabular}{r|l}
            \emph{4.96} & $\rep{\Digit}{4} \!\concat\! \stringliteral{-} \!\concat\! \rep{\Digit}{2} \!\concat\! \stringliteral{-} \!\concat\! \rep{\Digit}{2}$\\
            \emph{179.9} & $\HexDigit^+ \!\concat\! \Symbol \!\concat\! \HexDigit^+ \!\concat\! \stringliteral{-} \!\concat\! \HexDigit^+$\\
            \emph{46482} & $\Digit^+ \!\concat\! \Punct \!\concat\! \Any^+$
        \end{tabular}
    \end{subtable}
\end{table}

\noindent
Syntactically, these dates are very similar ---
they use the same delimiter \stringliteral{-}, and have same width for the numeric parts.
The best pattern found by \FlashProfile captures exactly these features.
However, Levenshtein distance for these dates is higher than the following dates
which uses a different delimiter and a different order for the numeric parts:
\begin{table}[!h]
    \hspace*{10pt}\textit{\small(b)}\hspace*{2pt}
    \begin{subtable}[!h]{0.175\linewidth}\small
        \setlength{\tabcolsep}{2pt}
        \begin{tabular}{|c|}\hline
            \texttt{1990-11-23}\\\hline
            \texttt{29/05/1923}\\\hline
        \end{tabular}
    \end{subtable}
    \hfill
    $\textsf{LD} = 5$
    \hspace*{4pt}\emph{vs.}\hspace*{4pt}
    $\SyDissSymb = 30.2\:\Bigg\{$\hspace*{-2pt}
    \begin{subtable}[!h]{0.45\linewidth}\smaller
        \def\arraystretch{1.125}
        \setlength{\tabcolsep}{8pt}
        \begin{tabular}{r|l}
            \emph{30.2} & $\Digit^+ \!\concat\! \Punct \!\concat\! \rep{\Digit}{2} \!\concat\! \Punct \!\concat\! \Digit^+$\\
            \emph{318.6} & $\Digit^+ \!\concat\! \Symbol \!\concat\! \Digit^+ \!\concat\! \Symbol \!\concat\! \Digit^+$\\
            \emph{55774} & $\Digit^+ \!\concat\! \Punct \!\concat\! \Any^+$
        \end{tabular}
    \end{subtable}
\end{table}

\noindent
The Levenshtein distance is again lower for the following pair containing a date and an ISBN code:
\begin{table}[!h]
    \hspace*{10pt}\textit{\small(c)}\hspace*{2pt}
    \begin{subtable}[!h]{0.175\linewidth}\small
        \setlength{\tabcolsep}{2pt}
        \begin{tabular}{|c|}\hline
            \texttt{1990-11-23}\\\hline
            \texttt{899-2119-33-X}\\\hline
        \end{tabular}
    \end{subtable}
    \hfill
    $\textsf{LD} = 7$
    \hspace*{4pt}\emph{vs.}\hspace*{4pt}
    $\SyDissSymb = 23595\:\Bigg\{$\hspace*{-2pt}
    \begin{subtable}[!h]{0.45\linewidth}\smaller
        \def\arraystretch{1.125}
        \setlength{\tabcolsep}{8pt}
        \begin{tabular}{r|l}
            \emph{23595} & $\Digit^+ \!\concat\! \stringliteral{-} \!\concat\! \Digit^+ \!\concat\! \stringliteral{-} \!\concat\! \Any^+$\\
            \emph{55415} & $\Digit^+ \!\concat\! \Punct \!\concat\! \Any^+$\\
            \emph{92933} & $\Any^+$
        \end{tabular}
    \end{subtable}
\end{table}

\noindent
The same trend is also observed for Jaro-Winkler~\citep{winkler1999state},
and other measures based on edit distance~\citep{gomaa2013survey}.
Whereas these measures look for exact matches on characters,
pattern-based measures have the key advantage of being able to generalize substrings to atoms.
\end{exmp}

%For computing the syntactic dissimilarity \SyDiss{s_1}{s_2},
%we invoke $\BestPatternAlgo(\{s_1, s_2\})$ (from \cref{fig:pattern-learning-algo}),
%which returns the best (least-cost) pattern and its cost, for a given dataset.

%\input{include/figures/adaptive-sampling-example}

\subsection{Adaptive Sampling of Patterns} \label{subsec:adaptive-sampling}

\noindent
Although \SyDissSymb\ accurately captures the syntactic dissimilarity of strings
over an arbitrary language of patterns,
it requires pattern learning and scoring for every pairwise dissimilarity computation,
which is computationally expensive.
While this may not be a concern for non-realtime scenarios,
such as profiling large datasets on cloud-based datastores,
we provide a tunable parameter to end users to be able to generate
approximately correct profiles for large datasets in real time.

%Typical clustering algorithms~\citep{jain1999data} require all pairwise dissimilarities.
%In such case, even with fast implementations of pattern learner and cost function,
%such as \FlashProfile's, which requires only $\sim10\,$ms per pair,
%one would spend $\sim3\,$s on clustering only $25$ strings.

%\footnoteNum{
%    Our pattern learner has a median learning time of 7\,ms per pairwise dissimilarity.
%    Recent \PROSE~\citep{web.prose}-based synthesizers require $\sim500$ms~\citep{polozov2016program} per learning task.}

%As shown in \cref{exmp:sampling-dissimilarity}, previously learned patterns
%may be used to estimate dissimilarities for other pairs.
%However, these patterns must be sampled carefully.
%An overly general pattern learned for a pair with very high dissimilarity would be unsuitable
%for approximating the \emph{least-cost} patterns for other pairs.

Besides a numeric measure of dissimilarity, computing \SyDissSymb\ over a pair also generates a pattern that describes the pair.
Since the patterns generalize substrings to atoms, often the patterns learned for one pair also describe many other pairs.
We aim to sample a subset of patterns that are likely to be sufficient for constructing a hierarchy
accurate until \maxClusters levels, i.e. $1 \leqslant k \leqslant \maxClusters$ clusters extracted from this hierarchy
should be identical to $k$ clusters extracted from a hierarchy constructed without approximations.
Our \SampleDissimilaritiesAlgo algorithm, shown in \cref{algo:sampling},
is inspired by the seeding technique of \tool{k-means++}~\citep{arthur2007k}.
Instead of computing all pairwise dissimilarities for pairs in $\dataset \times \dataset$,
we compute the dissimilarities for pairs in $\polar \times \dataset$,
where set \polar is a carefully selected small set of \emph{seed strings}.
The patterns learned during this process are likely to be sufficient
for accurately estimating the dissimilarities for the remaining pairs.

\begin{wrapfigure}{r}{0.52\textwidth}
    \vspace*{-1pt}\algobox{1.225}{
        \begin{algfunction}
          {\WithLCParam{\SampleDissimilarities}}
          {\dataset\colon \Tstring[\,], \maxClustersT\colon \Tint}
          {A dictionary mapping $O(\maxClustersT|\dataset|)$ pairs of strings from \dataset,\\ to the best pattern describing each pair and its cost}
              \State $D \gets \{\} \quad ; \quad a \gets \text{a random string in } \dataset \quad ; \quad \polar \gets \{ a \}$
              \For{$i \gets 1\ \textbf{to}\ \maxClustersT$}
                \ForAll{$b \in \dataset$}
                    \State $D[a, b] \gets \Call{\WithLCParam{\BestPattern}}{\{a, b\}}$
                \EndFor
                \LeftComment{\hspace*{11.25pt}Pick the most dissimilar string w.r.t. strings already in \polar.}
                \State $a \gets \argmax_{x \:\in\: \dataset}\; \min_{y \:\in\: \polar}\; D[y, x].\prop{Cost}$
                \State $\polar \gets \polar \cup \{ a \}$
              \EndFor
              \Return{$D$}
        \end{algfunction}}
    \captionsetup{skip=1.5pt}
    \caption{Adaptively sampling a small set of patterns}
    \label{algo:sampling}
\end{wrapfigure}

\SampleDissimilaritiesAlgo takes a dataset \dataset and a factor \maxClustersT,
and it samples dissimilarities for $O(\maxClustersT|\dataset|)$ pairs.
It iteratively selects a set \polar containing $\maxClustersT$ strings that are most dissimilar to each other.
Starting with a random string in \polar, in each iteration, at line 6,
it adds the string $x \in \dataset$ such that it is as dissimilar as possible,
even with its most-similar neighbor in \polar.
In the end, the set $D$ only contains dissimilarities for pairs in $\dataset \times \polar$,
computed at line 5.
Recall that, \maxClustersT is controlled by the pattern sampling factor \edgeSampleFactor.
In line 1 of \BuildHierarchyAlgo (in \cref{fig:hierarchy-algo}),
we set $\maxClustersT = \ceil*{\edgeSampleFactor\maxClusters}$.

Since the user may request up to at most \maxClusters clusters,
\edgeSampleFactor must be at least $1.0$,
so that we pick at least one seed string from each cluster to \polar.
Then, computing the dissimilarities with \emph{all} other strings in the dataset
would ensure we have a good distribution of patterns that describe intra-- and inter--\,cluster
dissimilarities, even for the finest granularity clustering with \maxClusters clusters.

\begin{exmp}
    Consider the dataset containing years in \cref{tab:years-example}.
    Starting with a random string, say {\small\stringliteral{1901}};
    the set \polar of seed strings grows as shown below, with increasing \maxClustersT.
    At each step, $NN$ (nearest neighbor) shows the new string added to \polar paired with its most similar neighbor.
    \begin{enumerate}[leftmargin=2.5cm,label=\textit{\small(\alph*)},itemsep=-1pt,topsep=-2pt,rightmargin=1.5cm]\small
        \item[\small$\maxClustersT = 2 \:\;\lvert$]
             $NN = \langle \epsilon, \stringliteral{1901} \rangle$
             \hfill
             $\polar=\{\stringliteral{1901}, \epsilon\}$
        \item[\small$\maxClustersT = 3 \:\;\lvert$]
             $NN = \langle \stringliteral{?}, \stringliteral{1901} \rangle$
             \hfill
             $\polar=\{\stringliteral{?}, \stringliteral{1901}, \epsilon\}$
        \item[\small$\maxClustersT = 4 \:\;\lvert$]
             $NN = \langle \stringliteral{1875?}, \stringliteral{1901} \rangle$
             \hfill
             $\polar=\{\stringliteral{1875?}, \stringliteral{?}, \stringliteral{1901}, \epsilon\}$
        \item[\small$\maxClustersT = 5 \:\;\lvert$]
             $NN = \langle \stringliteral{1817}, \stringliteral{1875?} \rangle$
             \hfill
             $\polar=\{\stringliteral{1817}, \stringliteral{1875?}, \stringliteral{?}, \stringliteral{1901}, \epsilon\}$
         \item[\small$\maxClustersT = 6 \:\;\lvert$]
              $NN = \langle \stringliteral{1898}, \stringliteral{1817} \rangle$
              \hfill
              $\polar=\{\stringliteral{1898}, \stringliteral{1817}, \stringliteral{1875?}, \stringliteral{?}, \stringliteral{1901}, \epsilon\}$
    \end{enumerate}
\end{exmp}

\subsection{Dissimilarity Approximation} \label{subsec:approx-clustering}

Now we present our technique for completing a dissimilarity matrix over a dataset \dataset,
using the patterns sampled from the previous step.
Note that, for a large enough value of the pattern sampling factor, i.e.
\resizebox{!}{0.925\baselineskip}{$\displaystyle\edgeSampleFactor \geqslant \frac{|\dataset|}{\maxClusters}$},
we would sampled all pairwise dissimilarities and no approximation would be necessary.
For smaller values of \edgeSampleFactor,
we use the patterns learned while computing \SyDissSymb\ over $\polar \times \dataset$
to approximate the remaining pairwise dissimilarities in $\dataset \times \dataset$.
The key observation here is that,
testing whether a pattern describes a string is typically much faster than learning a new pattern.

\begin{wrapfigure}{r}{0.54\textwidth}
    \vspace*{1pt}\algobox{1.225}{
        \begin{algfunction}
          {\WithLCParam{\ApproxDissimilarityMatrixAlgo}}
          {\dataset\colon \Tstring[\,], \\
           \hspace*{115pt}D\colon \Tstring\times\Tstring \mapsto \Tpattern\times\Tdouble}
          {A matrix \ApproxDissMatrix of all pairwise dissimilarities over strings in \dataset}
              \State $\ApproxDissMatrix \gets \{\}$
              \ForAll{$x \in \dataset$}
                \ForAll{$y \in \dataset$}
                  \If{$x = y$}
                    $\ApproxDissMatrix[x, y] \gets 0$
                  \ElsIf{$\langle x, y \rangle \in D$}
                    $\ApproxDissMatrix[x, y] \gets D[x, y].\prop{Cost}$
                  \Else
                  \LeftComment{\hspace*{34pt}Select the least cost pattern that describes $x$ and $y$.}
                    \State $V \gets \big\{ P \bigm\lvert \langle \prop{Pattern:}\,P, \prop{Cost:\,}\,\cdot \rangle \in D
                                                         \bigwedge \patternapply{x} \bigwedge \patternapply{y} \big\}$
                  %\LeftComment{\hspace*{34pt}Try to approximate $\SyDiss{x}{y}$ with the best pattern in $D$.}
                    \If{$V \neq \{\}$}
                        $\ApproxDissMatrix[x, y] \gets \min_{P \:\in\:V}\; \cost\big(P, \{x,y\}\big)$
                    \Else
                    \LeftComment{\hspace*{47pt}Compute $\SyDiss{s}{y}$, and store the learned pattern.}
                      \State $D[x, y] \gets \Call{\WithLCParam{\BestPattern}}{\{x,y\}}$
                      \State $\ApproxDissMatrix[x, y] \gets D[x,y].\prop{Cost}$
                    \EndIf
                  \EndIf
                \EndFor
              \EndFor
              \Return{\ApproxDissMatrix}
        \end{algfunction}}
    \caption{Approximating a complete dissimilarity matrix}
    \label{algo:approx}
\end{wrapfigure}

The \ApproxDissimilarityMatrixAlgo procedure, listed in \cref{algo:approx},
uses the dictionary $D$ of patterns from \SampleDissimilaritiesAlgo
to generate a matrix \ApproxDissMatrix of all pairwise dissimilarities over \dataset.
Lines 7 and 8 show the key approximation steps for a pair $\{x,y\}$.
In line 7, we test the patterns in $D$,
and select a set $V$ of them containing only those which describe both $x$ and $y$.
We then compute their new costs relative to $\{x,y\}$, in line 8,
and select the least cost as an approximation of \SyDiss{x}{y}.
If $V$ turns out to be empty, i.e. no sampled pattern describes both $x$ and $y$,
then, in line 10, we call \BestPatternAlgo to compute \SyDiss{x}{y}.
We also add the new pattern to $D$ for use in future approximations.

Although $\edgeSampleFactor = 1.0$ ensures that we pick a seed string from each final cluster,
in practice we use a \edgeSampleFactor that is slightly greater than $1.0$.
This allows us to sample a few more seed strings, and ensures a better distribution of patterns in $D$
at the cost of a negligible performance overhead.
In practice, it rarely happens that no sampled pattern describes a new pair (at line 9, \cref{algo:approx}),
since seed patterns for inter-cluster string pairs are usually overly general, as we show in the example below.

\begin{exmp}
    Consider a dataset $\dataset = \{${\small\stringliteral{07-jun}, \stringliteral{aug-18},
    \stringliteral{20-feb}, \stringliteral{16-jun}, \stringliteral{20-jun}}$\}$.
    Assuming $\maxClusters = 2$ and $\edgeSampleFactor = 1.0$ (i.e. $\maxClustersT = 2$),
    suppose we start with the string {\small\stringliteral{20-jun}}.
    Then, following the \SampleDissimilaritiesAlgo algorithm shown in \cref{algo:sampling},
    we would select $\polar = \stringPair{20-jun}{aug-18}$,
    and would sample the following seed patterns into $D$
    based on patterns defined over our default atoms (listed in \cref{tab:lp-default-atoms}) and constant string literals:
    \begin{enumerate}[leftmargin=5mm,label=\textit{\small(\alph*)},itemsep=1.5pt,topsep=2pt]
        \item $D[${\small\stringliteral{20-jun},\stringliteral{07-jun}}$] \mapsto$ \patternmode{\rep{\Digit}{2} \concat \stringliteral{-jun}}, and
        \item $D[${\small\stringliteral{20-jun},\stringliteral{20-feb}}$] \mapsto$ \patternmode{\stringliteral{20-} \concat \rep{\Lower}{3}},
        \item $D[${\small\stringliteral{20-jun},\stringliteral{16-jun}}$] \mapsto$ \patternmode{\rep{\Digit}{2} \concat \stringliteral{-jun}}, and
        \item $D[${\small\stringliteral{20-jun},\stringliteral{aug-18}}$]$,
              $D[${\small\stringliteral{aug-18},\stringliteral{07-jun}}$]$,
              $D[${\small\stringliteral{aug-18},\stringliteral{20-feb}}$]$,
              $D[${\small\stringliteral{aug-18},\stringliteral{16-jun}}$]$
              $\mapsto$\\
              \patternmode{\AlphaDigit^+ \concat \stringliteral{-} \concat \AlphaDigit^+}.
    \end{enumerate}

    Next, we estimate \SyDiss{\small\stringliteral{16-jun}}{\small\stringliteral{20-feb}} using these patterns.
    None of \emph{(a)} --- \emph{(c)} describe the pair, but \emph{(d)} does.
    However, it is overly general compared to the least-cost pattern, \patternmode{\rep{\Digit}{2} \concat \stringliteral{-} \concat \rep{\Lower}{3}}.
\end{exmp}

As in the case above, depending on the expressiveness of the pattern language, for a small \edgeSampleFactor
the sampled patterns may be too specific to be useful.
With a slightly higher $\edgeSampleFactor = 1.25$, i.e. $\maxClustersT = \ceil*{\edgeSampleFactor\maxClusters} = 3$,
we would also select {\small\stringliteral{07-jun}} as a seed string in \polar,
and sample the desired while computing $D[${\small\stringliteral{07-jun},\stringliteral{20-feb}}$]$.
We evaluate the impact of \edgeSampleFactor on performance and accuracy in \cref{subsec:eval-profiling-accuracy}.

%Another direction worth exploring, is to retain multiple patterns from the set $V = \learner\big(\{x,y\}\big)$ of learned patterns,
%while computing $\SyDiss{x}{y} = \min_{\pattern \,\in\, V}\;\cost\big(\pattern, \{x, y\}\big)$.
%For instance, we may retain the $n$ best patterns.
%However, cost functions are typically optimized to find only the best pattern from $V$,
%which may represent a very large number of patterns.
%Deciding a set of patterns to retain, and efficiently extracting them from $V$ is worth investigating.

\subsection{Hierarchy Construction and Splitting} \label{subsec:hierarchy-construction}

\begin{wrapfigure}{r}{0.475\textwidth}
    \vspace*{-18pt}\algobox{1.25}{
    \begin{algfunction}
      {\AHCAlgo}
      {\dataset\colon \Tstring[\,], \ApproxDissMatrix\colon \Tstring\times\Tstring \mapsto \Tdouble}
      {A hierarchy over \dataset using dissimilarity matrix \ApproxDissMatrix}
        \State $\hierarchy \gets \big\{ \{\, s \,\} \mid s \in \dataset \big\}$
        \While{$|\, \hierarchy \,| > 1$}
            \State $\langle X, Y \rangle \gets \argmin_{X,Y \in \hierarchy} \, \cluSyDiss{X}{Y}{\ApproxDissMatrix}$
            \State $\hierarchy \gets \big(\hierarchy \setminus \{ X,Y \}\big) \cup \big\{ \{\, X,Y \,\} \big\}$
        \EndWhile
        \Return{\hierarchy}
    \end{algfunction}}
    \captionsetup{skip=1pt}
    \caption{A standard algorithm for AHC}
    \label{algo:ahc}
\end{wrapfigure}

Once we have a dissimilarity matrix,
we use a standard agglomerative hierarchical clustering (AHC)~\citep[Section IIB]{xu2005survey} algorithm,
as outlined in \cref{algo:ahc}.
Note that AHC is not parameterized by \learner and \cost,
since it does not involve learning or scoring of patterns any more.

We start with each string in a singleton set (leaf nodes of the hierarchy).
Then, we iteratively join the least-dissimilar pair of sets,
until we are left with a single set (root of the hierarchy).
AHC relies on a \emph{linkage criterion} to estimate dissimilarity of sets of strings.
We use the classic complete-linkage (also known as further-neighbor linkage) criterion~\citep{sorensen1948method},
which has been shown to be resistant to outliers,
and yield useful hierarchies in practical applications \citep{jain1999data}.

\begin{defn}{Complete-Linkage}\label{defn:linkage}
For a set \dataset and a dissimilarity matrix \ApproxDissMatrix defined on \dataset,
given two arbitrarily-nested clusters $X$ and $Y$ over a subset of entities in \dataset,
we define the dissimilarity between their contents
(the \emph{flattened} sets $\overline{X},\overline{Y} \subseteq \dataset$, respectively) as:
\begin{equation*}
    \cluSyDiss{X}{Y}{\ApproxDissMatrix} \bydef \max_{x \in \overline{X}, y \in \overline{Y}} \, \ApproxDissMatrix[x,y]
\end{equation*}\vspace*{-9pt}
\end{defn}

Once a hierarchy has been constructed,
our \ProfilingAlgo algorithm (in \cref{fig:profiling-algo}) invokes the \ClusteringAlgo method (at line 2)
to extract $k$ clusters within the provided bounds $[\minClusters, \maxClusters]$.
If $\minClusters \not= \maxClusters$, we use a heuristic based on the \emph{elbow}
(also called \emph{knee}) method~\citep{halkidi2001clustering}:
between the top $\minClusters^\text{th}$ and the $\maxClusters^\text{th}$ nodes,
we split the hierarchy till the knee ---
a node below which the average intra-cluster dissimilarity does not vary significantly.
A user may request $\minClusters = k = \maxClusters$,
in which case \ClusteringAlgo simply splits the top $k$ nodes of the hierarchy to generate $k$ clusters.

%\paragraph{Discussion} Instead of sampling for only the best pattern (\cref{algo:sampling}$\,\cdot\,$line 5), depending on the available resources, one may store a subset of learned patterns $\learner(\{a,b\})$ for the pair $\{a,b\}$.
%The hierarchy may be split at a desired height to produce the subtrees as clusters. We use a simple heuristic which proposes a split for the tree based on the dissimilarities of a node w.r.t its children. \FlashProfile APIs allow the users to override this heuristic and request a desired number of clusters from the hierarchy.

\subsection{Profiling Large Datasets}\label{subsec:profiling-large-datasets}

\noindent
To scale our technique to large datasets, we now describe a second round of sampling.
Recall that in \SampleDissimilaritiesAlgo, we sample dissimilarities for $O(\edgeSampleFactor\maxClusters|\dataset|)$ pairs.
However, although $\edgeSampleFactor\maxClusters$ is very small,
$|\,\dataset\,|$ is still very large for real-life datasets.
In order to address this, we run our \ProfilingAlgo algorithm from \cref{fig:profiling-algo}
on small chunks of the dataset, and combine the generated profiles.

\begin{wrapfigure}{r}{0.405\textwidth}
  \vspace*{2pt}\algobox{1.25}{
    \begin{algfunction}
      {\WithLCParam{\LargeProfile}}
      {\dataset\colon \Tstring[\,], \minClusters\colon \Tint, \maxClusters\colon \Tint,\\
       \hspace*{90pt}\edgeSampleFactor\colon \Tdouble, \pointSampleFactor\colon \Tdouble}
      {A profile \profile that satisfies $\minClusters \leqslant |\: \profile \:| \leqslant \maxClusters$}
        \State $\profile \gets \{\}$
        \While {$|\,\dataset\,| > 0$}
          \State $X \gets \Call{SampleRandom}{\dataset, \ceil*{\pointSampleFactor \maxClusters}}$
          \State $\profile' \gets \Call{\WithLCParam{\Profiling}}{X, \minClusters, \maxClusters, \edgeSampleFactor}$
          \State $\profile \gets \Call{\WithLCParam{\MergeMostSimilarPatterns}}{\profile \cup \profile', \maxClusters}$
          \State $\dataset \gets \Call{RemoveMatchingStrings}{\dataset, \profile}$
        \EndWhile
        \Return{\profile}
    \end{algfunction}}
    \captionsetup{skip=2pt}
    \caption{Profiling large datasets}
    \label{algo:large-profile}
\end{wrapfigure}

We outline our \LargeProfileAlgo algorithm in \cref{algo:large-profile}.
This algorithm accepts a new \emph{string sampling factor} $\pointSampleFactor \geqslant 1$,
which controls the size of chunks profiled in each iteration,
until we have profiled all the strings in \dataset.
In \cref{subsec:eval-performance}, we evaluate the impact of \pointSampleFactor on performance and accuracy.

We start by selecting a random subset $X$ of size $\ceil*{\pointSampleFactor \maxClusters}$ from \dataset in line 3.
In line 4, we obtain a profile $\profile'$ of $X$, and merge it with the global profile \profile in line 5.
We repeat this loop with the remaining strings in \dataset that do not match the global profile.
While merging \profile and $\profile'$ in line 5, we may exceed the maximum number of patterns \maxClusters,
and may need to \emph{compress} the profile.

\begin{wrapfigure}{r}{0.59\textwidth}
  \vspace*{2pt}\algobox{1.265}{
    \begin{algfunction}
      {\WithLCParam{\MergeMostSimilarPatterns}}
      {\profile\colon \texttt{ref}\ \Tprofile, \maxClusters\colon \Tint}
      {A compressed profile \profile that satisfies $|\: \profile \:| \leqslant \maxClusters$}
        \While {$|\,\profile\,| > \maxClusters$}
          \LeftComment{\hspace*{11.5pt}Compute the most similar partitions in the profile so far.}
          \State $\langle X,Y \rangle \gets \displaystyle\argmin_{X,Y \in \profile} \big[\Call{\WithLCParam{\BestPattern}}{X.\prop{Data} \cup Y.\prop{Data}}\big].\prop{Cost}$
          \LeftComment{\hspace*{11.5pt}Merge partitions $\langle X, Y \rangle$, and update \profile.}
          \State $Z \gets X.\prop{Data} \;\cup\; Y.\prop{Data}$
          \State $P \gets \Call{\WithLCParam{\BestPattern}}{Z}.\prop{Pattern}$
          \State $\profile \gets (\profile \setminus \{X, Y\}) \cup \{\,\langle \prop{Data:}\;Z, \prop{Pattern:}\;P \rangle\,\}$
        \EndWhile
        \Return{\profile}
    \end{algfunction}}
    \captionsetup{skip=2pt}
    \caption{Limiting the number of patterns in a profile\vspace*{3pt}}
    \label{algo:compress-profile}
\end{wrapfigure}

For brevity, we elide the details of \algo{SampleRandom} and \algo{RemoveMatchingStrings},
which have straightforward implementations.
In \cref{algo:compress-profile} we outline \MergeMostSimilarPatternsAlgo.
It accepts a profile \profile and shrinks it to at most \maxClusters patterns.
The key idea is to repeatedly merge the most similar pair of patterns in \profile.
However, we cannot compute the similarity between patterns.
Instead, we estimate it using syntactic similarity of the associated data partitions.
The profile \profile must be of the same type as returned by \ProfilingAlgo,
i.e. a set of pairs, each containing a data partition and its pattern.
In line 2, we first identify the partitions $\langle X, Y\rangle$ which are the most similar,
i.e. require the least cost pattern for describing them together.
We then merge $X$ and $Y$ to $Z$, learn a pattern describing $Z$,
and update \profile by replacing $X$ and $Y$ with $Z$ and its pattern.
This process is repeated until the total number of patterns falls to \maxClusters.

\begin{theorem}[Termination]
    Over an arbitrary language \dsl of patterns,
    assume an arbitrary learner $\learner\colon 2\,^\Strings \to 2\,^\dsl$
    and a cost function $\cost\colon \dsl \,\times\, 2\,^\Strings \to \NonNegReals$,
    such that for any finite dataset $\dataset \subset \Strings$, we have:
    \begin{inlist}
        \item $\learner(\dataset)$ terminates and produces a finite set of patterns, and
        \item $\cost(P, \dataset)$ terminates for all $\pattern \in \dsl$.
    \end{inlist}
    Then, the \textnormal{\LargeProfileAlgo} procedure (\cref{algo:large-profile}) terminates
    on any finite dataset $\dataset \subset \Strings$,
    for arbitrary valid values of the optional parameters $m$, $M$, \edgeSampleFactor and \pointSampleFactor.
\end{theorem}

\begin{proof}
    We note that in \LargeProfileAlgo, the loop within lines 2 -- 6 runs
    for at most \resizebox{!}{0.875\baselineskip}{$\displaystyle\frac{|\dataset|}{\ceil*{\pointSampleFactor\maxClusters}}$} iterations,
    since at least $\ceil*{\pointSampleFactor\maxClusters}$ strings are removed from \dataset in each iteration.
    Therefore, to prove termination of \LargeProfileAlgo, it is sufficient to show that
    \ProfilingAlgo and \MergeMostSimilarPatternsAlgo terminate.
    %in the worst case when no additional strings,
    %other than the $\ceil*{\pointSampleFactor\maxClusters}$ strings sampled in to $X$,
    %are removed from \dataset in line 6.

    First, we note that termination of \BestPatternAlgo immediately follows from \inliststyle{(1)} and \inliststyle{(2)}.
    Then, it is easy to observe that \MergeMostSimilarPatternsAlgo terminates as well:
    \begin{inlist}
        \item the loop in lines 1 -- 5 runs for at most $|\profile|-M$ iterations, and
        \item \BestPatternAlgo is invoked $O(|\profile|^2)$ times in each iteration.
    \end{inlist}

    The \ProfilingAlgo procedure (\cref{fig:profiling-algo})
    makes at most $O\big((\pointSampleFactor\maxClusters)^2\big)$ calls to \BestPatternAlgo (\cref{fig:pattern-learning-algo})
    to profile the $\ceil*{\pointSampleFactor\maxClusters}$ strings sampled in to $X$ ---
    at most $O\big((\pointSampleFactor\maxClusters)^2\big)$ calls within \BuildHierarchyAlgo (\cref{fig:hierarchy-algo}),
    and $O(M)$ calls to learn patterns for the final partitions.
    Depending on \edgeSampleFactor, \BuildHierarchyAlgo may make many fewer calls to \BestPatternAlgo.
    However, it makes no more than 1 such call per pair of strings in $X$, to build the dissimilarity matrix.
    Therefore, \ProfilingAlgo terminates as well.
    %Finally, \MergeMostSimilarPatternsAlgo (\cref{algo:compress-profile}), invoked in line 5,
    %makes at most $O(M^2)$ calls to \BestPatternAlgo for merging clusters with similar patterns.
    %Therefore, \LargeProfileAlgo invokes \BestPatternAlgo at most
    %\resizebox{!}{0.9\baselineskip}{$\Big(O\big((\pointSampleFactor\maxClusters)^2\big) + O(M) + O(M^2)\Big)$}
    %$\cdot$
    %\resizebox{!}{0.925\baselineskip}{$\displaystyle\frac{|\dataset|}{\ceil*{\pointSampleFactor\maxClusters}}$}
    %$= O(\pointSampleFactor\maxClusters|\dataset|)$ times.
    %Since \learner is invoked from only within \BestPatternAlgo, the claim follows.
\end{proof}

% \paragraph{Termination Guarantee}
% We conclude this section with an informal proof of the termination guarantee of our profiling technique.
% Assuming that the specified learner \learner and cost function \cost,
% and therefore \BestPatternAlgo (\cref{fig:pattern-learning-algo}) terminate for any finite dataset,
% it is easy to observe that both \SampleDissimilaritiesAlgo (\cref{algo:sampling})
% and \ApproxDissimilarityMatrixAlgo (\cref{algo:approx}) terminate,
% since they may make at most quadratically as many calls to \learner and \cost as the size of the dataset.
% Therefore, the AHC algorithm (\cref{algo:ahc}), \BuildHierarchyAlgo (\cref{fig:hierarchy-algo}),
% and ultimately our main \ProfilingAlgo (\cref{fig:profiling-algo}) also terminate for finite datasets.
% Furthermore, since \MergeMostSimilarPatternsAlgo (\cref{algo:compress-profile}) may make at most quadratically many calls to
% \BestPatternAlgo as the number of discovered patterns,
% line 5 of \LargeProfileAlgo (\cref{algo:large-profile}) must terminate during each iteration.
% Finally, since we remove at least $\ceil*{\pointSampleFactor\maxClusters}$ strings from \dataset in each iteration,
% \LargeProfileAlgo  must also eventually terminate when $|\,\dataset\,| = 0$.
\section{Pattern Synthesis} \label{sec:pattern-synthesis}

\noindent
We now describe the specific pattern language, leaning technique and cost function
used to instantiate our profiling technique as \FlashProfile.
We begin with a brief description our pattern language in \cref{subsec:lp-syntax-semantics},
present our pattern synthesizer in \cref{subsec:lp-witnesses},
and conclude with our cost function in \cref{subsec:lp-ranking}.

\subsection{The Pattern Language \bookmarkLP}\label{subsec:lp-syntax-semantics}

\noindent
\Cref{subfig:lp-syntax} shows the formal syntax for our pattern language \LP.
Each pattern $P \in \LP$ is a predicate defined on strings,
i.e. a function $\pattern\colon \Tstring \rightarrow \Tbool$,
which embodies a set of constraints over strings.
A pattern \pattern \emph{describes} a given string $s$, i.e. $\patternapply{s} = \True$,
iff $s$ satisfies all constraints imposed by \pattern.
Patterns in \LP are composed of atomic patterns:

% \LP allows 3 kinds of patterns:
%\begin{itemize}[topsep=1pt, itemsep=1pt]
%	\item \NullSymb, which describes a only \Snull value,
%	\item \EmptySymb, which describes only the empty string $\epsilon$, and
%	\item \emph{\token-patterns}, which describe non-empty strings using a sequence of atomic patterns (\emph{atoms}).
%\end{itemize}

\vspace*{-1.5pt}\begin{defn}{Atomic Pattern (or Atom)}
    An atom, $\token\colon \Tstring \rightarrow \Tint$ is a function, which given a string $s$,
    returns the length of the longest prefix of $s$ that satisfies its constraints.
    Atoms only match non-empty prefixes. $\token(s) = 0$ indicates match failure of \token on $s$.
\end{defn}\vspace*{-1.5pt}

\noindent
We allow the following four kinds of atoms in \LP:
\begin{enumerate}[leftmargin=8mm,topsep=2pt,itemsep=1pt,label=\inliststyle{(\alph*)}]
    \item \emph{Constant Strings}:\ 
          A \patternmode{\tokConst_s} atom matches only the string $s$ as the prefix of a given string.
          For brevity, we denote \patternmode{\tokConst_\stringliteral{str}} as simply \patternmode{\stringliteral{str}} throughout the text.
    \item \emph{Regular Expressions}:\ 
          A \patternmode{\tokRegEx_r} atom returns the length of the longest prefix of a given string,
          that is matched by the regex $r$.
    \item \emph{Character Classes}:\ 
          A \patternmode{\tokClass^0_c} atom returns the length of the longest prefix of a give string,
          which contains only characters from the set $c$.
          A \patternmode{\tokClass^z_c} atom with $z > 0$ further enforces a fixed-width constraint ---
          the match \patternmode{\tokClass^z_c(s)} fails if \patternmode{\tokClass^0_c(s)}$ \neq z$, otherwise it returns $z$.
    \item \emph{Arbitrary Functions}:\ 
          A \patternmode{\tokFunct_f} atom uses the function $f$ that may contain arbitrary logic,
          to match a prefix $p$ of a given string and returns $|p|$.
\end{enumerate}

% HACK: Not sure why the number was being incremented by 2.
\addtocounter{footnoteNum}{-1}

\begin{figure}[t]
  \vspace*{2pt}\centering
  \begin{minipage}{0.35\textwidth}
      \resized{1.0}{1.025}{
            \centering
            \def\arraystretch{1.25}
            \setlength{\arraycolsep}{1.15pt}
            {\small$\begin{array}{rccl}
               \textsf{\small Pattern}  & \patexpr{s} & := & \Empty{s} \\
                                        &             &  | & \patexpr{\SuffixAfter{s}{\token}} \\[6pt]
               \textsf{\small Atom}     & \token & := & \tokClass^z_c\ |\ \tokRegEx_r \\
                                        &        &  | & \tokFunct_f\ |\ \tokConst_s
            \end{array}$}\\[5pt]\noindent{\color{darkbordercolor}\rule{\linewidth}{0.5pt}}\\[5pt]
            \def\arraystretch{1}
            {\small$\begin{array}{ccl}
               c  & \in & \textnormal{power set of characters} \\
               f  & \in & \textnormal{functions\;} \Tstring \rightarrow \Tint \\
               r  & \in & \textnormal{regular expressions} \\
               s  & \in & \textnormal{set of strings\;} \Strings \\
               z  & \in & \textnormal{non-negative integers}
            \end{array}$}
            \vspace*{2pt}
            \subcaption{Syntax of \LP patterns.}
            \label{subfig:lp-syntax}
        }
  \end{minipage}\hfill\begin{minipage}{0.63\textwidth}
    \resized{1.0}{1.025}{
      \centering
      {\small$\begin{array}{c}
        \infer{\BigStep{\Empty{\epsilon}}{\Strue}}
              {} \\[15pt]
        \infer{\BigStep{\SuffixAfter{s}{\token}}{s_1}}
              {s = s_0 \circ s_1
              \quad\enskip \token(s) = |s_0| > 0} \\[18pt]
        \infer{\BigStep{\tokFunct_f(s)}{f(s)}}
              {} \\[13pt]
        \infer{\BigStep{\tokConst_s(s_0)}{|s|}}
              {|s| > 0 \quad\enskip s_0 = s \circ s_1}
      \end{array}$}
      \bordersep
      {\small$\begin{array}{c}
        \infer{\BigStep{\tokRegEx_r(s)}{\max L}}
              {L = \{ n \in \Naturals \mid \matches{r}{\substr{s}{0}{n}} \}} \\[15pt]
        \infer{\BigStep{\tokClass^0_c(s)}{|s_0|}}
              {s = s_0 \circ s_1
               \quad\enskip \forall x \in s_0\colon x \in c \\[1pt]
               {\quad\quad\enskip s_1 = \epsilon \vee s_1[0] \not\in c}}\\[18pt]
        \infer{\BigStep{\tokClass^z_c(s)}{z}}
              {s = s_0 \circ s_1
               \quad\quad \forall x \in s_0\colon x \in c \\[1pt]
                {|s_0| = z > 0 \quad\enskip s_1 = \epsilon \vee s_1[0] \not\in c}}
      \end{array}$}
      \vspace*{4pt}
      \subcaption{Big-step semantics for \LP patterns:
                  We use the judgement $\BigStep{E}{v}$ to indicate that the expression $E$ evaluates to a value $v$.}
      \label{subfig:lp-semantics}
      }
  \end{minipage}
  \captionsetup{skip=4pt}
  \caption{Formal syntax and semantics of our DSL \LP for defining syntactic patterns over strings\protect\footnotemarkNum
           \vspace*{-8pt}}
  \label{fig:language-lp}
  \end{figure}

  \footnotetextNum{$a \circ b$ denotes the concatenation of strings $a$ and $b$, and \matches{r}{x} denotes that the regex $r$ matches the string $x$ in its entirety.}

Note that, although both constant strings and character classes may be expressed as regular expressions,
having separate terms for them has two key benefits:
\begin{itemize}[leftmargin=4mm,itemsep=0pt,topsep=0pt]
    \item As we show in the next subsection, we can \emph{automatically infer} all constant strings,
          and some character class atoms (namely, those having a fixed-width).
          This is unlike regular expression or function atoms,
          which we do not infer and they must be provided a priori.
    \item These atoms may leverage more efficient matching logic
          and do not require regular expression matching in its full generality.
          Constant string atoms use equality checks for characters,
          and character class atoms use set membership checks.
\end{itemize}
We list the default set of atoms provided with \FlashProfile, in \cref{tab:lp-default-atoms}.
Users may extend this set with new atoms from any of the aforementioned kinds.

\vspace*{-1.5pt}\begin{exmp}
The atom \patternmode{\Digit} is \patternmode{\tokClass^1_\mathsmaller{D}} with $D = \{0,\ldots,9\}$.
We write \patternmode{\tokClass^0_\mathsmaller{D}} as \patternmode{\Digit^+},
and \patternmode{\tokClass^n_\mathsmaller{D}} as \patternmode{\rep{\Digit}{n}} for clarity.
Note that, \patternmode{\rep{\Digit}{2}} matches \patternmode{\stringliteral{04/23}}
but not \patternmode{\stringliteral{2017/04}}, although \patternmode{\Digit^+} matches both,
since the longest prefix matched, \patternmode{\stringliteral{2017}}, has length $4 \neq 2$.
\end{exmp}

\begin{defn}{Pattern}\label{def:pattern}
    A pattern is simply a sequence of atoms.
    The pattern \EmptySymb\ denotes an empty sequence, which only matches the empty string $\epsilon$.
    We use the concatenation operator ``$\concat$'' for sequencing atoms.
    For $k > 1$, the sequence \patternmode{\token_1\concat\token_2\concat\ldots\concat\token_k} of atoms
    defines a pattern that is realized by the \LP expression:
        \[
            \patternmode{\Empty{\SuffixAfter{\cdots\SuffixAfter{s}{\token_1}\cdots}{\token_k}}},
        \]
    which matches a string $s$, iff
        \[
            s \neq \epsilon \enskip\textstyle{\bigwedge}\enskip \forall i \in \{1,...,k\}: \token_i(s_i) > 0 \enskip\textstyle{\bigwedge}\enskip s_{k+1} = \epsilon,
        \]
    where $s_1 \bydef s$ and $s_{i+1} \bydef \substr{s_i}{\token_i(s_i)}{}$
    is the remaining suffix of the string $s_i$ after matching atom $\token_i$.
\end{defn}

Throughout this section, we use $s[i]$ to denote the $i^\text{th}$ character of $s$,
and $\substr{s}{i}{j}$ denotes the substring of $s$ from the $i^\text{th}$ character, until the $j^\text{th}$.
We omit $j$ to indicate a substring extending until the end of $s$.
In \LP, the {\small\SuffixAfter{s}{\token}} operator computes $\substr{s}{\token(s)}{}$,
or fails with an error if $\token(s) = 0$.
We also show the formal semantics of patterns and atoms in \LP, in \cref{subfig:lp-semantics}.

Note that, we explicitly forbid atoms from matching empty substrings.
This reduces the search space by an exponential factor,
since an empty string may trivially be inserted between any two characters within a string.
However, this does not affect the expressiveness of our final profiling technique,
since a profile uses a disjunction of patterns.
For instance, the strings matching a pattern \patternmode{\token_1\concat(\epsilon\;|\;\token_2)\concat\token_3}
can be clustered into those matching \patternmode{\token_1\concat\token_3}
and \patternmode{\token_1\concat\token_2\concat\token_3}.

% \begin{exmp}
%     Consider the following URLs collected from a dataset\footnoteSym{
%         \fontsize{6.7pt}{8pt}\selectfont\url{https://support.travelpayouts.com/hc/ru/article\_attachments/201368927/places\_t.csv}}
%     containing flight data for various destinations:\\[5pt]
%     \begin{table}[h]\vspace*{-11pt}\centering
%         \small
%         \def\arraystretch{1.1}
%         \begin{tabular}{|l|}\hline
%             \texttt{http://www.jetradar.com/flights/EsaAla-ESA/}\\
%             \texttt{http://www.jetradar.com/flights/Mumbai-BOM/}\\
%             \texttt{http://www.jetradar.com/flights/NDjamena-NDJ/}\\
%             \texttt{http://www.jetradar.com/flights/Bangalore-BLR/}\\
%             \texttt{http://www.jetradar.com/flights/LaForges-YLF/}\\\hline
%         \end{tabular}
%     \end{table}\vspace*{-12pt}
% \end{exmp}

% \noindent
% The following pattern describes these URLs:\vspace*{-3pt}
% \ScaledEqn*{0.5825}{\begin{array}{l}\hspace*{-5pt}
%     \stringliteral{http://www.jetradar.com/flights/} \concat \Upper^+
%     \concat \Alpha^+ \concat \stringliteral{-} \concat \rep{\Upper}{3}
%     \concat \stringliteral{/}
% \end{array}}

%\noindent
%Not only does the pattern succinctly summarize the URLs, but also allows users to create extraction queries for interesting entities, such as airport codes, or names of cities.

\subsection{Synthesis of \bookmarkLP Patterns} \label{subsec:lp-witnesses}

\noindent
Our pattern learner \learnerP uses inductive program synthesis~\citep{gulwani2017program}
for synthesizing patterns that describe a given dataset \dataset using a specified set of atoms \Tokens.
For the convenience of end users, we automatically \emph{enrich} their specified atoms by including:
\begin{inlist}
    \item all possible {\small\tokConst}\ atoms, and
    \item all possible fixed-width variants of all {\small\tokClass}\ atoms specified by them.
\end{inlist}
Our learner \learnerP is instantiated with these enriched atoms derived from \Tokens, denoted as \enrichedTokens:
\begin{equation}\label{eqn:enrich}
\begin{array}{rl}
    \enrichedTokens \ \ \bydef \ \ \Tokens & \cup \quad \{ \patternmode{\tokConst_s} \mid s \in \Strings \} \\[1pt]
                                      & \cup \quad \{ \patternmode{\tokClass^z_c}
                                                      \mid \patternmode{\tokClass^0_c} \in \Tokens
                                                            \;\bigwedge\; z \in \Naturals \}
\end{array}
\end{equation}
Although \enrichedTokens is very large, as we describe below,
our learner \learnerP efficiently explores this search space,
and also provides a completeness guarantee on patterns possible over \enrichedTokens.

We build on \PROSE~\citep{web.prose}, a state-of-the-art inductive program synthesis library,
which implements the \FlashMeta~\citep{polozov2015flashmeta} framework.
\PROSE uses \emph{deductive reasoning} ---
reducing a problem of synthesizing an expression to smaller synthesis problems for its subexpressions,
and provides a robust framework with efficient algorithms and data-structures for this.
Our key contribution towards \learnerP are efficient \emph{witness functions}~\cite[\S 5.2]{polozov2015flashmeta}
that enable \PROSE to carry out the deductive reasoning over \LP.

An inductive program synthesis task is defined by:
\begin{inlist}
    \item a \emph{domain-specific language} (DSL) for the target programs, which in our case is \LP, and
    \item a \emph{specification}~\cite[\S 3.2]{polozov2015flashmeta} (spec) that defines a set of constraints over the output of the desired program.
\end{inlist}
For learning patterns over a collection \dataset of strings,
we define a spec \spec, that simply requires a learned pattern \pattern to describe all given strings,
i.e. $\forall s \in \dataset\colon \patternapply{s} = \True$.
We formally write this as:
\begin{equation*}
    \spec \bydef \bigwedge_{s \,\in\, \dataset}\; [s \tospec \True]
\end{equation*}

%The spec \spec\ may be satisfied by any of the three kinds of \LP patterns $\pattern[s]$, shown in \cref{subfig:lp-syntax}.
We provide a brief overview of the deductive synthesis process here,
and refer the reader to \FlashMeta~\citep{polozov2015flashmeta} for a detailed discussion.
In a deductive synthesis framework, we are required to define the logic for reducing
a spec for an expression to specs for its subexpressions.
The reduction logic for specs, called witness functions~\cite[\S 5.2]{polozov2015flashmeta},
is domain-specific, and depends on the semantics of the DSL.
Witness functions are used to recursively reduce the specs to terminal symbols in the DSL.
\PROSE uses a succinct data structure~\citep[\S 4]{polozov2015flashmeta} to track the valid solutions
to these specs at each reduction and generate expressions that satisfy the initial spec.
For \LP, we describe the logic for reducing the spec \spec\ over the two kinds of patterns:
\patternmode{\EmptySymb} and \patternmode{\patexpr{\SuffixAfter{s}{\token}}}.
For brevity, we elide the pseudocode for implementing the witness functions ---
their implementation is straightforward, based on the reductions we describe below.

For \patternmode{\Empty{s}} to satisfy a spec \spec,
i.e. describe all strings $s \in \dataset$, each string $s$ must indeed be $\epsilon$.
No further reduction is possible since $s$ is a terminal.
We simply check, $\forall s \in \dataset\colon s = \epsilon$,
and reject \patternmode{\Empty{s}} if \dataset contains at least one non-empty string.

The second kind of patterns for non-empty strings, \patternmode{\patexpr{\SuffixAfter{s}{\token}}},
allows for complex patterns that are a composition of an atom \token and a pattern \pattern.
The pattern \patternmode{\patexpr{\SuffixAfter{s}{\token}}} contains two unknowns:
\begin{inlist}
    \item an atom \token that matches a non-empty prefix of $s$, and
    \item a pattern \pattern that matches the remaining suffix \substr{s}{\token(s)}{}
\end{inlist}.
Again, note that this pattern must match all strings $s \in \dataset$.
Na\"{i}vely considering all possible combinations of \token and \pattern leads to an exponential blow up.

First we note that, for a fixed \token the candidates for \pattern can be generated recursively
by posing a synthesis problem similar to the original one,
but over the suffix \substr{s}{\token(s)}{} instead of each string $s$.
This reduction style is called a \emph{conditional witness function}~\citep[\S 5.2]{polozov2015flashmeta},
and generates the following spec for \pattern assuming a fixed \token:
\begin{equation}\label{eqn:spec_atom}
    \spec_\token \bydef \bigwedge_{s \,\in\, \dataset}\; \big[\substr{s}{\token(s)}{} \tospec \True\big]
\end{equation}

However, na\"{i}vely creating $\spec_\token$ for all possible values of \token is infeasible,
since our set \enrichedTokens of atoms is unbounded.
Instead, we consider only those atoms (called \emph{compatible atoms})
that match some non-empty prefix for \emph{all} strings in \dataset,
since ultimately our pattern needs to describe all strings.
Prior pattern-learning approaches~\citep{raman2001potter, singh2016blinkfill}
learn complete patterns for each individual string,
and then compute their intersection to obtain patterns consistent with the entire dataset.
In contrast, we compute the set of atoms that are compatible with the entire dataset at each step,
which allows us to generate this intersection in an incremental manner.
%The compatible atoms describe an intersection over atoms,
%which we perform before synthesizing the subsequent atoms in the pattern.

\begin{defn}{Compatible Atoms}\label{defn:compatible-atoms}
    Given a universe \Tokens of atoms, we say a subset $A \subseteq \Tokens$ is compatible with a dataset \dataset,
    denoted as \compat{A}{\dataset}, if all atoms in $A$ match each string in \dataset,
    i.e. \[
        \compat{A}{\dataset} \quad\textnormal{iff}\quad
        \forall\,\token \in A\,\colon\;
        \forall\,s \in \dataset\,\colon\; \token(s) > 0
    \]
    We say that a compatible set $A$ of atoms is \emph{maximal} under the given universe \Tokens,
    denoted as $A \,=\; \compatmax{\Tokens}{\dataset}$\
    iff\ $\forall\,X \subseteq \Tokens\,\colon\; \compat{X}{\dataset} \Rightarrow X \subseteq A$.
\end{defn}

%To synthesize \patexpr{\SuffixAfter{s}{\token}} that satisfies \spec, we first synthesize the possible values of the suffix \substr{s}{\token(s)}{} consistent with \spec. Then, for each distinct value of the suffix we create a new spec $\spec'$ by replacing $s$ with \substr{s}{\token(s)}{} in \spec.

%The naive approach of mapping each string $s$ to all its suffixes and simply letting \PROSE rule out the impossible cases by continuing with the deductive backpropagation is impractical! In \FlashProfile, we restrict the set of considered tokens to only those which successfully match some prefix of \emph{every} string in the spec. Therefore, \PROSE never rules out any tokens during backpropagation and only explore possible \token-patterns. We call such tokens as being \emph{compatible} with the given specification.

\begin{exmp}\label{exmp:compatible-tokens}
    Consider a dataset with Canadian postal codes:
    \dataset = \{ \patternmode{\stringliteral{V6E3V6}}, \patternmode{\stringliteral{V6C2S6}},
    \patternmode{\stringliteral{V6V1X5}}, \patternmode{\stringliteral{V6X3S4}} \}.
    With $\Tokens = $ the default atoms (listed in \cref{tab:lp-default-atoms}),
    we obtain the enriched set \enrichedTokens using \cref{eqn:enrich}.
    Then, the maximal set of atoms compatible with \dataset under \enrichedTokens,
    i.e. \compatmax{\enrichedTokens}{\dataset} contains $18$ atoms, such as
    \patternmode{\stringliteral{V6}}, \patternmode{\stringliteral{V}},
    \patternmode{\Upper}, \patternmode{\Upper^+}, \patternmode{\AlphaSpace},
    \patternmode{\rep{\AlphaDigit}{6}}, etc.
\end{exmp}

\begin{wrapfigure}{r}{0.54\textwidth}
    \vspace*{1pt}\algobox{1.25}{
        \begin{algfunction}
            {\GetPrefixTokensAlgo}
            {\dataset\colon \Tstring[\,], \Tokens\colon \Tatom[\,]}
            {The maximal set of atoms that are compatible with \dataset}
                \State $\CCTokens \gets \{\} \ ;\  \Lambda \gets \Tokens$
                \ForAll{$s \in \dataset$}
                    \ForAll{$\token \in \Lambda$}
                        \LeftComment{\hspace*{22.5pt}Remove incompatible atoms.}
                        \If{$\token(s) = 0$}
                            $\Lambda.\prop{Remove}(\token)\ ;\ \CCTokens.\prop{Remove}(\token)$
                        \ElsIf{$\token \in \vars{Class}$}
                            \LeftComment{\hspace*{34pt}Check if character class atoms maintain a fixed width.}
                            \If{$\token \not\in \CCTokens$}
                                $\CCTokens[\token] \gets \token(s)$
                            \ElsIf{$\CCTokens[\token] \neq \token(s)$}
                                $\CCTokens.\prop{Remove}(\token)$
                            \EndIf
                        \EndIf
                    \EndFor
                \EndFor
                \LeftComment{Add compatible fixed-width \tokClass\ atoms.}
                \ForAll{$\token \in \CCTokens$}
                    $\Lambda.\prop{Add}(\Call{RestrictWidth}{\token,\CCTokens[\token]})$
                \EndFor
                \LeftComment{Add compatible \tokConst\ atoms.}
                \State $L \gets \Call{LongestCommonPrefix}{\dataset}$
                \State $\Lambda.\prop{Add}(\tokConst_{L[0\,:\,1]},\, \tokConst_{L[0\,:\,2]},\, \dots,\, \tokConst_{L})$
                % \For{$i \gets 1 \ \textbf{to}\ |\,L\,|$}
                % \State $\Lambda.\prop{Add}(\tokConst_{L[0\,:\,i]})$
                % \EndFor
                \Return{$\Lambda$}
    \end{algfunction}}
    \caption{Computing the maximal set of compatible atoms}
    \label{algo:compat-tokens}
\end{wrapfigure}

For a given universe \Tokens of atoms and a dataset \dataset,
we invoke the \GetPrefixTokensAlgo method outlined in \cref{algo:compat-tokens}
to efficiently compute the set $\Lambda = \compatmax{\enrichedTokens}{\dataset}$,
where \enrichedTokens denotes the enriched set of atoms based on \Tokens given by \cref{eqn:enrich}.
Starting with $\Lambda = \Tokens$, in line 1, we iteratively remove atoms from $\Lambda$,
that are not compatible with \dataset, i.e. fail to match at least one string $s \in \dataset$, at line 4.
At the same time, we maintain a hashtable \CCTokens,
which maps each \patternmode{\tokClass}\ atom to its width at line 6.
\CCTokens is used to enrich \Tokens with fixed-width versions of \patternmode{\tokClass}\ atoms
that are already specified in \Tokens.
If the width of a \patternmode{\tokClass}\ atom is not constant over all strings in \dataset,
we remove it from our hashtable \CCTokens, at line 7.
For each remaining \patternmode{\tokClass} atom \token in \CCTokens,
we add a fixed-width variant for \token to $\Lambda$.
In line 8, we invoke \textsc{RestrictWidth}
to generate the fixed-width variant for \token with width $\CCTokens[\token]$.
Finally, we also enrich $\Lambda$ with \tokConst\ atoms ---
we compute the longest common prefix $L$ across all strings,
and add every prefix of $L$ to $\Lambda$, at line 12.
Note that, \GetPrefixTokensAlgo does not explicitly compute the entire set \enrichedTokens of enriched atoms,
but performs simultaneous pruning and enrichment on \Tokens to ultimately compute their maximal compatible subset,
$\Lambda = \compatmax{\enrichedTokens}{\dataset}$.

In essence, the problem of learning an expression \patternmode{\patexpr{\SuffixAfter{s}{\token}}}
with spec \spec\ is reduced to $\big|\!\compatmax{\enrichedTokens}{\dataset}\big|$ subproblems for $P$
with specs $\{ \spec_\token \mid \token \in \compatmax{\enrichedTokens}{\dataset} \}$,
where $\spec_\token$ is as given by \cref{eqn:spec_atom},
and \enrichedTokens denotes the enriched set of atoms derived from \Tokens by \cref{eqn:enrich}.
%Each of these problems is handled the same way as the original one.
Note that these subproblems are recursively reduced further,
until we match all characters in each string, and terminate with \patternmode{\EmptySymb}.
Given this reduction logic as witness functions,
\PROSE performs these recursive synthesis calls efficiently,
and finally combines the atoms to candidate patterns.
We conclude this subsection with a comment on the soundness
and completeness of \learnerP.

\begin{defn}{Soundness and \Tokens-Completeness}
We say that a learner for \LP patterns is \emph{sound} if, for any dataset \dataset,
every learned pattern \pattern satisfies $\forall s \in \dataset: \patternapply{s} = \True$.

We say that a learner for \LP, instantiated with a universe \Tokens of atoms is \emph{\Tokens-complete} if,
for any dataset \dataset, it learns every possible pattern $\pattern \in \LP$ over \Tokens atoms
that satisfy $\forall s \in \dataset: \patternapply{s} = \True$.
\end{defn}

\begin{theorem}[Soundness and \enrichedTokens-Completeness of \learnerP]
    For an arbitrary set \Tokens of user-specified atoms,
    \textnormal{\FlashProfile}'s pattern learner \learnerP is sound and \enrichedTokens-complete,
    where \enrichedTokens denotes the enriched set of atoms obtained by augmenting
    \Tokens with constant-string and fixed-width atoms, as per \cref{eqn:enrich}.
\end{theorem}

\begin{proof}
    Soundness is guaranteed since we only compose compatible atoms.
    \enrichedTokens-completeness follows from the fact that we always consider
    the \emph{maximal} compatible subset of \enrichedTokens.
\end{proof}

Due to the \enrichedTokens-completeness of \learnerP,
once the set $\learnerP(\dataset)$ of patterns over \dataset has been computed,
a variety of cost functions may be used to select the most suitable pattern for \dataset
amongst all possible patterns over \enrichedTokens,
without having to invoke pattern learning again.

\subsection{Cost of Patterns in \bookmarkLP}\label{subsec:lp-ranking}

\noindent
Our cost function \costP assigns a real-valued score to each pattern $\pattern \in \LP$ over a given dataset \dataset,
based on the structure of \pattern and its behavior over \dataset.
This cost function is used to select the most desirable pattern that represents the dataset \dataset.
{\small\EmptySymb} is assigned a cost of $0$ regardless of the dataset,
since {\small\EmptySymb} can be the only pattern consistent with such datasets.
For a pattern $\pattern = \token_1 \concat \ldots \concat \token_k$, we define the cost $\costP(\pattern,\dataset)$
with respect to a given dataset \dataset as:
\begin{equation*}
    \costP(\pattern, \dataset) = \sum_{i = 1}^k \staticCost(\token_i) \cdot \datascore(i, \dataset \mid \pattern)
\end{equation*}

\costP balances the trade-off between a pattern's specificity and complexity.
Each atom \token in \LP has a statically assigned cost $\staticCost(\token) \in (0, \costmax]$,
based on a priori bias for the atom.
%The static costs for the default atoms in \FlashProfile were empirically decided.
%We define $\costmax = \staticCost(\Any^+)$ to be the maximum cost across all atoms.
Our cost function \costP computes a sum over these static costs after applying a data-driven weight
$\datascore(i, \dataset \mid \pattern) \in (0,1)$:
\begin{equation*}
    \datascore(i, \dataset \mid \token_1 \concat \ldots \concat \token_k) = \frac{1}{|\dataset|} \cdot \sum_{s \,\in\, \dataset} \frac{\token_i(s_i)}{|s|},
\end{equation*}
where $s_1 \bydef s$ and $s_{i+1} \bydef \substr{s_i}{\token_i(s_i)}{}$ denotes the remaining suffix of $s_i$ after matching with $\token_i$, as in \cref{def:pattern}.
This dynamic weight is an average over the fraction of length matched by $\token_i$ across \dataset.
It gives a quantitative measure of how well an atom $\token_i$ generalizes over the strings in \dataset.
With a sound pattern learner, an atomic match would never fail
and $\datascore(i, \dataset \mid \pattern) > 0$ for all atoms $\token_i$.

\begin{exmp}
Consider $\dataset = \stringPair{Male}{Female}$, that are matched by
\patternmode{\pattern_1 = \Upper\concat\Lower^+},
and \patternmode{\pattern_2 = \Upper\concat\HexDigit\concat\Lower^+}.
Given \FlashProfile's static costs: \patternmode{\{ \Upper \mapsto 8.2,\; \HexDigit \mapsto 26.3,\; \Lower^+ \mapsto 9.1 \}},
the costs for these two patterns shown above are:
\begin{equation*}
\begin{split}
    \costP(\pattern_1, \dataset) \;&=\; 8.2 \times \frac{\nicefrac{1}{4}+\nicefrac{1}{6}}{2} \,+\, 9.1 \times \frac{\nicefrac{3}{4}+\nicefrac{5}{6}}{2} \;=\; 8.9
    \\
    \costP(\pattern_2, \dataset) \;&=\; 8.2 \times \frac{\nicefrac{1}{4}+\nicefrac{1}{6}}{2} \,+\, 26.3 \times \frac{\nicefrac{1}{4}+\nicefrac{1}{6}}{2} \,+\, 9.1 \cdot \frac{\nicefrac{2}{4}+\nicefrac{4}{6}}{2} \;=\; 12.5
\end{split}
\end{equation*}

\noindent
$\pattern_1$ is chosen as best pattern, since $\costP(\pattern_1, \dataset) < \costP(\pattern_2, \dataset)$.
\end{exmp}

Note that although \HexDigit is a more specific character class compared to \Upper and \Lower,
we assign it a higher static cost to avoid strings like \stringliteral{face}
being described as \rep{\HexDigit}{4} instead of \rep{\Lower}{4}.
\rep{\HexDigit}{4} would be chosen over \rep{\Lower}{4} only if we observe some other string in the dataset,
such as \stringliteral{f00d}, which cannot be described using \rep{\Lower}{4}.
%Since $\staticCost(\HexDigit)$ is much lower compared to $\staticCost(\AlphaDigit) = 639.6$,
%it is the preferred pattern for strings such as \stringliteral{f00d}.

% TODO --- May be change the example above. It's almost useless.

\paragraph{Static Cost ($\staticCost$) for Atoms}
Our learner \learnerP automatically assigns the static cost of a \patternmode{\tokConst_s} atom to be proportional to \resizebox{!}{0.8\baselineskip}{$\nicefrac{1}{|s|}$},
and the cost of a \patternmode{\tokClass^z_c} atom, with width $z \geqslant 1$,
to be proportional to \resizebox{!}{1.2\baselineskip}{$\frac{Q(\tokClass^0_c)}{z}$}.
However, static costs for other kinds of atoms must be provided by the user.

Static costs for our default atoms, listed in \cref{tab:lp-default-atoms},
were seeded with the values based on their estimated \emph{size} ---
the number of strings the atom may match.
Then they were penalized (e.g. the \HexDigit atom) with empirically decided penalties
to prefer patterns that are more \emph{natural} to users.
We describe our \emph{quality} measure for profiles in \cref{subsec:eval-profiling-accuracy},
which we have used to decide the penalties for the default atoms.
In future, we plan to automate the process of penalizing atoms by designing a learning procedure
which tries various perturbations to the seed costs to optimize profiling quality.

%In \cref{subsec:eval-profiling-accuracy}, we evaluate the quality of profiles learned by \FlashProfile,
%and show that profiles we obtain are \emph{natural} -- neither too specific, nor overly general.

%We conclude with a comment on our implementation of \costP. \PROSE only supports \emph{maximization} of an objective function over a set of synthesized programs, whereas our cost function needs to be minimized in \BestPatternAlgo (line 4, \cref{fig:pattern-learning-algo}). Therefore, our implementation uses negated static costs $\costP'(\token_i) = -\costP(\token_i)$, and computes the cost $\costP'\,=\, \sum_{i = 1}^n \costP'(\token_i) \cdot \datascore(i, \pattern, \dataset)$. We use $\argmax$ instead of $\argmin$ in \BestPatternAlgo, and subsequently return the cost of the best pattern \pattern as $\sqrt{-\costP(\pattern, \dataset)}$.

%Owing to the compact structure of a VSA and our monotonic (over the VSA structure) ranking function, we can compute the most natural pattern: $\argmax_{\pattern \,\in\, \vartheta(\dataset)}\;\cost(\pattern, \dataset)$, where $\vartheta(\dataset)$ denotes the VSA structure of patterns, produced by the synthesis procedure over dataset \dataset. \fix{Do I need to explain more?}

\section{Evaluation} \label{sec:evaluation}

\noindent
We now present experimental evaluation of the \FlashProfile tool which implements our technique,
focusing on the following key questions:
\begin{itemize}[topsep=3pt,itemsep=1pt,leftmargin=11mm]
    \item[(\,\cref{subsec:eval-syntactic-similarity}\,)]
         How well does our syntactic similarity measure capture similarity over real world entities?
    \item[(\,\cref{subsec:eval-profiling-accuracy}\,)]
         How accurate are the profiles?
         How do sampling and approximations affect their quality?
    \item[(\,\cref{subsec:eval-performance}\,)]
         How fast is \FlashProfile, and how does its performance depend on the various parameters?
    \item[(\,\cref{subsec:eval-case-studies}\,)]
         Are the profiles natural and useful?
         How do they compare to those from existing tools?
\end{itemize}

\paragraph{Implementation}
We have implemented \FlashProfile as a cross-platform \CSharp library built using Microsoft \PROSE~\citep{web.prose}.
It is now publicly available as part of the \PROSE \tool{NuGet} package.\nsp\footnoteNum{
  ~\FlashProfile has been publicly released as the \texttt{Matching.Text} module within the \PROSE library.
  For more information, please see: \url{https://microsoft.github.io/prose/documentation/matching-text/intro/}.
} All of our experiments were performed with \PROSE 2.2.0 and \tool{.NET Core} 2.0.0,
on an Intel i7 3.60\,GHz machine with 32\,GB RAM running 64-bit \tool{Ubuntu} 17.10.

\begin{wrapfigure}{r}{0.51\textwidth}\vspace*{12pt}
  \hspace*{-3.5pt}\centering
  \includegraphics[width=1.02\linewidth]{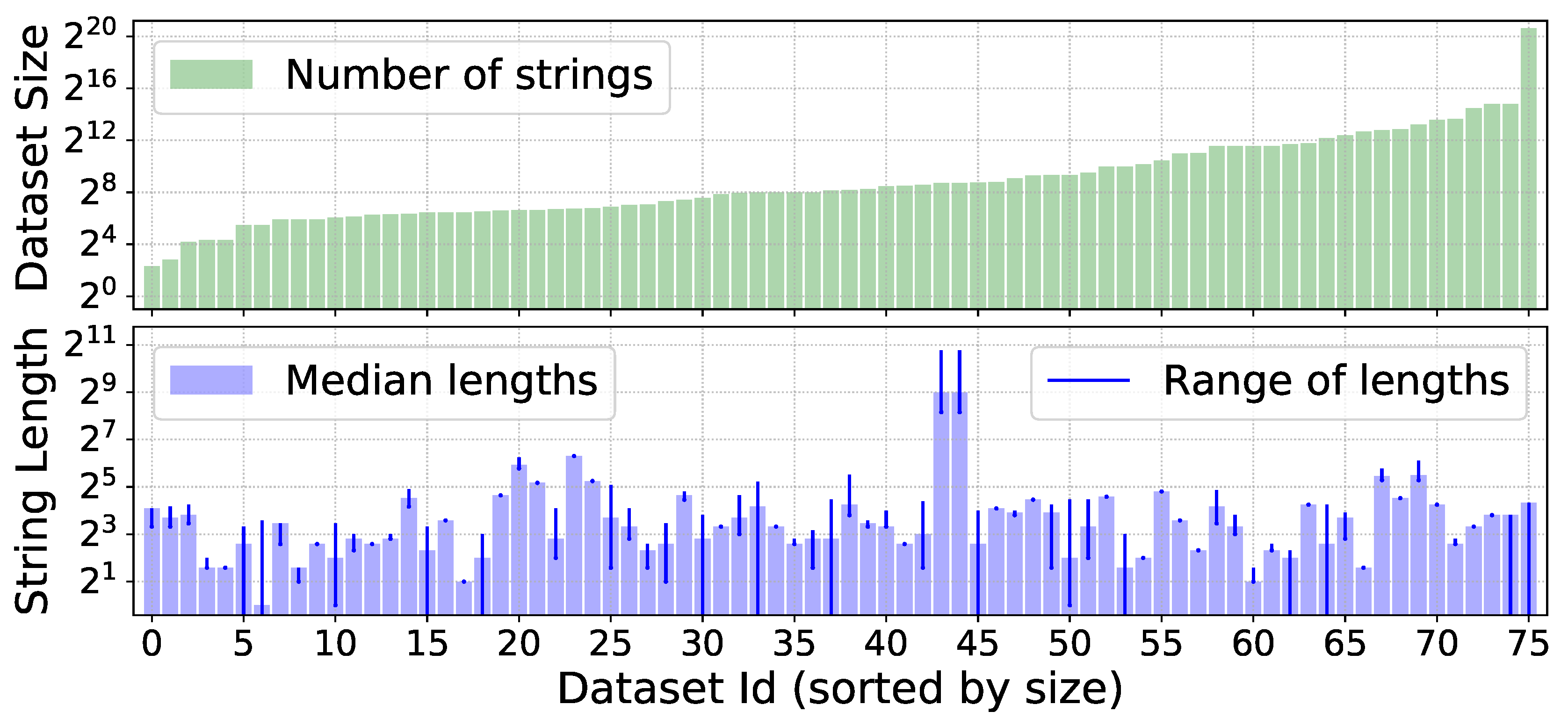}
  \captionsetup{skip=4pt}
  \caption{Number and length of strings across datasets\protect\FootnotemarkNum{\ref{foot:datasets}}}
  \label{fig:eval-data}
\end{wrapfigure}

\paragraph{Test Datasets}
We have collected $75$ datasets from public sources,\nsp\footnoteNum{\label{foot:datasets}
  All public datasets are available at: \redact{\github{tests}}.}
spanning various domains such as names, dates, postal codes, phone numbers, etc.
Their sizes and the distribution of string lengths is shown in \cref{fig:eval-data}.
We sort them into three (overlapping) groups:

\begin{itemize}[topsep=2.5pt,leftmargin=4mm]
  \item \dsetgroup{Clean} (25 datasets):
        Each of these datasets, uses a \emph{single format} that is \emph{distinct} from other datasets.
        We test syntactic similarity over them --- strings from the same dataset must be labeled as similar.
\end{itemize}

\begin{itemize}[topsep=0pt,leftmargin=4mm]
  \item \dsetgroup{Domains} (63 datasets):
        These datasets belong to \emph{mutually-distinct domains} but may exhibit multiple formats.
        We test the quality of profiles over them ---
        a profile learned over fraction of a dataset should match rest of it,
        but should not be too general as to also match other domains.
  \item \dsetgroup{All} (75 datasets):
        We test \FlashProfile's performance across all datasets.
\end{itemize}

\subsection{Syntactic Similarity}\label{subsec:eval-syntactic-similarity}

\noindent
%A key advantage of our technique is that, it can be used to evaluate similarity between a pair of strings or even a pair of set of strings, using our syntactic dissimilarity measure given in \cref{defn:dissimilarity}.
We evaluate the applicability of our dissimilarity measure from \cref{defn:dissimilarity}, over real-life entities.
From our \dsetgroup{Clean} group, we randomly pick two datasets and select a random string from each.
A good similarity measure should recognize when the pair is drawn from the same dataset by
assigning them a lower dissimilarity value, compared to a pair from different datasets.
For example, the pair \stringPair{A. Einstein}{I. Newton} should have a lower dissimilarity value than
\stringPair{A. Einstein}{03/20/1998}.
We instantiated \FlashProfile with only the default atoms listed in \cref{tab:lp-default-atoms} and tested $240400$ such pairs.
\Cref{fig:eval-similarity} shows a comparison of our method against two simple baselines:
\begin{inlist}
  \item a character-wise edit-distance-based similarity measure (\texttt{JarW}), and
  \item a machine-learned predictor (\SimBaseline) over intuitive syntactic features.
\end{inlist}

For evaluation, we use the standard precision-recall (PR)~\citep{manningir} measure.
In our context, precision is the fraction of pairs that truly belongs to the same dataset,
out of all pairs that are labeled to be ``similar'' by the predictor.
Recall is the fraction of pairs retrieved by the predictor, out of all pairs truly drawn from same datasets.
By varying the threshold for labelling a pair as ``similar'', we generate a PR curve
and measure the area under the curve (AUC).
A good similarity measure should exhibit high precision and high recall, and therefore have a high AUC.

% HACK: Not sure why the number was being incremented by 2.
\addtocounter{footnoteNum}{-1}

\begin{wrapfigure}{r}{0.5\textwidth}
    \vspace*{-2pt}\centering\hspace*{-10pt}
    \subcaptionbox{Precision-Recall curves\label{fig:eval-pr-curves}}{
        \centering
        \includegraphics[width=0.725\linewidth]{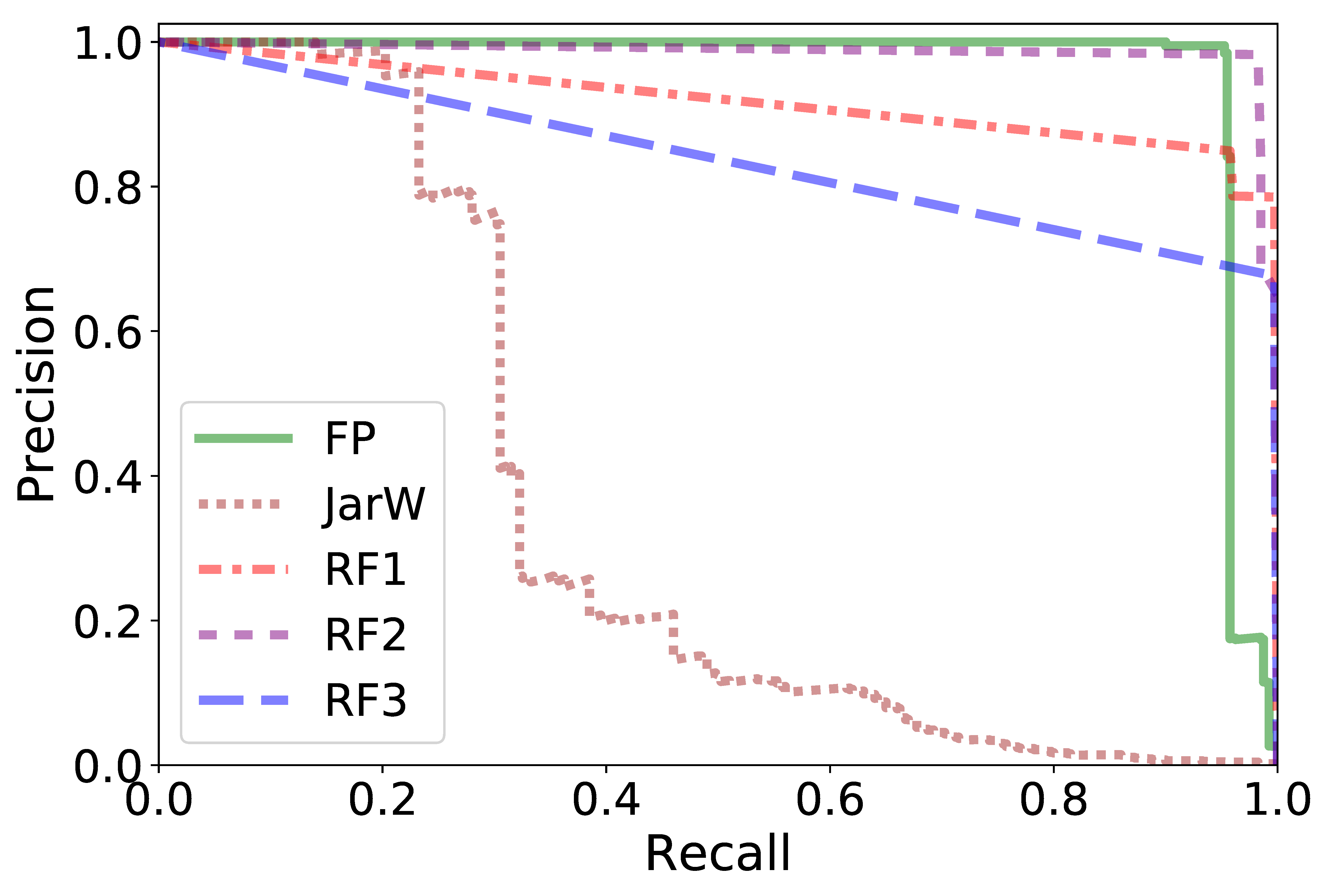}
    }\subcaptionbox{Features\protect\footnotemarkNum\label{tab:baseline-features}}{
        \centering\hspace*{-8pt}\smaller[2]
        \def\arraystretch{1.1}
        \begin{tabular}{|C{6em}|}\hline
        $\Delta[\strlen]$\\
        $\Delta[\strcount\langle\Digit\rangle]$\\
        $\Delta[\strcount\langle\Lower\rangle]$\\
        $\Delta[\strcount\langle\Upper\rangle]$\\
        $\Delta[\strcount\langle\;\Space\,\rangle]$\\
        $\Delta[\strcount\langle\stringliteral{.}\rangle]$\\
        $\Delta[\strcount\langle\stringliteral{,}\rangle]$\\
        $\Delta[\strcount\langle\stringliteral{-}\rangle]$\\
        $\strbegin\langle\Upper\rangle$\\
        $\strbegin\langle\Lower\rangle$\\
        $\strbegin\langle\Digit\rangle$\\\hline
        \end{tabular}
    }\\[3pt]\hspace*{-18pt}
    \begin{subfigure}[b]{1.15\linewidth}
        \centering\smaller[1.5]
        \def\arraystretch{1.05}
        \begin{tabularx}{0.875\linewidth}{|X|*{5}{C{2.6em}|}}\hline
            \centering\cellcolor{bgcolor}\emph{Predictor} &
            \cellcolor{bgcolor}\FP &
            \cellcolor{bgcolor}\tool{JarW} &
            \cellcolor{bgcolor}$\SimBaseline_1$ &
            \cellcolor{bgcolor}$\SimBaseline_2$ &
            \cellcolor{bgcolor}$\SimBaseline_3$ \\\hline
            \centering{}\cellcolor{bgcolor}\emph{AUC} &
            96.28\% &
            35.52\% &
            91.73\% &
            98.71\% &
            76.82\%\\\hline
        \end{tabularx}
    \end{subfigure}
    \caption{Similarity prediction accuracy of \FlashProfile (\FP)
             \emph{vs.} a character-based measure (\texttt{JarW)},
             and random forests ({\small$\SimBaseline_{1\ldots3}$})
             trained on different distributions\vspace{-6pt}}
    \label{fig:eval-similarity}
\end{wrapfigure}

\footnotetextNum{~$\strlen$ returns string length, $\strbegin\langle{}X\rangle$ checks if both strings begin with a character in $X$, $\strcount\langle{}X\rangle$ counts occurrences of characters from $X$ in a string, and $\Delta[f]$ computes $|f(s_1) - f(s_2)|^2$ for a pair of strings $s_1$ and $s_2$.}

First, we observed that character-based measures~\citep{gomaa2013survey} show poor AUC,
and are not indicative of syntactic similarity.
Levenshtein distance~\citep{levenshtein1966binary}, used for string clustering by \OpenRefine~\citep{web.openrefine}, a popular data-wrangling tool,
exhibits a negligible AUC over our benchmarks.
Although the Jaro-Winkler distance~\citep{winkler1999state}, indicated as \texttt{JarW} in \cref{fig:eval-pr-curves},
shows a better AUC, it is quite low compared to both our and machine-learned predictors.

Our second baseline is a standard random forest~\citep{breiman2001random} model \SimBaseline using the syntactic features
listed in \cref{tab:baseline-features}, such as difference in length, number of digits, etc.
We train ${\SimBaseline}_1$ using $\sim 80,000$ pairs with $\big(\nicefrac{1}{25}\big)^2 = 0.16\%$ pairs drawn from same datasets.
We observe from \cref{fig:eval-pr-curves} that the accuracy of \SimBaseline is quite susceptible to changes in the
distribution of the training data.
${\SimBaseline}_2$ and ${\SimBaseline}_3$ were trained with $0.64\%$ and $1.28\%$ pairs from same datasets, respectively.
While ${\SimBaseline}_2$ performs marginally better than our predictor, ${\SimBaseline}_1$ and ${\SimBaseline}_3$ perform worse.
%In contrast, our technique does not require training, and is therefore not biased by data distribution.

\subsection{Profiling Accuracy}\label{subsec:eval-profiling-accuracy}

\noindent
We demonstrate the accuracy of \FlashProfile along two dimensions:
\begin{itemize}[leftmargin=5mm]
  \item \emph{Partitions:} Our sampling and approximation techniques preserve partitioning accuracy
  \item \emph{Descriptions:} Profiles generated using \learnerP and \costP are natural, not overly specific or general.
\end{itemize}
For these experiments, we used \FlashProfile with only the default atoms.

\paragraph{Partitioning}
For each $c \in \{2,\ldots,8\}$, we measure \FlashProfile's ability to repartition $256c$ strings ---
$256$ strings collected from each of $c$ randomly picked datasets from \dsetgroup{Clean}.
Over $10$ runs for each $c$, we pick different sets of $c$ files,
shuffle the $256c$ strings, and invoke \FlashProfile to partition them into $c$ clusters.
For a fair distribution of strings across each run, we ignore one dataset from the \dsetgroup{Clean} group which had much longer strings (> 1500 characters)
compared to other datasets (10 -- 100 characters).
We experiment with different values of $1.0 \leqslant \pointSampleFactor \leqslant 5.0$
(\emph{string-sampling factor}, which controls the size of chunks given to the core \ProfilingAlgo method),
and $1.0 \leqslant \edgeSampleFactor \leqslant 3.0$
(\emph{pattern-sampling factor}, which controls the approximation during hierarchical clustering).

We measure the precision of clustering using \emph{symmetric uncertainty}~\citep{witten2016data},
which is a measure of normalized mutual information (NMI).
An NMI of 1 indicates the resulting partitioning to be identical to the original clusters,
and an NMI of 0 indicates that the final partitioning is unrelated to the original one.
For each $\langle\pointSampleFactor,\edgeSampleFactor\rangle$-configuration,
we show the mean NMI of the partitionings over $10c$ runs ($10$ for each value of $c$), in \cref{fig:eval-nmi}.
The NMI improves with increasing \edgeSampleFactor, since we sample more dissimilarities, resulting in better approximations.
However, the NMI drops with increasing \pointSampleFactor, since more pairwise dissimilarities are approximated.
Note that the number of string pairs increases quadratically with \pointSampleFactor,
but reduces only linearly with \edgeSampleFactor.
This is reflected in \cref{fig:eval-nmi} -- for $\pointSampleFactor > 4.0$,
the partitioning accuracy does not reach 1.0 even for $\edgeSampleFactor = 3.0$.
\FlashProfile's default configuration $\langle \pointSampleFactor = 4.0, \edgeSampleFactor = 1.25 \rangle$,
achieves a median NMI of 0.96 (mean 0.88) (indicated by a circled point).
The dashed line indicates the median NMIs with $\pointSampleFactor = 4.0$.
The median NMI is significantly higher than the mean, indicating our approximations were accurate in most cases.
As we explain below in \cref{{subsec:eval-performance}}, with $\langle \pointSampleFactor = 4.0, \edgeSampleFactor = 1.25\rangle$,
\FlashProfile achieves the best performance \emph{vs.} accuracy trade-off.

\begin{figure}[t]
    \centering
    \includegraphics[width=0.9\linewidth]{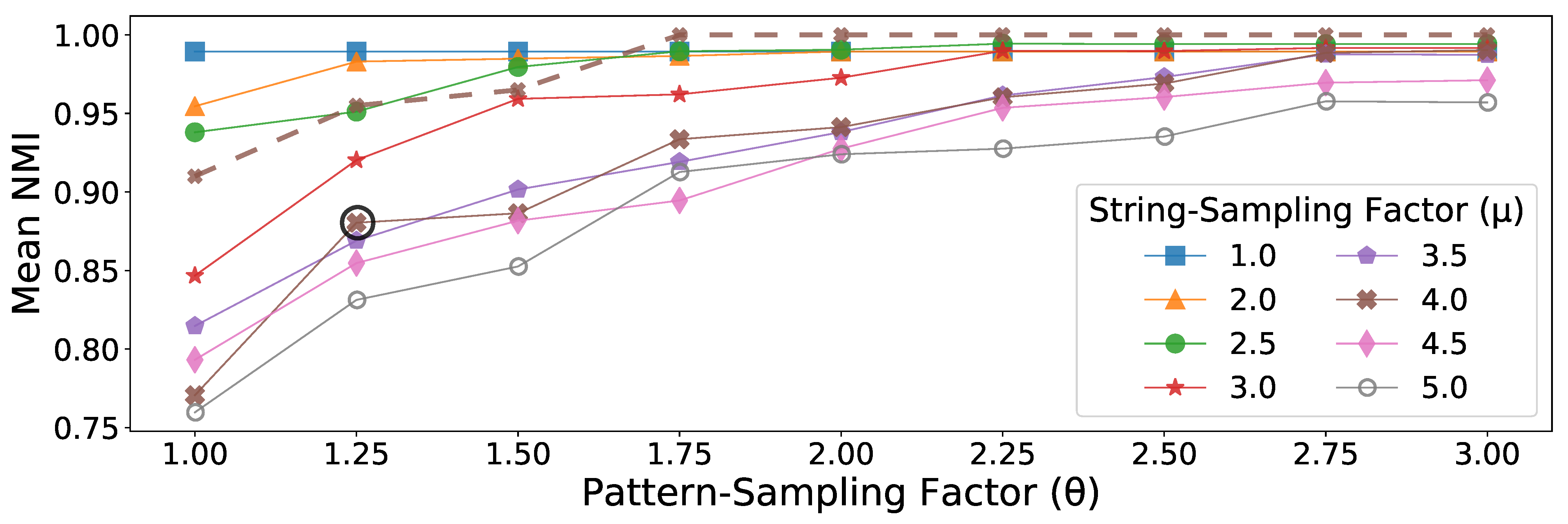}
    \captionsetup{skip=1pt}
    \caption{\FlashProfile's partitioning accuracy with different $\langle\pointSampleFactor,\edgeSampleFactor\rangle$-configurations\vspace*{-6pt}}
    \label{fig:eval-nmi}
\end{figure}

%for a specific $(\edgeSampleFactor, \pointSampleFactor)$ pair, \FlashProfile samples only $\pointSampleFactor\maxClusters$ strings and computes exact dissimilarities for all strings with respect to an adaptively sampled set of $\edgeSampleFactor\maxClusters$ points.

%\vspace*{-12pt}
\paragraph{Descriptions}
We evaluate the suitability of the automatically suggested profiles, by measuring their overall precision and recall.
A natural profile should not be too specific -- it should generalize well over the dataset (high true positives),
but not beyond it (low false positives).

\begin{wrapfigure}{r}{0.5\textwidth}
    \vspace*{-1pt}\hspace*{-4pt}\centering
    \includegraphics[width=\linewidth]{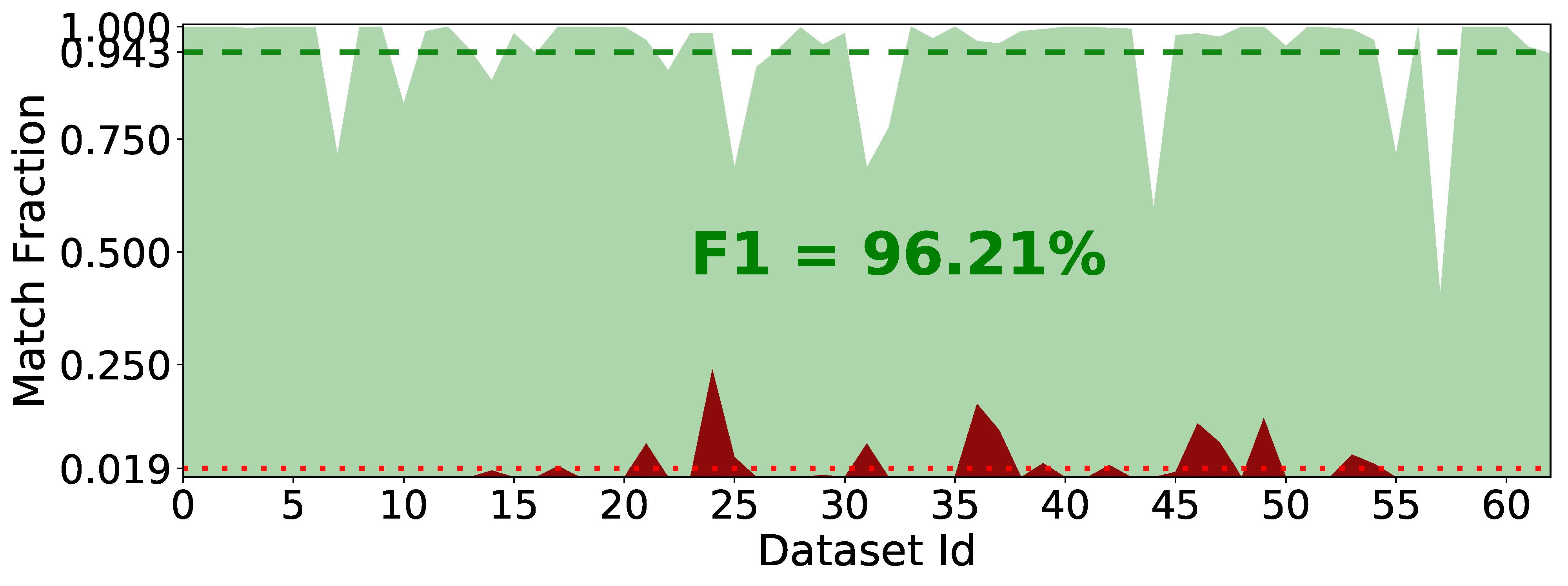}
    \captionsetup{skip=0pt}
    \caption{Quality of descriptions at $\langle \pointSampleFactor = 4.0, \edgeSampleFactor = 1.25 \rangle$\vspace*{4pt}}
    \label{fig:eval-overall-pr}
\end{wrapfigure}

For each dataset in our \dsetgroup{Domains}, we profile a randomly selected 20\% of its strings, and measure:
\begin{inlist}
  \item the fraction of the remaining dataset described by it, and
  \item the fraction of an equal number of strings from other datasets, matched by it. % synthetic dataset created by randomly sampling strings from other datasets, matched by it.
\end{inlist}
\Cref{fig:eval-overall-pr} summarizes our results.
The lighter and darker shades indicate the fraction of true positives and false positives respectively.
The white area at the top indicates the fraction of false negatives --
the fraction of the remaining 80\% of the dataset that is not described by the profile.
We observe an overall precision of $97.8\%$, and a recall of $93.4\%$.
The dashed line indicates a mean true positive rate of $93.2\%$,
and the dotted line shows a mean false positive rate of $2.3\%$; across all datasets.

We also perform similar quality measurements for profiles generated by
\tool{Microsoft SSDT}~\citep{web.ssdt} and \Ataccama~\citep{web.ataccama}.
We use ``Column Pattern Profiling Tasks'' with {\small\texttt{PercentageDataCoverageDesired}} = 100 within \SSDT,
and ``Pattern Analysis'' feature within the \Ataccama platform.
We summarize the per-dataset description quality for \SSDT in \cref{fig:ssdt-quality},
and for \Ataccama in \cref{fig:ataccama-quality}.
We observe a low overall F1 score for both tools.

\begin{figure}[t]
    \begin{minipage}{0.495\textwidth}
        \centering
        \subcaptionbox{\tool{Microsoft SSDT}~\citep{web.ssdt}\label{fig:ssdt-quality}}{
            \hspace*{-2pt}\includegraphics[width=\linewidth]{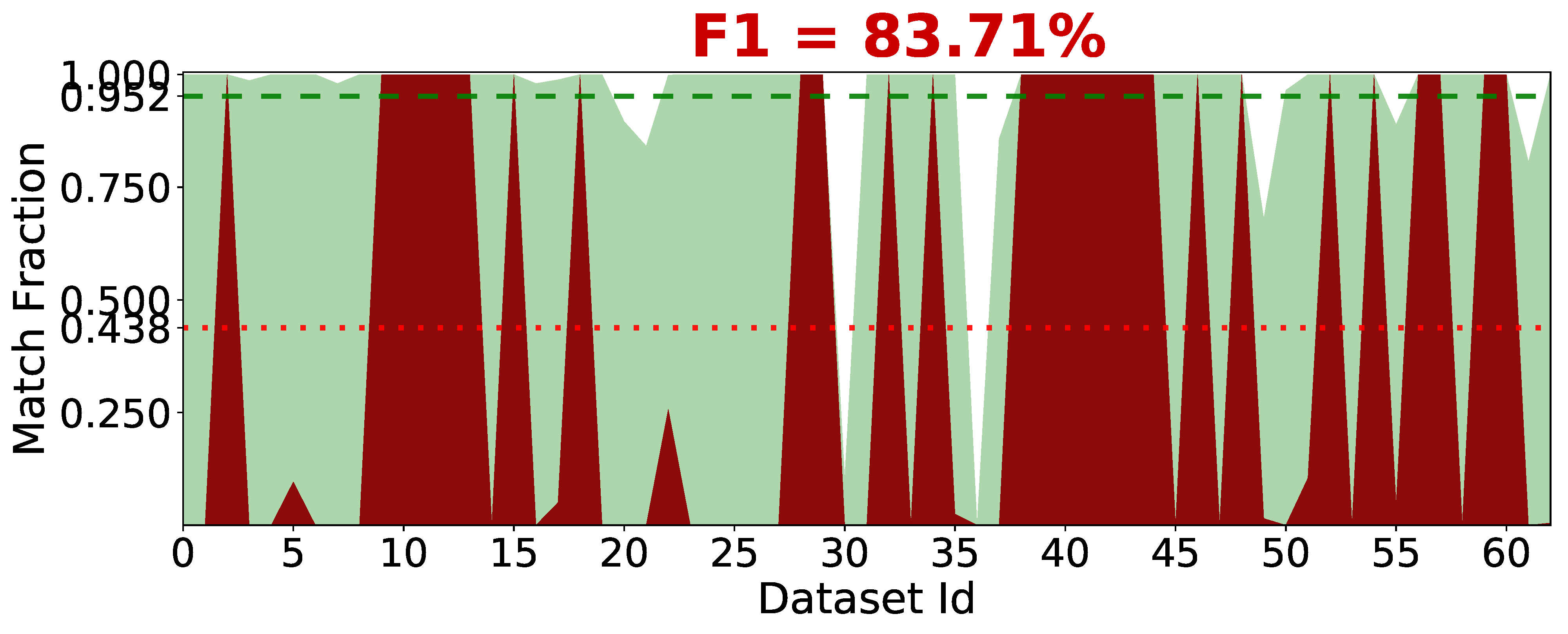}
        }
    \end{minipage}\hfill\begin{minipage}{0.495\textwidth}
        \centering
        \subcaptionbox{\Ataccama~\citep{web.ataccama}\label{fig:ataccama-quality}}{
            \hspace*{-2pt}\includegraphics[width=\linewidth]{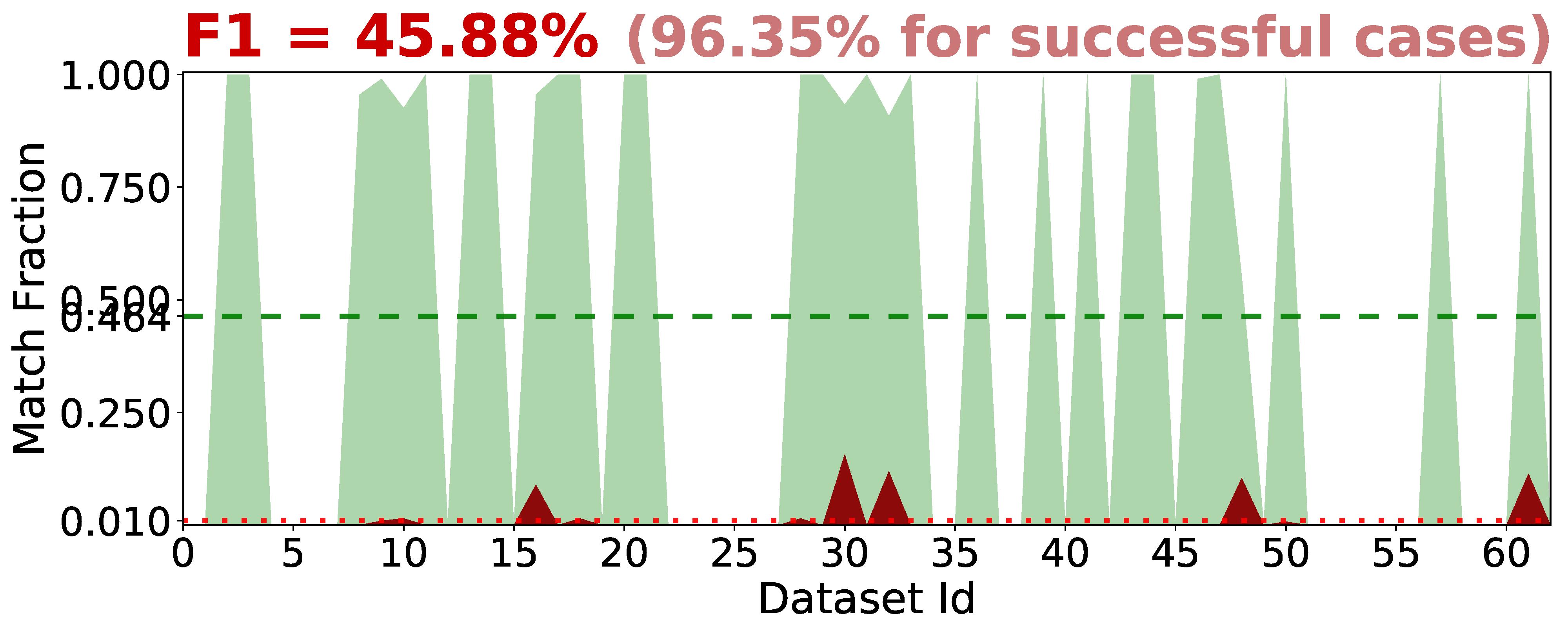}
        }
    \end{minipage}
    \captionsetup{skip=2pt}
    \caption{Quality of descriptions from current state-of-the-art tools\vspace*{-10pt}}
    \label{fig:eval-other-mismatch}
\end{figure}

While \SSDT has a very high false positive rate, \Ataccama has a high failure rate.
For $27$ out of $63$ datasets, \SSDT generates \stringliteral{.*} as one of the patterns,
and it fails to profile one dataset that has very long strings (up to $1536$ characters).
On the other hand, \Ataccama fails to profile $33$ datasets.
But for the remaining $30$ datasets, the simple atoms (digits, numbers, letters, words)
used by \Ataccama seem to work well --- the profiles exhibit high precision and recall.
Note that, this quantitative measure only captures the specificity of profiles, but not their readability.
We present a qualitative comparison of profiles generated by these tools in \cref{subsec:eval-case-studies}.

\begin{figure}[b]
    \vspace*{-12pt}\centering
    \subcaptionbox{Mean Profiling Time\label{fig:eval-perf-sampling}}{
        \includegraphics[width=0.56\linewidth]{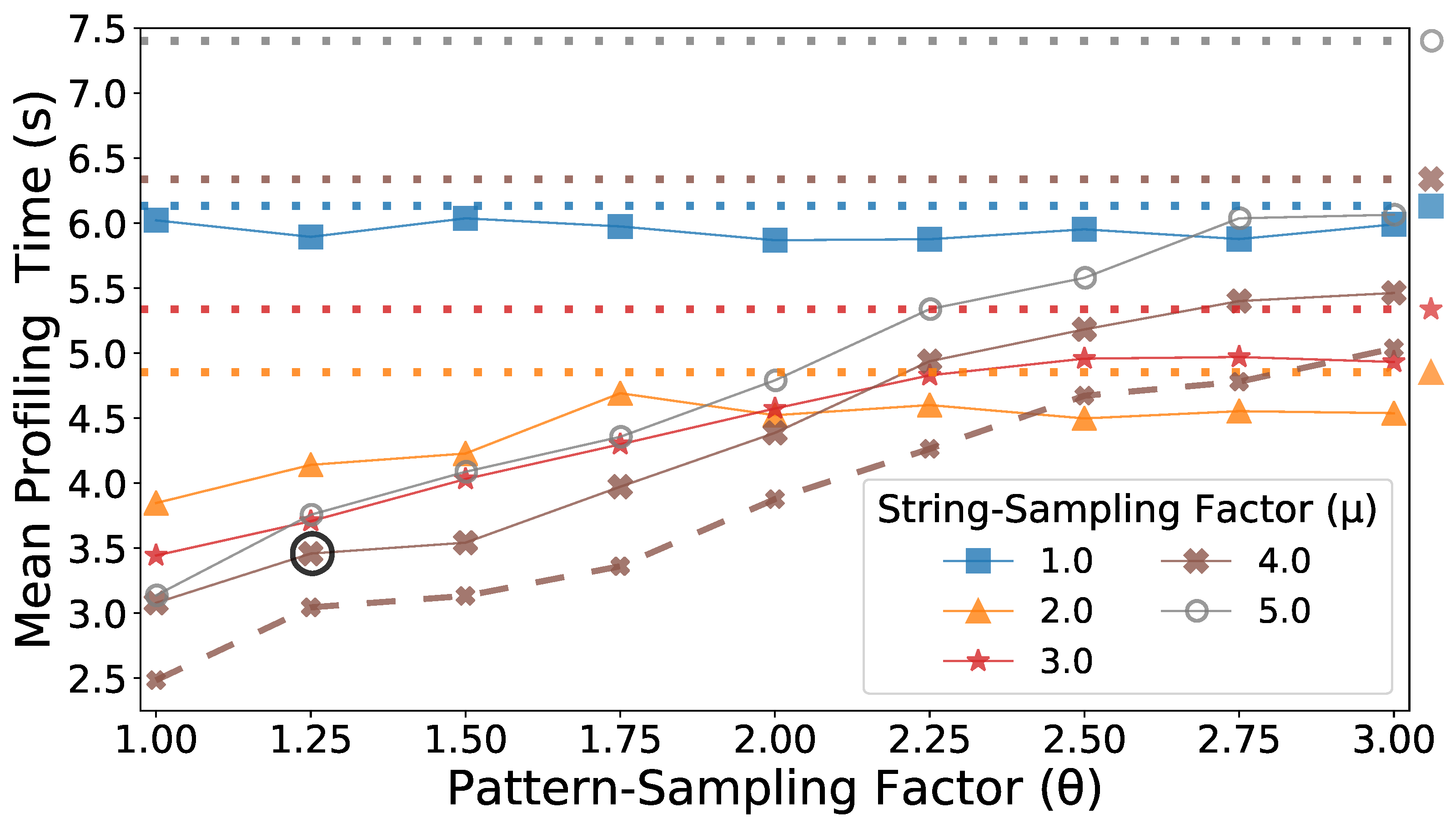}\vspace*{-2pt}
    }\hspace*{16pt}\subcaptionbox{Performance $\sim$ Accuracy\label{fig:eval-perf-nmi}}{
        \includegraphics[width=0.285\linewidth]{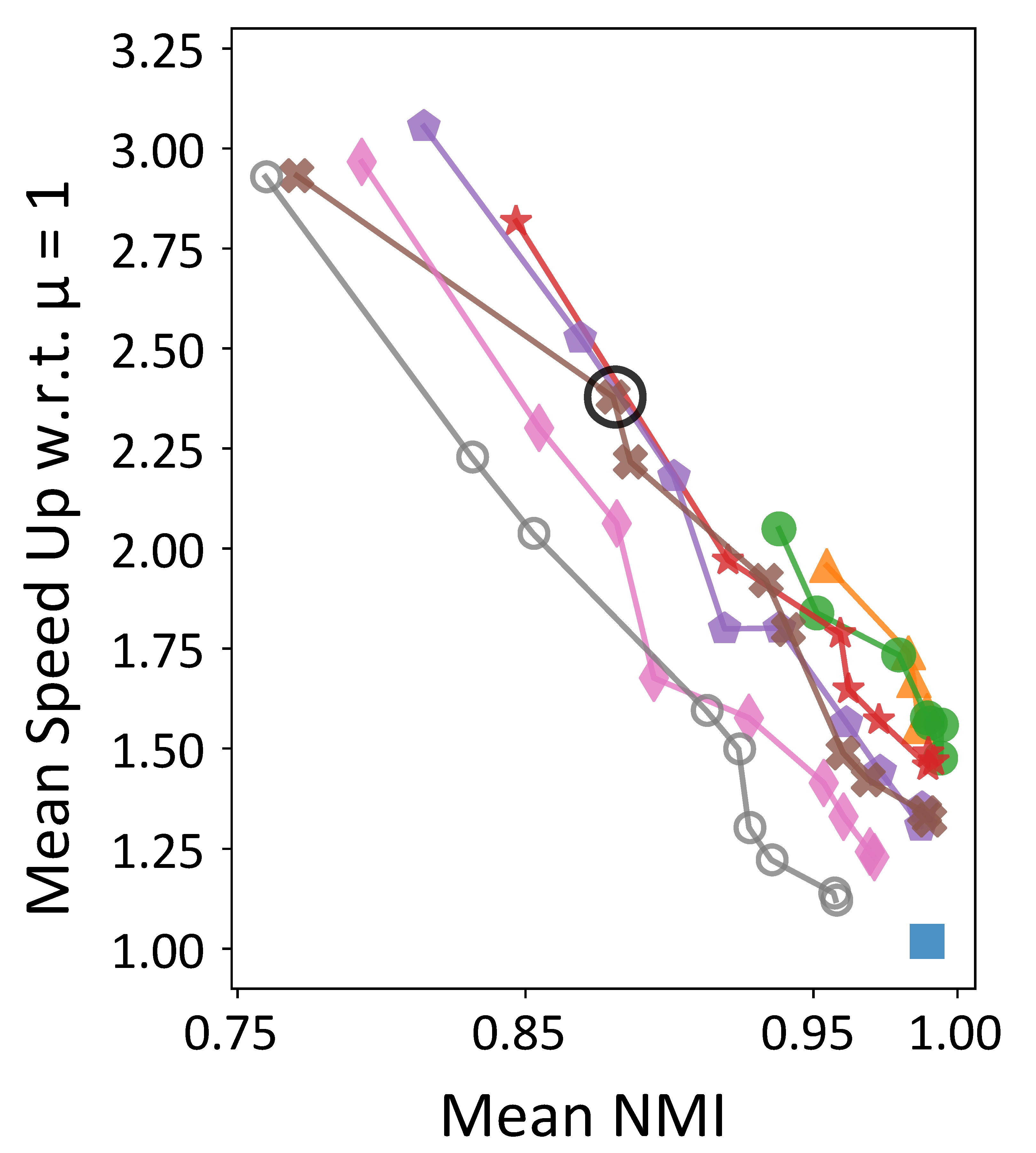}\vspace*{-2pt}
    }
    \captionsetup{skip=2pt}
    \caption{Impact of sampling on performance (using the same colors and markers as \cref{fig:eval-nmi})\vspace*{-8pt}}
    \label{fig:eval-perf}
\end{figure}

\subsection{Performance}\label{subsec:eval-performance}

\noindent
We measure the mean profiling time with various $\langle \pointSampleFactor, \edgeSampleFactor \rangle$-configurations,
and summarize our findings in \cref{fig:eval-perf-sampling}.
The dotted lines indicate profiling time without pattern sampling, i.e. $\edgeSampleFactor \rightarrow \infty$,
for different values of the \pointSampleFactor factor.
The dashed line shows the median profiling time for different values of \edgeSampleFactor
with our default $\pointSampleFactor = 4.0$.
We also show the performance-accuracy trade off in \cref{fig:eval-perf-nmi}
by measuring the mean speed up of each configuration w.r.t. $\langle \pointSampleFactor = 1.0, \edgeSampleFactor = 1.0 \rangle$.
We select the \emph{Pareto optimal} point $\langle \pointSampleFactor = 4.0, \edgeSampleFactor = 1.25 \rangle$
as \FlashProfile's default configuration.
It achieves a mean speed up of $2.3\times$ over $\langle \pointSampleFactor = 1.0, \edgeSampleFactor = 1.0 \rangle$,
at a mean NMI of $0.88$ (median = $0.96$).

As one would expect, the profiling time increases with \edgeSampleFactor, due to sampling more patterns and making more calls to \learnerP.
The dependence of profiling time on \pointSampleFactor however, is more interesting.
Notice that with $\pointSampleFactor = 1$, the profiling time is \emph{higher} than any other configurations,
when pattern sampling is enabled, i.e. $\edgeSampleFactor \neq \infty$ (solid lines).
This is due to the fact that \FlashProfile learns very specific profiles with $\pointSampleFactor = 1$ with very small samples of strings,
which do not generalize well over the remaining data.
This results in many \textsf{Sample}$-\ProfilingAlgo-$\textsf{Filter} iterations.
Also note that with pattern-sampling enabled,
the profiling time decreases with \pointSampleFactor until $\pointSampleFactor = 4.0$ as,
and then increases as profiling larger samples of strings becomes expensive.

Finally, we evaluate \FlashProfile's performance on end-to-end real-life profiling tasks on all $75$ datasets,
that have a mixture of clean and dirty datasets.
Over $153$ tasks -- $76$ for automatic profiling, and $77$ for refinement,
we observe a median profiling time of $0.7$\,s.
With our default configuration, $77\%$ of the requests are fulfilled within $2$ seconds --
$70\%$ of automatic profiling tasks, and $83\%$ of refinement tasks.
In \cref{fig:eval-perf-real} we show the variance of profiling times w.r.t. size of the datasets (number of strings in them),
and the average length of the strings in the datasets (all axes being logarithmic).
We observe that the number of string in the dataset doesn't have a strong impact on the profiling time.
This is expected, since we only sample smaller chunks of datasets,
and remove strings that are already described by the profile we have learned so far.
We repeated this experiment with $5$ dictionary-based custom atoms:
{\small$\Regex{DayName}$, $\Regex{ShortDayName}$, $\Regex{MonthName}$, $\Regex{ShortMonthName}$, $\Regex{US\_States}$},
and noticed an increase of $\sim0.02$\,s in the median profiling time.

\subsection{Comparison of Learned Profiles}\label{subsec:eval-case-studies}

% HACK: Not sure why the number was being incremented by 2.
\addtocounter{footnoteNum}{-1}

\begin{figure}[t]
    \begin{minipage}{0.525\linewidth}
      \vspace*{3.5pt}\centering
      \hspace*{-3pt}\includegraphics[width=0.965\linewidth]{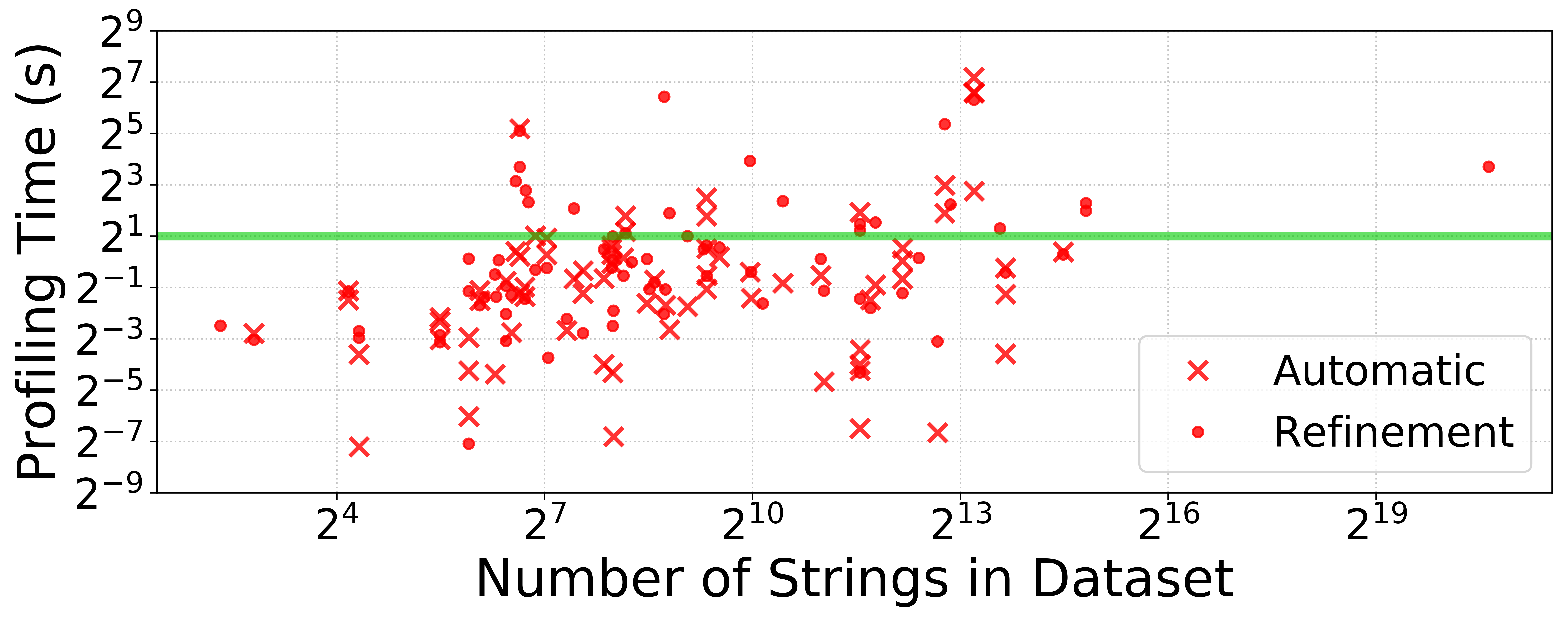}\\[-2.5pt]
      \hspace*{-3pt}\includegraphics[width=0.965\linewidth]{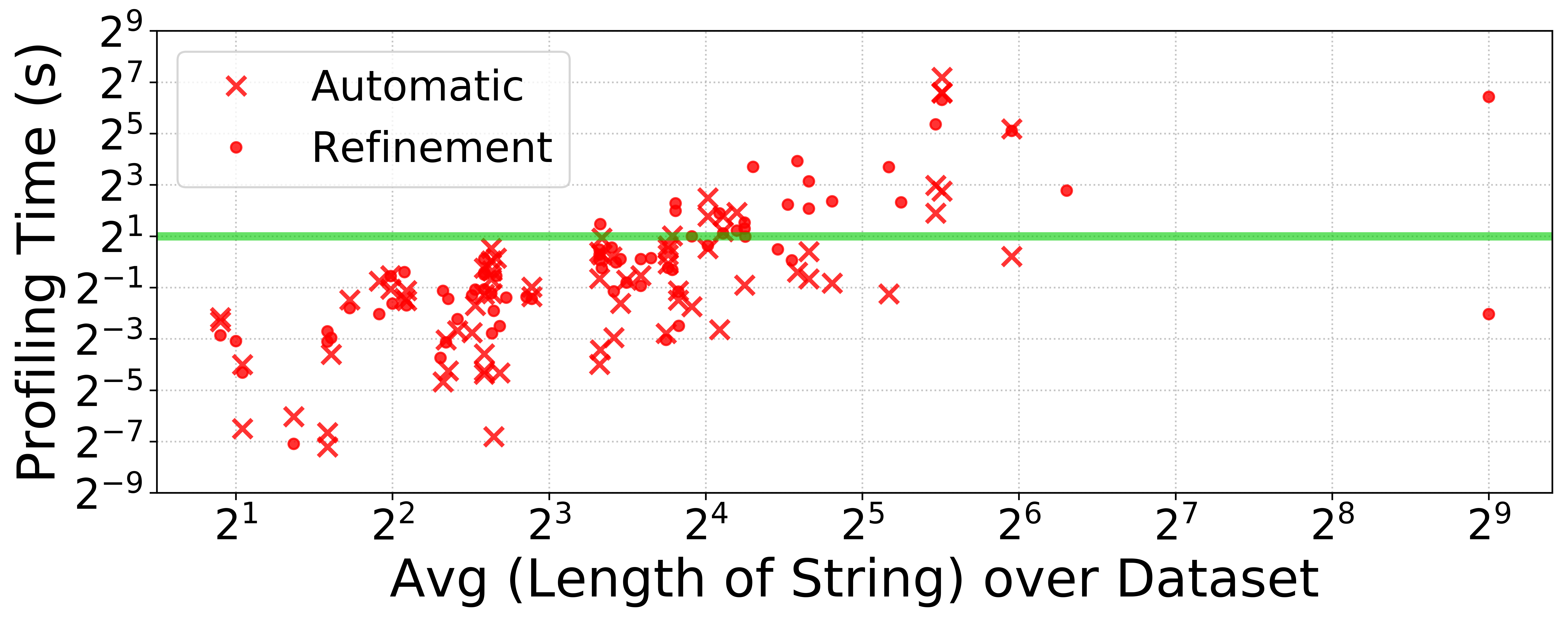}
      \captionsetup{skip=1.5pt}
      \caption{Performance over real-life datasets}
      \label{fig:eval-perf-real}
    \end{minipage}%
    \begin{minipage}{0.48\linewidth}\vspace*{4pt}
      \centering
      \begin{minipage}{0.215\linewidth}
        \hspace*{-8pt}\centering
        \subcaptionbox{Dataset}{
          \hspace*{-5pt}\centering\smaller[1.5]
          \texttt{
            \def\arraystretch{1.35}
            \begin{tabular}{|l|} \hline
            \rowcolor{bordercolor}
            \textsf{Zip Code} \\\hline
            99518\\\hline
            61021-9150\\\hline
            2645\\\hline
            83716\\\hline
            \rule{0pt}{3ex}{\small\smash{$\vdots$}}\\\hline
            K0K 2C0\\\hline
            14480\\\hline
            S7K7K9\\\hline
          \end{tabular}}}
      \end{minipage}%
      \begin{minipage}{0.38\linewidth}
        \centering
        \subcaptionbox{\One}{
          \centering\smaller[1.5]
          \texttt{
            \def\arraystretch{1.25}
            \begin{tabular}{|c|} \hline
              LDL DLD\\\hline
              LDLDLD\\\hline
              N-N\\\hline
              N\\\hline
            \end{tabular}}}\\[10pt]
        \subcaptionbox{\SSDT}{
        \centering\smaller[1.5]
        \texttt{
          \def\arraystretch{1.25}
          \begin{tabular}{|c|} \hline
            \bs{}w\bs{}w\bs{}w \bs{}w\bs{}w\bs{}w\\\hline
            \bs{}d\bs{}d\bs{}d\bs{}d\bs{}d\\\hline
            \bs{}d\bs{}d\bs{}d\bs{}d\\\hline
            .*\\\hline
          \end{tabular}}}
      \end{minipage}%
      \begin{minipage}{0.35\linewidth}
        \centering
        \subcaptionbox{\FP\label{fig:profile-us-canada-zip-fp}}{
          \centering\smaller[1.5]
          \def\arraystretch{1.01}
          \begin{tabular}{|c|} \hline
            $\sUpper \sconcat \sDigit \sconcat \sUpper \sconcat \sSpace \sconcat \sDigit \sconcat \sUpper \sconcat \sDigit$ \\\hline
            $\sstringliteral{61} \sconcat \srep{\sDigit}{3} \sconcat \sstringliteral{-} \sconcat \srep{\sDigit}{4}$ \\\hline
            $\sstringliteral{S7K7K9}$ \\\hline
            $\sDigit^+$ \\\hline
            $\epsilon$ \\\hline
          \end{tabular}}\\[2pt]
        \subcaptionbox{\FPk{6}}{
          \centering\smaller[1.5]
          \def\arraystretch{1.01}
          \begin{tabular}{|c|} \hline
            $\sUpper \sconcat \sDigit \sconcat \sUpper \sconcat \sSpace \sconcat \sDigit \sconcat \sUpper \sconcat \sDigit$ \\\hline
            $\sstringliteral{61} \sconcat \srep{\sDigit}{3} \sconcat \sstringliteral{-} \sconcat \srep{\sDigit}{4}$ \\\hline
            $\sstringliteral{S7K7K9}$ \\\hline
            $\srep{\sDigit}{5}$ \\\hline
            $\srep{\sDigit}{4}$ \\\hline
            $\epsilon$ \\\hline
          \end{tabular}}
      \end{minipage}\\[2pt]
      {\small Most frequent pattern from \PottersWheel = \texttt{int}}
      \caption{Profiles for a dataset with zip codes\protect\footnotemarkNum\label{fig:profile-us-canada-zip}}
  \end{minipage}%
  \vspace*{-6pt}
\end{figure}

\footnotetextNum{~Dataset collected from a database of vendors across US and Canada: \url{https://goo.gl/PGS2pL}}

\noindent
We compare the profiles learned by \FlashProfile to the outputs from 3 state-of-the-art tools:
\begin{inlist}
  \item \Ataccama~\citep{web.ataccama}: a dedicated profiling tool,
  \item \tool{Microsoft}'s SSDT~\citep{web.ssdt} a feature-rich IDE for database applications, and
  \item \PottersWheel~\citep{raman2001potter}: a tool that detects the most frequent data pattern
        and predicts anomalies in data.
\end{inlist}
\Cref{fig:profile-us-canada-zip} and \cref{fig:profile-us-routes} show the observed outputs.
We list the output of \Ataccama against \One, the suggested profile from \FlashProfile against \FP,
and the one generated on requesting $k$ patterns against \FPk{k}. For brevity, we
\begin{inlist}
  \item omit the concatenation operator ``$\concat$'' between atoms, and
  \item abbreviate $\Digit \mapsto \sDigit$, $\Upper \mapsto \sUpper$, $\AlphaSpace \mapsto \sAlphaSpace$,
        $\AlphaDigitSpace \mapsto \sAlphaDigitSpace$.
\end{inlist}

First, we observe that \SSDT generates an overly general \stringliteral{.*} pattern in both cases.
\Ataccama generates a very coarse grained profile in both cases, which although explains the pattern of special characters,
does not say much about other syntactic properties, such as common prefixes, or fixed-length patterns.
With \FlashProfile, one can immediately notice in \cref{fig:profile-us-canada-zip-fp},
that \stringliteral{S7K7K9} is the only Canadian zip code which does not have a space in the middle,
and that some US zip codes have 4 digits instead of 5 (probably the leading zero was lost while interpreting it as a number).
Similarly, one can immediately observe that in \cref{fig:profile-us-routes-fp},
\stringliteral{12348 N CENTER} is not a route.
Similarly the pattern $\sstringliteral{US 26(} \sconcat \sAlphaSpace^+ \sconcat \sstringliteral{)}$
indicates that it is the only entry with a space instead of a dash between the \stringliteral{US} and \stringliteral{26}.

In many real-life scenarios, simple statistical profiles are not enough for data understanding or validation.
\FlashProfile allows users to gradually drill into the data by requesting profiles with a desired granularity.
Furthermore, they may also provide custom atoms for domain-specific profiling.

% HACK: Not sure why the number was being incremented by 2.
\addtocounter{footnoteNum}{-1}

\begin{figure}[t]
    \vspace*{5pt}\centering
    \begin{minipage}{0.21\linewidth}
      \centering
      \subcaptionbox{Dataset}{
        \centering\smaller[1.5]
        \texttt{
          \def\arraystretch{1.125}
          \begin{tabular}{|l|} \hline
          \rowcolor{bordercolor}
          \textsf{Routes} \\\hline
          OR-213\\\hline
          I-5 N\\\hline
          I-405 S\\\hline
          OR-141\\\hline
          \rule{0pt}{3ex}{\small\smash{$\vdots$}}\\\hline
          OR-99E\\\hline
          US-26 E\\\hline
          12348 N CENTER\\\hline
          US-217 S\\\hline
          \rule{0pt}{3ex}{\small\smash{$\vdots$}}\\\hline
          I-84 E\\\hline
          US 26(SUNSET)\\\hline
          OR-224\\\hline
        \end{tabular}}}
    \end{minipage}%
    \begin{minipage}{0.19\linewidth}
      \centering
      \subcaptionbox{\One}{
        \centering\smaller[1.5]
        \texttt{
          \def\arraystretch{0.95}
          \begin{tabular}{|c|} \hline
            N L W\\\hline
            W N (W)\\\hline
            W N (W W W)\\\hline
            W-N\\\hline
            W-NW\\\hline
            W-N W\\\hline
          \end{tabular}}}\\[6pt]
      \subcaptionbox{\SSDT}{
        \centering\smaller[1.5]
        \texttt{
          \def\arraystretch{0.975}
          \begin{tabular}{|c|} \hline
            US-26 E\\\hline
            US-26 W\\\hline
            I-5 N\\\hline
            I-5 S\\\hline
            I-84 E\\\hline
            I-84 W\\\hline
            I-\bs{}d\bs{}d\bs{}d N\\\hline
            I-\bs{}d\bs{}d\bs{}d S\\\hline
            .*\\\hline
          \end{tabular}}}
    \end{minipage}%
    \begin{minipage}{0.26\linewidth}
      \centering
      \subcaptionbox{\FP\label{fig:profile-us-routes-fp}}{
        \centering\smaller[1.5]
        \def\arraystretch{1.325}
        \begin{tabular}{|c|} \hline
          $\sstringliteral{12348 N CENTER}$ \\\hline
          $\sstringliteral{US 26(} \sconcat \sAlphaSpace^+ \sconcat \sstringliteral{)}$\\\hline
          $\sUpper^+ \sconcat \sstringliteral{-} \sconcat \sAlphaDigitSpace^+$\\\hline
          $\epsilon$ \\\hline
        \end{tabular}}\\[7pt]
      \subcaptionbox{\FPk{7}}{
        \centering\smaller[1.5]
        \def\arraystretch{1.325}
        \begin{tabular}{|c|} \hline
          $\sstringliteral{12348 N CENTER}$\\\hline
          $\sstringliteral{US 26(SUNSET)}$\\\hline
          $\sstringliteral{US 26(MT HOOD HWY)}$\\\hline
          $\sUpper^+\sconcat\sstringliteral{-}\sconcat\sDigit^+$\\\hline
          $\srep{\sUpper}{2}\sconcat\sstringliteral{-}\sconcat\srep{\sDigit}{2}\sconcat\sUpper^+$\\\hline
          $\sUpper^+\sconcat\sstringliteral{-}\sconcat\sDigit^+\sconcat\!\sSpace\sconcat\sUpper^+$\\\hline
          $\epsilon$ \\\hline
        \end{tabular}}
    \end{minipage}%
    \begin{minipage}{0.34\linewidth}
      \centering
      \subcaptionbox{\FPk{9}}{
        \centering\smaller[1.5]
        \def\arraystretch{1.225}
        \begin{tabular}{|c|c|} \hline
          $\sstringliteral{US-30BY}$ & $\sstringliteral{12348 N CENTER}$\\\hline
          $\epsilon$ & $\sstringliteral{US 26(SUNSET)}$\\\hline
          $\sUpper^+\sconcat\!\sstringliteral{-}\,\sconcat\sDigit^+$ &
          $\sstringliteral{OR-99}\sconcat\srep{\sUpper}{1}$\\\hline
          $\sstringliteral{I-}\sconcat\sDigit^+\sconcat\!\sSpace\sconcat\sUpper^+$ &
          $\srep{\sUpper}{2}\sconcat\sstringliteral{-2}\sconcat\sDigit^+\sconcat\!\sSpace\sconcat\srep{\sUpper}{1}$\\\hline
          \multicolumn{2}{|c|}{$\sstringliteral{US 26(MT HOOD HWY)}$}\\\hline
        \end{tabular}}\\[5pt]
      \subcaptionbox{\FPk{13}}{
        \centering\smaller[1.5]
        \def\arraystretch{1.225}
        \begin{tabular}{|c|c|} \hline
          $\sstringliteral{US-30BY}$ & $\sstringliteral{12348 N CENTER}$\\\hline
          $\sstringliteral{I-5}$ & $\sstringliteral{US-26}\sconcat\!\sSpace\sconcat\srep{\sUpper}{1}$ \\\hline
          $\sstringliteral{US-30}$ & $\sstringliteral{US 26(SUNSET)}$\\\hline
          $\sstringliteral{OR-}\sconcat\sDigit^+$ & $\sstringliteral{OR-99}\sconcat\srep{\sUpper}{1}$\\\hline
          $\sstringliteral{I-5}\sconcat\sSpace\sconcat\sUpper^+$ &
          $\sstringliteral{I-}\sconcat\sDigit^+\sconcat\!\sSpace\sconcat\srep{\sUpper}{1}$\\\hline
          $\epsilon$& $\sstringliteral{OR-217}\sconcat\!\sSpace\sconcat\srep{\sUpper}{1}$\\\hline
          \multicolumn{2}{|c|}{$\sstringliteral{US 26(MT HOOD HWY)}$}\\\hline
        \end{tabular}}
    \end{minipage}\\[6pt]
    {\small Most frequent pattern from \PottersWheel = \texttt{IspellWord int space AllCapsWord}}
    \caption{Profiles for a dataset containing US routes\protect\footnotemarkNum\label{fig:profile-us-routes}\vspace*{-4pt}}
\end{figure}

\footnotetextNum{~Dataset collected from \url{https://portal.its.pdx.edu/fhwa}}

%\medbreak
%\begin{exmp}\emph{A dirty dataset containing salaries}\\[1pt]
%\input{include/figures/study-salaries.tex}
%\end{exmp}

%\smallbreak
%\begin{exmp}\emph{A dataset containing postal codes}\\[1pt]
%\input{include/figures/study-us-canada-zip.tex}
%\end{exmp}

%\begin{exmp}\emph{A dataset containing U.S. routes}\\
%\hspace*{10pt}{\small (We abbreviate: $\AlphaSpace \mapsto \sAlphaSpace$, $\AlphaDigitSpace \mapsto \sAlphaDigitSpace$)}\\
%\input{include/figures/study-primary-routes.tex}
%\end{exmp}

\section{Applications in PBE Systems}\label{sec:applications-in-pbe-systems}

\noindent
In this section, we discuss how syntactic profiles can improve
programming-by-example (PBE)~\citep{lieberman2001your, gulwani2017program} systems,
which synthesize a desired program from a small set of input-output examples.
For instance, given an example $\stringliteral{Albert Einstein} \tospec \stringliteral{A.E.}$,
the system should learn a program that extracts the initials for names.
Although many PBE systems exist today, most share criticisms on low usability
and confidence in them~\citep{lau2009programming,mayer2015user}.

Examples are an inherently under-constrained form of specifying the desired program behavior.
Depending on the target language, a large number of programs may be consistent with them.
Two major challenges faced by PBE systems today are:
\begin{inlist}
    \item obtaining a set of examples that accurately convey the desired behavior
          to limit the space of synthesized programs, and
    \item ranking these programs to select the ones that are \emph{natural} to users.
\end{inlist}

In a recent work, \citet{ellis2017learning} address \inliststyle{(2)} using data profiles.
They show that incorporating profiles for input-output examples significantly improves ranking,
compared to traditional techniques which only examine the structure of the synthesized programs.
We show that data profiles can also address problem \inliststyle{(1)}.
\citet{raychev2016learning} have presented \emph{representative data samplers} for synthesis scenarios,
but they require the outputs for all inputs.
In contrast, we show a novel interaction model for proactively requesting users to supply
the desired outputs for syntactically different inputs,
thereby providing a representative set of examples to the PBE system.

\paragraph{Significant Inputs}
Typically, users provide outputs for only the first few inputs of target dataset.
However, if these are not representative of the entire dataset,
the system may not learn a program that generalizes over other inputs.
Therefore, we propose a novel interaction model that requests the user
to provide the desired outputs for \emph{significant} inputs, incrementally.
A significant input is one that is syntactically the most dissimilar with
all previously labelled inputs.
%We have implemented this interaction model for \FlashFill, using \FlashProfile.

%To compute the most dissimilar input, we use our learner \learner and cost function \cost to compute 
We start with a syntactic profile \profile for the input dataset
and invoke the \OrderPartitionsAlgo function, listed in \cref{algo:order-parts},
to order the partitions identified in \profile based on mutual dissimilarity, i.e.
each partition $\dataset_i$ must be as dissimilar as possible
with (its most-similar neighbor within) the partitions $\{\dataset_1,\ldots,\dataset_{i-1}\}$.
It is a simple extension of our \SampleDissimilaritiesAlgo procedure (\cref{algo:sampling})
to work with sets of strings instead of strings.
We start with the partition that can be described with the minimum-cost pattern.
Then, from the remaining partitions,
we iteratively select the one that is most dissimilar to those previously selected.
We define the dissimilarity between two partitions as
the cost of the best (least-cost) pattern required to describe them together.

\begin{wrapfigure}{r}{0.5\textwidth}
    \vspace*{4pt}\algobox{1.2}{
    \begin{algfunction}
      {\WithLCParam{\OrderPartitions}}
      {\profile\colon \Tprofile}
      {A sequence of partitions $\langle \dataset_1, \ldots, \dataset_{|\profile|} \rangle$ over \dataset}
      \vspace{-2pt}
      \LeftComment{Select with the partition that has the minimum cost.}
      \State $\polar \gets \big\langle \big( \argmin_{X \in \profile} \enskip \Call{\cost}{X.\prop{Pattern}, X.\prop{Data}} \big).\prop{Data} \big\rangle$
      \While{$|\,\polar\,| < |\,\profile\,|$}
        \LeftComment{\hspace*{11.5pt}Pick the most dissimilar partition w.r.t. those in \polar.}
        \State{$T \gets \argmax_{Z \in \profile} \;\min_{X \in \polar}$\\\hspace*{42pt}$\big(\Call{\WithLCParam{\BestPattern}}{Z.\prop{Data} \cup X}\big).\prop{Cost}$}
        \State $\polar.\prop{Append}(T.\prop{Data})$
      \EndWhile
      \Return{$\polar$}
    \end{algfunction}
  }
  \caption{Ordering partitions by mutual dissimilarity}
  \label{algo:order-parts}
\end{wrapfigure}

Once we have an ordered set of partitions, $\langle \dataset_1,\ldots,\dataset_{|\profile|} \rangle$,
we request the user to provide the desired output for a randomly selected input from each partition in order.
Since PBE systems like \FlashFill are interactive, and start synthesizing programs right from the first example,
the user can inspect and skip over inputs for which the output is correctly predicted by the synthesized program.
After one cycle through all partitions, we restart from partition $\dataset_1$,
and request the user to provide the output for a new random string in each partition.

\begin{wrapfigure}{r}{0.55\textwidth}
  \vspace*{-8pt}\hspace*{-4pt}
  \includegraphics[width=1.025\linewidth]{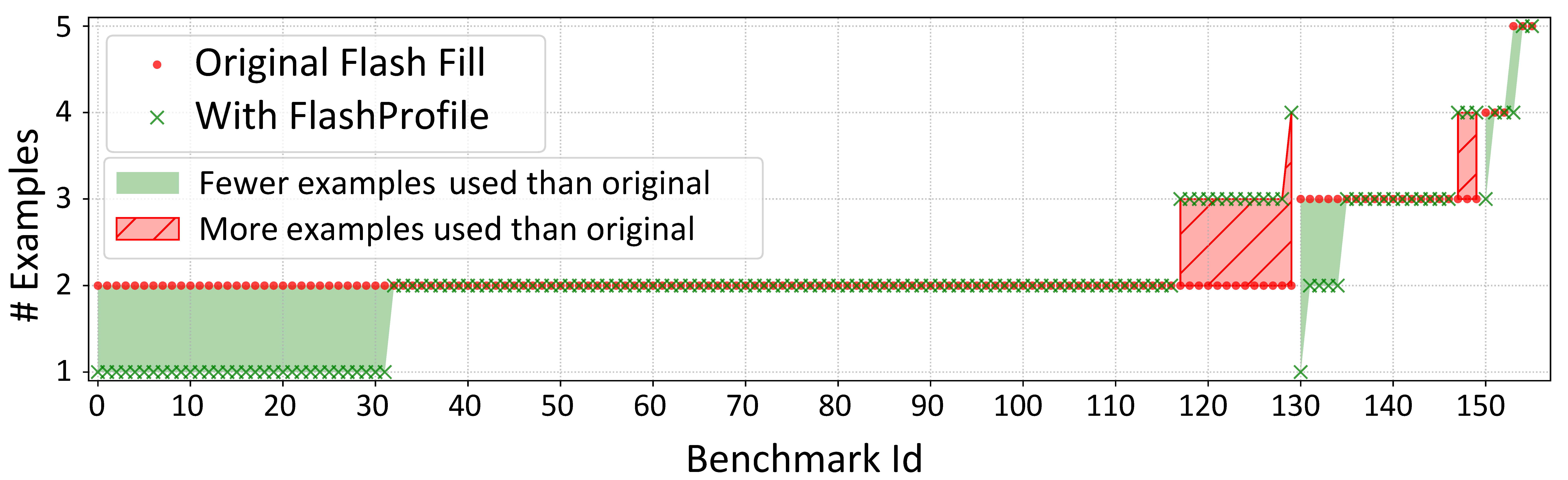}
  \caption{Examples needed with and without \FlashProfile\vspace*{5pt}}
  \label{fig:flash-fill-results}
\end{wrapfigure}

We evaluate the proposed interaction model over \numFFTotalCond \FlashFill benchmarks\footnoteNum{
~These benchmarks are a superset of the original set of \FlashFill~\citep{gulwani2011automating} benchmarks,
with many more real-world scenarios collected from customers using products powered by \PROSE~\citep{web.prose}.
} that require more than one example to learn the desired string-transformation program.
\Cref{fig:flash-fill-results} compares the number of examples required originally,
to that using our interaction model.
Seven cases that timeout due to the presence of extremely long strings have been omitted.

Over the remaining 156 cases, we observe that, \FlashFill
\begin{inlist}
    \item requires a single example per partition for \numFFRecallCond (=\,\numFFRecallCondPercent) cases, and
    \item uses the minimal set\footnoteNum{
          ~By \emph{minimal}, we mean that there is no smaller set of examples with which \FlashFill can synthesize the desired program.
          } of examples to synthesize the desired program for \numFFPrecisionCond (=\,\numFFPrecisionCondPercent) cases ---
          39 of which were improvements over \FlashFill.
\end{inlist}
Thus, \inliststyle{(1)} validates our hypothesis that our partitions indeed identify representative inputs,
and \inliststyle{(2)} further indicates that our interaction model is highly effective.
Selecting inputs from partitions ordered based on mutual syntactic dissimilarity
helps \FlashFill converge to the desired programs with fewer examples.
Note that, these results are based on the default set of atoms.
Designing custom atoms for string-transformation tasks,
based on \FlashFill's semantics is also an interesting direction.
%However, our we required more examples than \FlashFill for 16 cases,
%where syntactic similarity based on our default atoms
% TODO ^^^

Although the significant inputs scenario is similar to \emph{active learning},
which is well-studied in machine-learning literature~\citep{hannekeal},
typical active-learning methods require hundreds of labeled examples.
In contrast, PBE systems deal with very few examples~\citep{mayer2015user}.

\section{Related Work} \label{sec:related-work}

There has been a line of work on profiling various aspects of datasets ---
\citet{abedjan2015profiling} present a recent survey.
Traditional techniques for summarizing data target statistical profiles~\citep{cormode2012synopses},
such as sampling-based aggregations~\citep{haas1995sampling},
histograms~\citep{ioannidis2003history},
and wavelet-based summaries~\citep{karras2007haar}.
However, pattern-based profiling is relatively underexplored,
and is minimally or not supported by state-of-the-art data analysis tools~\citep{
  web.openrefine,web.ssdt,web.ataccama,web.wrangler}.

To our knowledge, no existing approach learns syntactic profiles over an extensible language
and allows refinement of generated profiles.
We present a novel dissimilarity measure which is the key to learning refinable profiles over arbitrary user-specified patterns.
Microsoft's SQL Server Data Tools (\SSDT)~\citep{web.ssdt} learns rich regular expressions but is neither extensible not comprehensive.
A dedicated profiling tool \Ataccama~\citep{web.ataccama} generates comprehensive profiles over a very small set of base patterns.
Google's \OpenRefine~\citep{web.openrefine} does not learn syntactic profiles,
but it allows clustering of strings using character-based similarity measures~\citep{gomaa2013survey}.
In \cref{sec:evaluation} we show that such measures do not capture syntactic similarity.
While \PottersWheel~\citep{raman2001potter} does not learn a complete profile,
it learns the most frequent data pattern over arbitrary user-defined \emph{domains}
that are similar to our atomic patterns.

\paragraph{Application-Specific Structure Learning}
There has been prior work on learning specific structural properties aimed at aiding data wrangling applications,
such as data transformations~\citep{raman2001potter,singh2016blinkfill},
information extraction~\citep{li2008regular}, and reformatting or text normalization~\citep{kini2015flashnormalize}.
However, these approaches make specific assumptions regarding the target application,
which do not necessarily hold when learning general purpose profiles.
Although profiles generated by \FlashProfile are primarily aimed at data understanding,
in \cref{sec:applications-in-pbe-systems} we show that they may aid PBE applications,
such as \FlashFill~\citep{gulwani2011automating} for data transformation.
\citet{bhattacharya2015automated} also utilize hierarchical clustering to group together sensors used in building automation based on their tags.
However, they use a fixed set of domain-specific features for tags and do not learn a pattern-based profile.

\paragraph{Grammar Induction}
Syntactic profiling is also related to the problem of learning regular expressions,
or more generally grammars from a given set of examples.
\citet{de2010grammatical} present a recent survey on this line of work.
Most of these techniques, such as \LStar~\citep{angluin1987learning} and \RPNI~\citep{oncina1992identifying},
assume availability of both positive and negative examples, or a membership oracle.
\citet{bastani2017synthesizing} show that these techniques are either too slow or do not generalize well
and propose an alternate strategy for learning grammars from positive examples.
When a large number of negative examples are available, genetic programming has also been shown to be
useful for learning regular expressions~\citep{svingen1998learning,bartoli2012automatic}.
Finally, \tool{LearnPADS}~\citep{fisher2008learnpads,zhu2012learnpads++} also generates a syntactic description,
but does not support refinement or user-specified patterns.

\paragraph{Program Synthesis}
Our techniques for sampling-based approximation and finding representative inputs
relate to prior work by \citet{raychev2016learning} on synthesizing programs from noisy data.
However, they assume a single target program and the availability of outputs for all inputs.
In contrast, we synthesize a disjunction of several programs,
each of which returns {\small \True} only on a specific partition of the inputs,
which is unknown a priori.

\FlashProfile's pattern learner uses the \PROSE library~\citep{web.prose},
which implements the \FlashMeta framework~\citep{polozov2015flashmeta}
for inductive program synthesis, specifically programming-by-examples (PBE)~\citep{lieberman2001your, gulwani2017program}.
PBE has been leveraged by recent works on automating repetitive text-processing tasks,
such as string transformation~\citep{gulwani2011automating,singh2016blinkfill},
extraction~\citep{le2014flashextract}, and format normalization~\citep{kini2015flashnormalize}.
However, unlike these applications, data profiling does not solicit any (output) examples from the user.
We demonstrate a novel application of a supervised synthesis technique to solve an unsupervised learning problem.
\section{Conclusion} \label{sec:conclusion}

\noindent
With increasing volume and variety of data,
we require powerful profiling techniques to enable end users to understand and analyse their data easily.
Existing techniques generate a single profile over pre-defined patterns,
which may be too coarse grained for a user's application.
We present a framework for learning syntactic profiles over user-defined patterns,
and also allow refinement of these profiles interactively.
Moreover, we show domain-specific approximations that allow end users to
control accuracy \emph{vs.} performance trade-off for large datasets,
and generate approximately correct profiles in realtime on consumer-grade hardware.
We instantiate our approach as \FlashProfile,
and present extensive evaluation on its accuracy and performance on real-life datasets.
We also show that syntactic profiles are not only useful for data understanding and manual data analysis tasks,
but can also help existing PBE systems.
\begin{acks}
  %% contents suppressed with 'anonymous'
  %% Commands \grantsponsor{<sponsorID>}{<name>}{<url>} and
  %% \grantnum[<url>]{<sponsorID>}{<number>} should be used to
  %% acknowledge financial support and will be used by metadata
  %% extraction tools.
  The lead author is thankful to the \PROSE team at Microsoft, especially to
  Vu Le, Danny Simmons, Ranvijay Kumar, and Abhishek Udupa,
  for their invaluable help and support.
  We also thank the anonymous reviewers for their constructive
  feedback on earlier versions of this paper.

  This research was supported in part by an internship at Microsoft,
  by the \grantsponsor{GS100000001}{National Science Foundation (NSF)}{http://dx.doi.org/10.13039/100000001}
  under Grant No.~\grantnum{GS100000143}{CCF-1527923},
  and by a Microsoft Research Ph.D. Fellowship.
  Any opinions, findings, and conclusions or recommendations expressed
  in this material are those of the author and do not necessarily
  reflect the views of the NSF or of the Microsoft Corporation.
\end{acks}

\bibliography{paper}

\end{document}